\documentclass{article}

\usepackage[preprint]{neurips_2026}
\makeatletter
\renewcommand{\@noticestring}{}
\makeatother

\usepackage[utf8]{inputenc}
\usepackage[T1]{fontenc}
\usepackage{hyperref}
\usepackage{url}
\usepackage{booktabs}
\usepackage{amsmath}
\usepackage{amssymb}
\usepackage{amsfonts}
\usepackage{amsthm}
\usepackage{mathtools}
\usepackage{nicefrac}
\usepackage{microtype}
\usepackage{xcolor}
\usepackage{graphicx}
\usepackage{subcaption}   
\usepackage{enumitem}
\usepackage{tcolorbox}                              \tcbuselibrary{breakable,skins}  

\newtheorem{proposition}{Proposition}
\newtheorem{theorem}{Theorem}
\newtheorem{assumption}{Assumption}
\newtheorem{lemma}{Lemma}

\DeclareMathOperator{\Var}{Var}
\DeclareMathOperator{\Cov}{Cov}
\DeclareMathOperator{\AVar}{AVar}
\newcommand{\method}{\textsc{LACE}}

\AtBeginDocument{%
  \setlength{\abovedisplayskip}{6pt plus 2pt minus 2pt}%
  \setlength{\belowdisplayskip}{6pt plus 2pt minus 2pt}%
  \setlength{\abovedisplayshortskip}{3pt plus 2pt minus 1pt}%
  \setlength{\belowdisplayshortskip}{4pt plus 2pt minus 2pt}%
}

\title{Conditional Evaluation of Language Models with Cheap Auxiliary Signals}

\author{%
  Zhi Zhang\textsuperscript{\rm 1}\thanks{Equal contribution.}\quad
  Lingfeng Lyu\textsuperscript{\rm 2}\footnotemark[1]\quad
  Yue Kang\textsuperscript{\rm 3}\quad
  Doudou Zhou\textsuperscript{\rm 4}\thanks{Corresponding author:
  \texttt{doudouzhou@nus.edu.sg}}\\[6pt]
  \textsuperscript{\rm 1}Department of Statistics and Data Science,
  University of California, Los Angeles\\
  \textsuperscript{\rm 2}Department of Statistics and Finance,
  School of Management,\\ University of Science and Technology of China\\
  \textsuperscript{\rm 3}Microsoft\\
  \textsuperscript{\rm 4}Department of Statistics and Data Science,
  National University of Singapore\\
}

\begin{document}

\maketitle

\begin{abstract}
Aggregate accuracy hides where models succeed and fail. Estimating conditional performance profiles
from gold labels alone is expensive, while cheap auxiliary signals such as LLM-judge scores,
pairwise comparisons, confidence scores, and judge-disagreement features can be collected for every
benchmark item but are often biased or miscalibrated. We propose \method{} (Local Augmented Control-Variate
Evaluation), a semi-supervised estimator for conditional LLM evaluation. The key step is local
centering: after subtracting the conditional mean of a cheap signal within the target profile region,
any linear augmentation has zero conditional mean and therefore cannot change the estimand. The
augmentation coefficient is used only for efficiency, and a local ridge control variate combines a
gold-label residual mean from the labeled subset with a cheap-signal mean from the full item pool.
We prove calibration-free identification, unbiasedness for grouped profiles, local
oracle optimality within centered linear augmentations, and first-order adaptivity to the estimated
coefficient. The resulting gain formula is governed by a population local $R^2$, which
characterizes how the efficiency attainable from the cheap signals varies across profile values.
We also derive corresponding estimators for direct paired model gaps
and deployment-weighted scores. We empirically evaluate the primary
performance-profile estimator on MATH-500, ScienceQA, MMLU, WinoGrande,
HellaSwag, TruthfulQA, GSM8K, and ARC.
\end{abstract}

\section{Introduction}

Large language model evaluation is increasingly a question of conditional performance, not merely
aggregate score. A model may be strong on algebra but weak on geometry, reliable in high-school
science but brittle in elementary grades, or dominant in some subject categories while losing in
others. We index such profiles by a profiling covariate $Z\in\mathcal Z$. The profile space
$\mathcal Z$ may represent a single metadata axis, such as difficulty level or subject category, or a
multidimensional combination of attributes, such as grade, subject, and task type. Such profiles
matter for model selection, deployment targeting, and
scientific diagnosis, echoing broader calls for benchmark reports that expose performance across
tasks, metrics, and relevant subgroups \citep{mitchell2019modelcards,liang2023helm}. The target of
estimation is the conditional performance profile of a candidate model $m$ over benchmark items $i$:
$\theta_m(z)=\Pr(Y_{im}=1\mid Z_i=z)$, where $Y_{im}$ is the gold correctness of model $m$ on item
$i$ and the profiling covariate $Z_i$ is derived from benchmark metadata. Unfortunately, gold labels
are often expensive: they may require expert grading, careful rubrics, or repeated adjudication. With
limited labels, estimates of $\theta_m(z)$ can be too noisy to support fine-grained decisions.

At the same time, modern evaluation pipelines produce many inexpensive signals. LLM-as-a-judge
scores, pairwise comparisons against anchor models, self-reported confidence, and disagreement across
judge prompts are cheap enough to collect on every benchmark item
\citep{zheng2023judging,liu2023geval,dubois2024length}. These signals are useful but dangerous to
treat as calibrated probabilities. A judge can be biased, overconfident, prompt-sensitive, influenced
by output length or position, or systematically wrong on particular subpopulations
\citep{wang2024fair,wu2023style,dubois2024length}. The key question is therefore not whether cheap
evaluators are correct in an absolute sense, but whether they explain residual variation in gold
correctness locally.

\paragraph{Running example: MATH-500.}
Let $Z_i$ record the difficulty of problem $i$, let $Y_{im}$ indicate whether candidate model $m$
answers it correctly, and let $S_{im}$ collect inexpensive item-level information such as judge
scores, pairwise comparisons with anchor models, self-confidence, and judge disagreement. The
scientific target is $\Pr(Y_{im}=1\mid Z_i=z)$, the model's accuracy at difficulty $z$. The kernel
localizes items in the profile space defined by $Z$; it does not measure similarity among the cheap
signals. Centering $S$ within each $Z$-neighborhood therefore lets those signals improve precision
without changing the estimand.

This paper develops \method{} (Local Augmented Control-Variate Evaluation). We present it through
LLM benchmarking because that setting makes the problem concrete, but the statistical problem is more
general: estimate a conditional mean when the gold outcome is observed on a small random subset, while
multiple auxiliary measurements are observed on every item. The estimator starts from a simple
identity: after centering a cheap signal by its conditional mean given $Z_i=z$, any linear adjustment
has conditional expectation zero. The adjustment coefficient is therefore not needed for
identification; it is chosen only to reduce variance. Estimating this coefficient locally gives a
ridge-regularized control variate that can use unlabeled items through the local mean of cheap
auxiliary signals. This construction is related to recent work on local and conditional
prediction-powered inference (PPI) \citep{gu2024localppi,sui2026ppci}, which combines a supplied
outcome prediction with a localized labeled residual correction. \method{} addresses a different
version of the problem: $Z$ alone defines the conditional target, while an additional vector $S$ is
used only for precision. It learns the full coefficient $\beta(z)$ from the same scarce labels, so
the relative value of the cheap signals may change across the profile without adding $S$ to the
conditioning event. The construction also treats continuous, ordered discrete, and categorical
profile spaces within one estimator and provides a local theoretical characterization of the
attainable efficiency gain. Concretely, we contribute a general semi-supervised profiling formulation for
estimating $\mathbb E(Y\mid Z=z)$ with scarce gold outcomes and abundant heterogeneous auxiliary
measurements; a local control-variate estimator covering continuous, ordered discrete, and unordered
categorical profile spaces; theory establishing calibration-free identification
(Proposition~\ref{prop:identity}), group-profile unbiasedness (Theorem~\ref{thm:finite}), oracle
optimality within centered linear augmentations (Theorem~\ref{thm:oracle-optimality}), and
first-order adaptivity (Theorem~\ref{thm:adaptivity}); a local gain formula
$\operatorname{Gain}_m(z)=1/\{1-(1-\pi)R_m^2(z)\}$ that characterizes heterogeneous signal quality across
$\mathcal{Z}$; theoretical extensions to direct paired model gaps and
deployment-weighted scores; and an empirical evaluation of the primary
profile estimator, with cell-level RE ranging from $2.73\times$ to
$11.03\times$ across eight benchmarks, three models, and three label
budgets.

\section{Related Work}

\paragraph{Conditional LLM evaluation.}
Fine-grained and conditional reporting is already central to LLM evaluation. MMLU reports
subject-level performance \citep{hendrycks2021mmlu}; MATH, ScienceQA, GPQA, and MMLU-Pro expose
metadata or splits such as level, grade, subject, and domain
\citep{hendrycks2021math,lu2022learn,rein2024gpqa,wang2024mmluPro};
and evaluation suites such as BIG-bench, HELM, and the Language Model Evaluation
Harness organize results across tasks, scenarios, metrics, and subtasks
\citep{srivastava2022beyond,liang2023helm,gao2024harness}. Recent benchmarks go further by defining
fine-grained skill or instruction axes, including FLASK, IFEval, MT-Bench-101, and BiGGen Bench
\citep{ye2023flask,zhou2023ifeval,bai2024mtbench101,kim2024biggen}. Documentation frameworks such
as model cards also advocate reporting performance across relevant conditions and subgroups
\citep{mitchell2019modelcards}. Our focus is complementary: given such a profiling axis, we ask how
to estimate the corresponding conditional profile when only a small random fraction of items carry
gold labels, using cheap auxiliary signals for variance reduction rather than as replacement labels.

\paragraph{Cheap evaluators, LLM judges, and aggregation.}
LLM judges and pairwise preference evaluators are widely used because they are scalable and often
correlated with human or gold-label assessments \citep{zheng2023judging,liu2023geval}. Related
systems use LLMs to estimate factuality or response quality at far lower cost than
manual evaluation \citep{min2023factscore,dubois2024length}, and multi-agent judging has been
proposed to improve reliability \citep{chan2024chateval}. A parallel line trains or distills
dedicated judge models, including PandaLM, Prometheus, Auto-J, JudgeLM, and CritiqueLLM
\citep{wang2023pandalm,kim2023prometheus,li2024autoj,zhu2023judgelm,ke2024critiquellm}. Pairwise
outcomes are often modeled with Bradley--Terry--Luce (BTL) models
\citep{bradley1952rank}; Chatbot Arena instantiates this at scale through crowdsourced human
preferences \citep{chiang2024chatbotarena}. Pointwise confidence scores can be post-hoc calibrated
to approximate empirical accuracy \citep{guo2017calibration}. Other work aggregates noisy, sparse,
or dependent judge outputs: judge-aware ranking extends BTL models with judge-specific discrimination
parameters \citep{xu2026judgeaware}; CARE models multi-judge scores as arising from a latent
true-quality signal and shared confounding factors \citep{zhao2026care}; and tensor methods model
evaluation as low-rank completion of pairwise comparisons or cluster question--answerer--evaluator
score tensors \citep{li2026tensorcompletion,watanabe2026multiwaypam}. However, LLM judges are known
to exhibit position, verbosity, style, and self-preference biases
\citep{wang2024fair,wu2023style,koo2023cognitive,dubois2024length}. Our estimator imposes neither a
BTL model nor a calibration requirement; it only requires cheap signals to explain local variation in
gold correctness within each subgroup.

\paragraph{Control variates, semi-supervised estimation, and local PPI.}
Control variates are a classical variance-reduction tool in Monte Carlo estimation \citep{owen2013mc}.
Augmented inverse-probability weighting (AIPW), one-step estimation, and semi-supervised regression
estimators combine scarce labels with abundant covariates through fitted means and labeled residual
corrections \citep{robins1994estimation,zhang2019semisupervised,chernozhukov2018double}.
Prediction-powered inference uses machine predictions for scalar estimands
\citep{angelopoulos2023prediction}; PPI++ optimally tunes a global coefficient
\citep{angelopoulos2023ppiplus}. MultiPPI extends this global setting to multiple lower-cost
measurement sources and jointly optimizes acquisition and linear combination weights
\citep{cowenbreen2026multipppi}. Recent work has instead localized PPI to conditional targets.
Local PPI \citep{gu2024localppi} fits a local polynomial to predictions from a supplied model on an
independent unlabeled sample and subtracts a local labeled residual correction. Prediction-powered
conditional inference (PPCI) \citep{sui2026ppci} learns an RKHS localization weight for a conditional
moment at a fixed test point, then applies a prediction-plus-residual decomposition using a supplied
predictor.

\paragraph{Estimator structure in common notation.}
Let $P\in\mathbb R$ denote a supplied scalar prediction and $S\in\mathbb R^K$ a vector of auxiliary
signals. Write $\bar A_D$ for the mean of a variable over $D\in\{L,T\}$ and $\bar A_D(z)$ for its
profile-local counterpart. The four structures can then be displayed as
\[
\begin{aligned}
\text{PPI++ (global scalar):}\quad
  &\widehat\mu=\bar Y_L-\widehat\omega\{\bar P_L-\bar P_T\},\\
\text{Local PPI (conditional scalar):}\quad
  &\widehat\theta(z)=\bar Y_L(z)-\{\bar P_L(z)-\bar P_T(z)\},\\
\text{global vector control variate:}\quad
  &\widehat\mu=\bar Y_L-\widehat\beta^\top\{\bar S_L-\bar S_T\},\\
\text{\method{} (conditional vector):}\quad
  &\widehat\theta(z)=\bar Y_L(z)-\widehat\beta(z)^\top
    \{\bar S_L(z)-\bar S_T(z)\}.
\end{aligned}
\]
These displays use the nested labeled/full-pool notation only to expose estimator structure; the
original PPI++ and Local PPI methods retain their own sampling assumptions. The centered
control-variate algebra itself is classical. Our contribution is the joint construction for a
conditional target with multiple auxiliary signals, profile-dependent relevance, and a local vector
coefficient learned from the same scarce labels with first-order adaptivity.

Our formulation separates two roles that are coupled in these conditional-PPI setups. The profiling
variable $Z$ alone defines the target $\mathbb E(Y\mid Z=z)$, whereas the additional cheap-signal
vector $S$ is used only for precision and is not added to the conditioning event. This distinction is
important in evaluation: setting the conditional-PPI covariate to $Z$ leaves the item-level signals
$S$ unused, while conditioning on $(Z,S)$ changes the scientific target. \method{} instead centers
$S$ within the $Z$-profile region and learns the full vector coefficient $\beta(z)$ from the same
scarce labels. Thus both the magnitude and direction of the auxiliary combination may change with
$z$, without requiring a separately supplied outcome predictor or changing the estimand. The theory
establishes first-order adaptivity to this same-label coefficient estimate and explicitly covers the
nested finite-pool design $L\subset T$ with a non-negligible labeled fraction. The same construction
handles continuous, ordered discrete, and unordered categorical profile spaces and yields the local
gain formula $1/\{1-(1-\pi)R_m^2(z)\}$. In the LLM context, Zhou et al.\ target global win-rate estimation
\citep{zhou2025cv} and Fisch et al.\ apply stratified PPI to pre-specified discrete strata
\citep{fisch2024stratppi}; both target global or stratum-level means. Fogliato et al.\ use
empirical-Bayes shrinkage for small benchmark subgroups \citep{fogliato2024precise};
tinyBenchmarks selects IRT-curated item subsets \citep{polo2024tinybenchmarks}.

\paragraph{Local smoothing, survey estimation, and deployment weighting.}
Our continuous-$Z$ estimator uses kernel weights in the spirit of local nonparametric regression
\citep{fan1996local}; the discrete version is the corresponding group estimator, and the ordinal
version borrows strength across neighboring levels. Ratio and regression estimators in survey
statistics exploit auxiliary information to improve precision
\citep{cochran1977sampling,sarndal1992model}, post-stratification and generalized regression
estimators are discrete analogues of local smoothing, and Horvitz--Thompson weighting handles
unequal inclusion probabilities \citep{horvitz1952generalization}. When a deployment distribution
$Q$ differs from the benchmark distribution $P$, importance weighting corrects distribution shift
\citep{sugiyama2007covariate}. \method{} combines these ideas with local control variates: it targets
nonparametric conditional profiles, allows noisy and biased auxiliary signals, and embeds deployment
reweighting through $r_Q(z)=dQ_Z/dP_Z$ rather than treating it as a post-hoc adjustment to global
accuracy.

\section{\method{}: Local Augmented Control-Variate Evaluation}
\label{sec:method}

\subsection{Setup and profile weights}

We use LLM-evaluation notation, but the ingredients are only a gold outcome, auxiliary
measurements, and a conditioning covariate. Fix a candidate model $m$.
Items in the benchmark pool $T=\{1,\dots,M\}$ have $(Z_i,Y_{im},S_{im})\stackrel{\mathrm{i.i.d.}}{\sim}\mathbb P$, where $Z_i\in\mathcal Z$ is the profiling covariate (continuous, ordered discrete, or unordered categorical), $Y_{im}\in\{0,1\}$ is the gold correctness label, and $S_{im}\in\mathbb R^K$ is the cheap-signal vector. We observe $(Z_i,S_{im})$ for every $i\in T$ but $Y_{im}$ only on a
labeled subset $L\subset T$ of size $n$, drawn uniformly without replacement independently of the item data; write
$\pi=n/M$. The main target is the conditional performance profile
$\theta_m(z)=\mathbb P(Y_{im}=1\mid Z_i=z)$.

The profiling variable $Z$ is prespecified by the scientific question rather than formed by
automatically combining all available metadata. For overlapping tags, each tag may define a separate
low-dimensional query, with the same item contributing to several queries. A joint profile is needed
only when performance on the intersection is itself the target.

For any sample set $A\subseteq T$, local weights $w_{i,A}(z)$ are nonnegative and sum to one over
$A$. When the continuous profile space is a subset of $\mathbb R^q$, we use
$w_{i,A,h}(z)=K_h(Z_i-z)/\sum_{\ell\in A}K_h(Z_\ell-z)$, with
$K_h(u)=h^{-q}K(u/h)$. Ordered discrete profiles use normalized triangular ordinal weights
$w_{i,A}(g)\propto(1-|Z_i-g|/b_{\mathrm{ord}})_+\mathbf 1(i\in A)$, with within-group weights as
the special case $b_{\mathrm{ord}}=1$. Unordered categorical profiles use
$w_{i,A}(g)=\mathbf 1(Z_i=g)/\sum_{\ell\in A}\mathbf 1(Z_\ell=g)$. In what follows,
$w_{i,A}(z)$ denotes the chosen local-weight rule, with bandwidth parameters suppressed when
unambiguous. For any scalar- or vector-valued quantity $V_i$, write
$\bar V_A(z)=\sum_{i\in A}w_{i,A}(z)V_i$; model-specific means such as $\bar S_{m,L}(z)$ use
$V_i=S_{im}$.

\subsection{Centered augmentation principle}

Cheap auxiliary signals may be biased, prompt-sensitive, or miscalibrated, so \method{} never treats
them as substitute labels. Let $\mu_{S,m}(z)=\mathbb E(S_{im}\mid Z_i=z)$. For any deterministic
vector $b(z)$,
\[
\mathbb E\!\left[Y_{im}-b(z)^\top\{S_{im}-\mu_{S,m}(z)\}\mid Z_i=z\right]=\theta_m(z).
\]
Thus centering protects the target: the coefficient controls efficiency, not identification. The
population coefficient that removes the largest amount of local linear variation is
\[
\beta_m^\star(z)=\Sigma_{SS,m}(z)^\dagger\Sigma_{SY,m}(z),
\]
where
\[
\Sigma_{SS,m}(z)=\Var(S_{im}\mid Z_i=z),
\qquad
\Sigma_{SY,m}(z)=\Cov(S_{im},Y_{im}\mid Z_i=z).
\]
Section~\ref{sec:theory} states the formal identification and optimality results.

\subsection{The LACE estimator and local coefficient}
\label{subsec:estimator}

At a target profile value $z$, \method{} combines a labeled residual mean with a full-pool
cheap-signal mean. Using the local weights defined above, define
$\bar Y_{m,L}(z)=\sum_{i\in L}w_{i,L}(z)Y_{im}$ and
$\bar S_{m,A}(z)=\sum_{i\in A}w_{i,A}(z)S_{im}$ for $A\in\{L,T\}$. For any coefficient
$\gamma\in\mathbb R^K$, consider
\begin{equation}
\label{eq:lace-estimator}
\widehat\theta_m(z;\gamma)
=
\bar Y_{m,L}(z)
-
\gamma^\top
\left\{\bar S_{m,L}(z)-\bar S_{m,T}(z)\right\}.
\end{equation}
We set $\gamma$ to a regularized local coefficient. Let $w_{i,A,b}(z)$ denote
coefficient-fitting weights that use the same profile-space rule as $w_{i,A}(z)$, possibly with a
larger bandwidth or ordinal span for stability, and let $\bar S_{m,T,b}(z)$, $\bar S_{m,L,b}(z)$, and
$\bar Y_{m,L,b}(z)$ be the corresponding local means. Define
\[
\widehat\Sigma_{SS,m}(z)
=
\sum_{i\in T}w_{i,T,b}(z)
\{S_{im}-\bar S_{m,T,b}(z)\}\{S_{im}-\bar S_{m,T,b}(z)\}^\top,
\]
\[
\widehat\Sigma_{SY,m}(z)
=
\sum_{i\in L}w_{i,L,b}(z)
\{S_{im}-\bar S_{m,L,b}(z)\}\{Y_{im}-\bar Y_{m,L,b}(z)\},
\]
and set
\[
\widehat\beta_m(z)
=
\{\widehat\Sigma_{SS,m}(z)+\lambda I_K\}^{-1}\widehat\Sigma_{SY,m}(z),
\]
where $\lambda>0$ is a ridge penalty that stabilizes the inversion when the local covariance is
ill-conditioned or the number of labeled items is small relative to $K$.
The final estimator is
$\widehat\theta_m^{\mathrm{LACE}}(z)=\widehat\theta_m(z;\widehat\beta_m(z))$:
\begin{equation}
\label{eq:lace-feasible}
\widehat\theta_m^{\mathrm{LACE}}(z)
=
\bar Y_{m,L}(z)
-
\widehat\beta_m(z)^\top
\left\{\bar S_{m,L}(z)-\bar S_{m,T}(z)\right\}.
\end{equation}
Equivalently,
\[
\widehat\theta_m^{\mathrm{LACE}}(z)
=
\sum_{i\in L}w_{i,L}(z)\{Y_{im}-\widehat\beta_m(z)^\top S_{im}\}
+
\sum_{i\in T}w_{i,T}(z)\widehat\beta_m(z)^\top S_{im}.
\]
Thus the gold labels estimate only the local residual mean, while all items estimate the
cheap-signal component. The ridge term protects against ill-conditioned local covariances; when the
cheap signals are locally uninformative, $\widehat\beta_m(z)$ is shrunk toward zero and the
estimator reduces to the gold-label smoother. When they explain local correctness variation, the
labeled part has lower residual variance, yielding the gain quantified in Section~\ref{sec:theory}.

\subsection{Other evaluation targets}
\label{sec:gaps}

The same centered augmentation applies after changing the outcome, the auxiliary vector, or the
target sampling law. We use this template to define estimators and
derive theory for two LLM-evaluation targets beyond a single model's
profile.

\paragraph{Direct paired model gaps.}
For models $a$ and $b$, define $\Delta_{ab}(z)=\theta_a(z)-\theta_b(z)$ and the paired outcome
$D_{i,ab}=Y_{ia}-Y_{ib}$. Let
$H_{i,ab}=(S_{ia}^\top,S_{ib}^\top,(S_{ia}-S_{ib})^\top)^\top$, with any direct pairwise judge
feature appended when available. The direct gap estimator is
\begin{equation}
\label{eq:lace-gap}
\widehat\Delta_{ab}^{\mathrm{LACE}}(z)
=
\bar D_{ab,L}(z)
-
\widehat\beta_{ab}(z)^\top
\{\bar H_{ab,L}(z)-\bar H_{ab,T}(z)\}.
\end{equation}
Here $\widehat\beta_{ab}(z)$ is fitted by the same ridge covariance formula, replacing $(Y,S)$ by
$(D,H)$. Estimating the gap directly uses the within-item dependence between $Y_{ia}$ and $Y_{ib}$,
which can be more efficient than subtracting two separately estimated profiles.

\paragraph{Deployment-weighted scores.}
Let $Q_Z\ll P_Z$ be a deployment profile distribution and define the joint target law by the
covariate-shift model $Q(dz,dy,ds)=Q_Z(dz)P(dy,ds\mid z)$. Write
$r_Q(z)=dQ_Z(z)/dP_Z(z)$ and $\rho_i=r_Q(Z_i)$. For $A\in\{L,T\}$, define
\[
\bar S_{m,A}^{Q}
=
\frac{\sum_{i\in A}\rho_iS_{im}}{\sum_{i\in A}\rho_i},
\qquad
\bar Y_{m,L}^{Q}
=
\frac{\sum_{i\in L}\rho_iY_{im}}{\sum_{i\in L}\rho_i}.
\]
To estimate the variance-optimal deployment coefficient, set
\[
\widehat\phi^Q_{S,i}=\rho_i\{S_{im}-\bar S^Q_{m,T}\},
\qquad
\widehat\phi^Q_{Y,i}=\rho_i\{Y_{im}-\bar Y^Q_{m,L}\},
\]
and define the influence-moment estimators
\[
\widehat\Gamma^Q_{SS,m}
=
\frac{1}{M}\sum_{i\in T}\widehat\phi^Q_{S,i}\widehat\phi^{Q\top}_{S,i},
\qquad
\widehat\Gamma^Q_{SY,m}
=
\frac{1}{n}\sum_{i\in L}\widehat\phi^Q_{S,i}\widehat\phi^Q_{Y,i}.
\]
For a ridge level $\lambda_{Q,n}>0$, let
\[
\widehat\beta_m(Q)
=
\{\widehat\Gamma^Q_{SS,m}+\lambda_{Q,n}I_K\}^{-1}
\widehat\Gamma^Q_{SY,m}.
\]
Because these are empirical products of the estimated influence quantities, their summands contain
$\rho_i^2=r_Q(Z_i)^2$, rather than a single deployment-weight factor.
The deployment-aware score is
\begin{equation}
\label{eq:lace-deployment}
\widehat\Psi_m^{\mathrm{LACE}}(Q)
=
\bar Y_{m,L}^{Q}
-
\widehat\beta_m(Q)^\top\{\bar S_{m,L}^{Q}-\bar S_{m,T}^{Q}\},
\end{equation}
with the same self-normalized deployment means used both for centering and for the final score.

\section{Theory}
\label{sec:theory}

This section collects the formal guarantees for the estimators in Section~\ref{sec:method}; proofs
are deferred to Appendix~\ref{app:proof}. 
Throughout this section, all conditional moments are
evaluated at profile values $z$ at which the chosen versions of the corresponding conditional
moment functions are well defined.
Let
$\mu_{S,m}(z)=\mathbb E(S_{im}\mid Z_i=z)$,
$\Sigma_{SS,m}(z)=\Var(S_{im}\mid Z_i=z)$,
$\Sigma_{SY,m}(z)=\Cov(S_{im},Y_{im}\mid Z_i=z)$,
$\sigma^2_{Y,m}(z)=\Var(Y_{im}\mid Z_i=z)$, and
$R_m^2(z)=\Sigma_{SY,m}(z)^\top\Sigma_{SS,m}(z)^\dagger\Sigma_{SY,m}(z)/\sigma^2_{Y,m}(z)$ when $\sigma^2_{Y,m}(z)>0$.
Here $R_m^2(z)$ is a population (oracle) quantity; we do not estimate or report it in the real-data experiments.
Unless stated otherwise, limits take $n,M\to\infty$. We write $\rightsquigarrow$ for convergence
in distribution, $\mathcal N(0,v)$ for a normal law with variance $v$, $O_p(\cdot)$ and
$o_p(\cdot)$ under the joint benchmark and label-subsampling law, and $\AVar\{\cdot\}$ for the
leading asymptotic variance of the unscaled estimator.

\begin{proposition}[Calibration-free identification]
\label{prop:identity}
Assume $\mathbb E(\|S_{im}\|_2^2\mid Z_i=z)<\infty$.
For any deterministic vector $b(z)\in\mathbb R^K$,
$\mathbb E[Y_{im}-b(z)^\top\{S_{im}-\mu_{S,m}(z)\}\mid Z_i=z]=\theta_m(z)$.
Among all linear centered augmentations, the conditional variance is minimized by any solution of
$\Sigma_{SS,m}(z)b=\Sigma_{SY,m}(z)$; the minimum-norm solution is
$\beta_m^\star(z)=\Sigma_{SS,m}(z)^\dagger\Sigma_{SY,m}(z)$.
\end{proposition}

\begin{assumption}[Local regularity]
\label{ass:regularity}
The profiling density $f_Z$ is continuous and positive at $z$, and the
conditional first and second moments of $(Y_{im},S_{im})$ are continuous at
$z$. The kernel $K$ is nonnegative, bounded, symmetric, and integrable,
with $\int K(u)\,du=1$ and $R(K)=\int K(u)^2\,du<\infty$. For some
$\delta>0$, either (i) $K$ has compact support and
$\sup_{t\in\mathcal N(z)}\mathbb E\{(1+\|S_{im}\|_2)^{2+\delta}\mid Z_i=t\}<\infty$
for some neighborhood $\mathcal N(z)$ of $z$, or (ii)
$\sup_t f_Z(t)\,\mathbb E\{(1+\|S_{im}\|_2)^{2+\delta}\mid Z_i=t\}<\infty$.
Finally, $h\to0$ and $nh^q\to\infty$.
\end{assumption}

\begin{assumption}[Stable nuisance estimation]
\label{ass:nuisance}
The labeled fraction satisfies $\pi=n/M\to\pi_0\in(0,1]$. The local ridge coefficient obeys
$\|\widehat\beta_m(z)-\beta_m^\star(z)\|=o_p(1)$. Under
Assumption~\ref{ass:regularity}, this holds, for example, when the nuisance
bandwidth satisfies $b\ge h$, $b\to0$, and $nb^q\to\infty$, the ridge level
satisfies $\lambda\to0$, and $\Sigma_{SS,m}(z)$ is nonsingular.
\end{assumption}

For discrete profiling groups, the estimator admits an exact finite-sample statement under random
benchmark sampling and random label subsampling.

\begin{theorem}[Unbiasedness and variance for group profiles]
\label{thm:finite}
Fix a group $g$. Suppose $T_g$ contains $M_g$ i.i.d.\ items from the population conditional on $Z=g$
and, conditional on the realized item values in $T_g$, $L_g$ is uniformly distributed over all
size-$n_g$ subsets of $T_g$, where $0<n_g\le M_g$. For any fixed coefficient $b_g$, the group-weight
version of $\widehat\theta_m(g;b_g)$ in~\eqref{eq:lace-estimator} satisfies
\[
\mathbb E\{\widehat\theta_m(g;b_g)\}
=
\theta_m(g).
\]
In this theorem, expectation and variance are over both benchmark sampling and label sampling. If
$\mathbb E(\|S_{im}\|_2^2\mid Z_i=g)<\infty$ and $M_g>1$, then
\[
\Var\{\widehat\theta_m(g;b_g)\}
=
\frac{\sigma^2_{Y,m}(g)}{M_g}
+
\left(1-\frac{n_g}{M_g}\right)\frac{\sigma^2_{R,g}(b_g)}{n_g},
\]
where $\sigma^2_{Y,m}(g)=\Var(Y_{im}\mid Z_i=g)$,
$R_i(b_g)=Y_{im}-b_g^\top S_{im}$, and
$\sigma^2_{R,g}(b_g)=\Var\{R_i(b_g)\mid Z_i=g\}$.
\end{theorem}

\paragraph{Two exact oracle gains for group profiles.}
Let $\pi_g=n_g/M_g$. At the respective variance-minimizing population and
realized-pool coefficients, when the corresponding outcome variances are
positive, Theorem~\ref{thm:finite} and its conditional
variance calculation give
\[
\underbrace{\frac{1}{1-(1-\pi_g)R_m^2(g)}}_{\substack{\text{population target }\theta_m(g)\\
\text{randomness in }(T_g,L_g)}}
\le \frac{1}{\pi_g},
\qquad
\underbrace{\frac{1}{1-R_{T,g}^{2,\mathrm d}}}_{\substack{\text{realized-pool target }\bar Y_{T,g}\\
\text{randomness in }L_g\mid T_g}}.
\]
Here $R_m^2(g)$ is the population squared multiple correlation and
$R_{T,g}^{2,\mathrm d}$ is its counterpart computed from realized-pool
covariances. In the second ratio, the common finite-population correction
$1-\pi_g$ cancels. This second identity is an exact group-profile oracle
comparison, not a pointwise gain theorem for the continuous profiles in the
real-data study or an equality claimed for the feasible estimator.

\begin{theorem}[First-order adaptivity to the nuisance coefficient]
\label{thm:adaptivity}
For the estimator $\widehat\theta_m(z;b)$ in~\eqref{eq:lace-estimator} and any possibly
data-dependent $\widehat\beta_m(z)$,
\[
\widehat\theta_m(z;\widehat\beta_m(z))-\widehat\theta_m(z;\beta_m^\star(z))
=
-
\{\widehat\beta_m(z)-\beta_m^\star(z)\}^\top
\{\bar S_{m,L}(z)-\bar S_{m,T}(z)\}.
\]
Under Assumption~\ref{ass:regularity}, $\|\bar S_{m,L}(z)-\bar S_{m,T}(z)\|=O_p((nh^q)^{-1/2})$ by standard kernel central limit arguments \citep{fan1996local}.
Consequently, if $\|\widehat\beta_m(z)-\beta_m^\star(z)\|=o_p(1)$, then $\sqrt{nh^q}\{\widehat\theta_m(z;\widehat\beta_m(z))-\widehat\theta_m(z;\beta_m^\star(z))\}=o_p(1)$.
\end{theorem}

\begin{theorem}[Local oracle optimality of the augmentation]
\label{thm:oracle-optimality}
Under Assumption~\ref{ass:regularity} and $n/M\to\pi_0\in(0,1]$, for any fixed deterministic
$b\in\mathbb R^K$, the estimator $\widehat\theta_m(z;b)$ in~\eqref{eq:lace-estimator} satisfies
\[
\sqrt{nh^q}
\{\widehat\theta_m(z;b)-\theta_m(z)-B_h(z)\}
\rightsquigarrow
\mathcal N\!\left(0,\frac{R(K)}{f_Z(z)}\mathcal V_m(z;b)\right),
\]
where
\[
B_h(z)
=
\frac{\mathbb E\{K_h(Z_i-z)\theta_m(Z_i)\}}{\mathbb E\{K_h(Z_i-z)\}}
-\theta_m(z)
\]
is the usual local-constant smoothing bias and
\[
\mathcal V_m(z;b)=
\sigma^2_{Y,m}(z)
+(1-\pi_0)
\{b^\top\Sigma_{SS,m}(z)b-2b^\top\Sigma_{SY,m}(z)\}.
\]
If $\pi_0<1$, the minimizers are exactly the solutions of
$\Sigma_{SS,m}(z)b=\Sigma_{SY,m}(z)$, and the minimum first-order variance factor is
$\mathcal V_m(z;\beta_m^\star)=
\sigma^2_{Y,m}(z)\{1-(1-\pi_0)R_m^2(z)\}$.
\end{theorem}

\begin{theorem}[Local semisupervised efficiency gain]
\label{thm:gain}
Under Assumptions~\ref{ass:regularity}--\ref{ass:nuisance}, for an interior continuous point $z$, the feasible estimator $\widehat\theta_m^{\mathrm{LACE}}(z)$ in~\eqref{eq:lace-feasible} satisfies
\[
\sqrt{nh^q}
\{\widehat\theta_m^{\mathrm{LACE}}(z)-\theta_m(z)-B_h(z)\}
\rightsquigarrow
\mathcal N\!\left(
0,\,
\frac{R(K)}{f_Z(z)}
\sigma^2_{Y,m}(z)\{1-(1-\pi_0)R_m^2(z)\}
\right).
\]
The naive smoother has the same limit with variance factor $\{R(K)/f_Z(z)\}\sigma^2_{Y,m}(z)$, so the asymptotic efficiency gain is
\[
\operatorname{Gain}_m(z)
=
\frac{1}{1-(1-\pi_0)R_m^2(z)}.
\]
For finite-sample simulation comparisons, we replace $\pi_0$ by the observed labeled fraction $\pi=n/M$.
For a fixed discrete group $g$ with $p_g=\mathbb P(Z_i=g)>0$, the within-group estimator satisfies the same statement with $h^q f_Z(z)$ replaced by $p_g$, $R(K)$ replaced by $1$, and $B_h(g)=0$.
\end{theorem}

The gain formula is local. If cheap auxiliary signals have no local explanatory power,
$R_m^2(z)=0$ and \method{} matches the gold-label smoother to first order. If they explain nearly all
local correctness variation, $R_m^2(z)\approx1$ and the maximum gain approaches $1/\pi$.

\begin{theorem}[Extensions to gaps and deployment scores]
\label{thm:extensions}
For a model pair $(a,b)$, suppose Assumptions~\ref{ass:regularity}--\ref{ass:nuisance} hold after replacing $(Y_{im},S_{im})$ by $(D_{i,ab},H_{i,ab})$.
Let $\sigma^2_{\Delta,ab}(z)=\Var(D_{i,ab}\mid Z_i=z)$, $\Sigma_{HH,ab}(z)=\Var(H_{i,ab}\mid Z_i=z)$, $\Sigma_{HD,ab}(z)=\Cov(H_{i,ab},D_{i,ab}\mid Z_i=z)$, $R^2_{\Delta,ab}(z)=\Sigma_{HD,ab}(z)^\top\Sigma_{HH,ab}(z)^\dagger\Sigma_{HD,ab}(z)/\sigma^2_{\Delta,ab}(z)$, and let $B_{\Delta,h}(z)$ be the analog of $B_h(z)$ in Theorem~\ref{thm:oracle-optimality} with $\theta_m$ replaced by $\Delta_{ab}$.
Then the gap estimator in~\eqref{eq:lace-gap} satisfies
\[
\sqrt{nh^q}\{\widehat\Delta_{ab}^{\mathrm{LACE}}(z)-\Delta_{ab}(z)-B_{\Delta,h}(z)\}
\rightsquigarrow
\mathcal N\!\left(0,\frac{R(K)}{f_Z(z)}\sigma^2_{\Delta,ab}(z)\{1-(1-\pi_0)R^2_{\Delta,ab}(z)\}\right).
\]

For the deployment estimator in~\eqref{eq:lace-deployment}, consider the covariate-shift law
$Q(dz,dy,ds)=Q_Z(dz)P(dy,ds\mid z)$. Suppose $r_Q$ is known and bounded with
$\mathbb E_P r_Q(Z_i)=1$, $K$ is fixed, $n/M\to\pi_0\in(0,1]$, and
$\lambda_{Q,n}\to0$. Let $\Psi_m(Q)=\mathbb E_Q(Y_{im})$.
Define the influence functions
\[
\phi^Q_{Y,i}=r_Q(Z_i)\{Y_{im}-\Psi_m(Q)\},
\qquad
\phi^Q_{S,i}=r_Q(Z_i)\{S_{im}-\mathbb E_Q(S_{im})\},
\]
together with $\Gamma^Q_{SS,m}=\Var_P(\phi^Q_{S,i})$, $\Gamma^Q_{SY,m}=\Cov_P(\phi^Q_{S,i},\phi^Q_{Y,i})$, $\beta^\star_{Q,m}=(\Gamma^Q_{SS,m})^\dagger\Gamma^Q_{SY,m}$, and $R^2_{Q,m}=(\Gamma^Q_{SY,m})^\top(\Gamma^Q_{SS,m})^\dagger\Gamma^Q_{SY,m}/\Var_P(\phi^Q_{Y,i})$.
Assume $\mathbb E_Q\|S_{im}\|_2<\infty$, $\Var_P(\phi^Q_{Y,i})<\infty$, and
$\mathbb E_P\|\phi^Q_{S,i}\|_2^2<\infty$.
Then the influence-moment ridge coefficient defined above satisfies
$\widehat\beta_m(Q)\xrightarrow{p}\beta^\star_{Q,m}$, and
\[
\sqrt n\{\widehat\Psi_m^{\mathrm{LACE}}(Q)-\Psi_m(Q)\}
\rightsquigarrow
\mathcal N\big(0,\Var_P(\phi^Q_{Y,i})\{1-(1-\pi_0)R^2_{Q,m}\}\big).
\]
\end{theorem}

\section{Experiments}
\label{sec:experiments}

We evaluate \method{} on eight real benchmarks across three models.
Appendix~\ref{app:sim} separates three controlled analyses: discrete-profile
mechanism stress tests, a continuous shrinking-bandwidth diagnostic of the
local gain formula, and a repeated-pool comparison of practical estimators.
Together they cover continuous, ordered discrete, and unordered categorical
profiles under known data-generating processes.

\subsection{Setup}

The three evaluated candidate models are Claude Haiku 3, Ministral 3B,
and Qwen3 32B. Claude Opus 4.6 supplies the judge-derived measurements. For
each candidate, the three pairwise anchors are Ministral 3B, Claude Haiku 3,
and Claude Opus 4.6; each pairwise signal averages the two response orderings.

The benchmarks and profiling axes are listed in Table~\ref{tab:benchmarks}.
For each model $m$ and item $i$, we construct
$S_{im}=(S_{im}^{\mathrm{judge}},S_{im}^{\mathrm{pair},1},S_{im}^{\mathrm{pair},2},
S_{im}^{\mathrm{pair},3},S_{im}^{\mathrm{conf}},S_{im}^{\mathrm{disagree}})^\top$:
$S_{im}^{\mathrm{judge}}=3^{-1}\sum_{r=1}^3G_{imr}$ averages pointwise judge scores across prompt variants;
$S_{im}^{\mathrm{pair},h}=(w_{imh}+0.5t_{imh})/2$ summarizes wins and ties against anchor model $h$;
$S_{im}^{\mathrm{conf}}$ is self-reported confidence;
and $S_{im}^{\mathrm{disagree}}=3^{-1}\sum_{r=1}^3(G_{imr}-\bar G_{im})^2$.
Each component is standardized within each benchmark/model pair using all $M$ items in $T$.
The exact prompts are in Appendix~\ref{app:prompts}.

We evaluate eight benchmarks spanning mathematical reasoning,
commonsense, science, and general knowledge. Although full labels are
available for evaluation, estimators are fit only on a sampled labeled set
$L$. For each benchmark--model cell we draw $B=100$ matched random
permutations and use the nested prefixes of sizes
$n_{\mathrm{lab}}\in\{50,100,200\}$; every method receives the same labeled
sets.

Our experiment protocol uses $\lambda=0.3$, profile bandwidth
$h=h_0=1.5\times1.06\,\widehat{\operatorname{sd}}(Z)M^{-1/5}$,
coefficient-fitting bandwidth $b=h$, and logistic regularization $C=1$.
These values are shared across all benchmarks, candidates, budgets, and
splits. The main results use no per-cell tuning or data-dependent
hyperparameter fallback.
The fixed ridge is a finite-sample stabilization choice rather than the
vanishing-ridge sequence that provides one sufficient route to
Assumption~\ref{ass:nuisance}; its empirical robustness is examined in
Appendix~\ref{app:robustness}.

\begin{table}[ht]
\caption{Benchmarks and profiling axes used in the main experiments.}
\label{tab:benchmarks}
\centering\small
\begin{tabular}{llrl}
\toprule
Benchmark & Profiling axis $Z$ & $M$ & Domain \\
\midrule
MATH-500 & difficulty level ($1,\ldots,5$) & 500 & math reasoning \\
ScienceQA & grade level ($2,\ldots,12$) & 500 & science (text-only) \\
MMLU & question token length & 500 & general knowledge \\
WinoGrande & sentence length (tokens) & 500 & commonsense \\
HellaSwag & context sentence length & 500 & commonsense \\
TruthfulQA & question token length & 500 & factuality \\
GSM8K & solution steps ($3,\ldots,10$) & 500 & grade-school math \\
ARC & question token length & 500 & science (challenge) \\
\bottomrule
\end{tabular}
\end{table}

\subsection{Estimators compared}
\label{estimators_compared}
All methods use the same labeled split, full-pool target, and evaluation
grid. They use the localization prescribed by each method; in particular,
PPCI uses RKHS weights and StratPPI uses strata.
All fitted logistic predictors standardize their covariates and use
$C=1$. For prediction-based rectifiers, predictions at labeled items are
two-fold out of fold, while predictions at the remaining pool items come from
the model refit on all of $L$. Here and below, ``CV'' in an estimator
name denotes control variate, whereas ``cross-validation'' refers to
hyperparameter selection. Scalar prediction control variate (CV) uses
the unregularized local covariance--variance coefficient.

\textbf{Naive}: $\widehat\theta_m^{\mathrm{naive}}(z)=\bar Y_{m,L,h}(z)$, the local weighted mean of gold labels in $L$.

\textbf{Plug-in (judge)}: fits a logistic regression of $Y_{im}$ on $S_{im}^{\mathrm{judge}}$ using $L$; estimates $\theta_m(z)$ as $\sum_{i\in T}w_{i,T,h}(z)\widehat p_m(S_{im}^{\mathrm{judge}})$.

\textbf{Plug-in (multi)}: fits a ridge logistic regression of $Y_{im}$ on $(Z_i, S_{im})$ using $L$; estimates $\theta_m(z)$ as $\sum_{i\in T}w_{i,T,h}(z)\widehat p_m(Z_i,S_{im})$.

\textbf{Aug.\ plug-in}: $\sum_{i\in T}w_{i,T,h}(z)\widehat p_m(Z_i,S_{im})+\sum_{i\in L}w_{i,L,h}(z)\{Y_{im}-\widehat p_m(Z_i,S_{im})\}$; adds a labeled residual correction to reduce bias from model misspecification.

\textbf{StratPPI} \citep{fisch2024stratppi}: fits a
two-fold cross-fitted logistic composite signal from $(S,Z)$, divides the
range of $Z$ into five equal-width strata, applies a scalar PPI correction
within each stratum, and linearly interpolates the stratum estimates. This is
an adaptation of StratPPI to our setting: the original estimand is a
stratum-weighted global mean rather than a continuous profile.
When either the labeled subset or the full pool contains fewer than two
items in a stratum, we use the labeled stratum mean, or $0.5$ if the stratum
contains no labeled item; this deterministic sparse-stratum rule involves no
tuning.

\textbf{Scalar prediction CV} \citep{angelopoulos2023ppiplus}: $\bar Y_{m,L,h}(z)-\widehat\eta_m(z)\{\bar P_{m,L,h}(z)-\bar P_{m,T,h}(z)\}$, where $\widehat\eta_m(z)$ is a locally optimized scalar and $P_{im}$ is a logistic predictor of $Y_{im}$ from $S_{im}$ fit on $L$.

\textbf{Local PPI} \citep{gu2024localppi}: applies the original
local-linear predictor and rectifier equations of \citet{gu2024localppi},
using the fixed raw judge-mean score as the supplied external predictor.

\textbf{PPCI} \citep{sui2026ppci}: applies the prediction-plus-residual
correction with Mat\'ern-$5/2$ RKHS weights cross-fitted over the fixed pool
$T$, the localized residual and prediction terms averaged over the $n$ labeled
and the pool items before dividing by the localized Jacobian, and uses
Tikhonov regularization selected at the corner of an L-curve defined by the
RKHS residual norm and a variance proxy. We report it as a finite-pool
adaptation, not the original algorithm.

\textbf{Global CV}: \method{} with a single coefficient $\widehat\beta_m^{\mathrm{glob}}$ estimated by pooling all $z$, rather than fitting locally.

\textbf{Per-signal CV}: \method{} restricted to the primary judge signal $S_{im}^{\mathrm{judge}}$ with a scalar local coefficient; isolates the gain from using $K$ signals jointly.

\textbf{Residual-only}: $\bar Y_{m,L,h}(z)-\widehat\beta_m(z)^\top\bar S_{m,L,h}(z)$, i.e., \method{} without the full-pool term $\widehat\beta_m(z)^\top\bar S_{m,T,h}(z)$; isolates the contribution of unlabeled signal means.

\textbf{\method{} (oracle $\beta^\star$)}: true $\beta_m^\star(z)$ injected in place of $\widehat\beta_m(z)$; an efficiency upper bound (simulation only).

\textbf{\method{} (feasible)}: full \method{} with local ridge coefficient $\widehat\beta_m(z)$ and all $K$ signals.

The empirical comparisons below concern the primary
performance-profile estimand; the paired-gap and deployment-weighted
constructions of Section~\ref{sec:gaps} are theoretical extensions and are
not presented as separate empirical experiments.

\subsection{Results}

\paragraph{Metrics.}
For every cell, the evaluation target is the full-pool gold profile
$\widehat\theta_{m,T}(z)=\sum_{i\in T}w_{i,T,h_0}(z)Y_{im}$, computed with
the Gaussian kernel and bandwidth $h_0$ on a fixed grid of $G=20$ points
spanning the 5th--95th percentiles of $Z$. Across repetitions, $T$ is
fixed and only $L$ is resampled, so MSE and RE evaluate the realized target;
the exact group formulas following Theorem~\ref{thm:finite} distinguish it
from the population target. Every fitted profile is
evaluated on the same grid. For split $r$, profile MSE is the unweighted grid
average
$\operatorname{MSE}_{m,r}=\frac{1}{G}\sum_{g=1}^{G}
\{\widehat\theta_{m,r}(z_g)-\widehat\theta_{m,T}(z_g)\}^2$.
Cellwise relative efficiency is the ratio of paired mean MSEs,
$\operatorname{RE}_{m,n}=
\{\frac1B\sum_r\operatorname{MSE}^{\mathrm{naive}}_{m,r}\}/
\{\frac1B\sum_r\operatorname{MSE}_{m,r}\}$; values above one favor the
method over the gold-label-only estimator. The reported 95\% Monte Carlo
intervals use paired split-level delta-method influence values; aggregate
intervals combine cellwise log-RE variances across independent cells, and
aggregates that pool the three nested label budgets treat the budgets within
each cell jointly, with the label-split permutation as the joint resampling
unit, so that the cross-budget covariances induced by the nested prefixes are
included. Full-pool
gold outcomes are used only to evaluate MSE, never to fit an estimator. We
evaluate the corresponding superpopulation gain formula only in controlled
simulations, where the population profile and local $R_m^2(z)$ are analytically
known; Figure~\ref{fig:thm4-pointwise} provides its pointwise validation.

{\color{black}
\begin{table}[t]
\centering\scriptsize
\setlength{\tabcolsep}{3pt}
\caption{Overall results. Entries are geometric-mean RE $[95\%\ \mathrm{MC\ CI}]$ across the 24 benchmark--model cells at each label budget; Overall aggregates all 72 cells. Each cell is a ratio of mean MSEs over the same $B=100$ label splits. Intervals apply the paired delta method within each cell; budget columns combine cellwise log-RE variances across independent cells, and the Overall column treats the three nested budgets within each cell jointly, using the label-split permutation as the joint resampling unit, so its variance includes the cross-budget covariances induced by the nested prefixes. Target: realized full-pool gold kernel profile; randomness: matched $L$-draws conditional on $T$.}
\label{tab:real:aggregate}
\resizebox{\textwidth}{!}{%
\begin{tabular}{lcccc}
\toprule
Estimator & $n_{\mathrm{lab}}=50$ & $n_{\mathrm{lab}}=100$ & $n_{\mathrm{lab}}=200$ & Overall \\
\midrule
Naive & $1.000\,[1.000,1.000]$ & $1.000\,[1.000,1.000]$ & $1.000\,[1.000,1.000]$ & $1.000\,[1.000,1.000]$ \\
Plug-in (judge) & $3.128\,[2.342,4.177]$ & $2.780\,[2.398,3.223]$ & $2.093\,[1.863,2.351]$ & $2.630\,[2.270,3.047]$ \\
Plug-in (multi) & $3.214\,[2.438,4.236]$ & $2.876\,[2.366,3.497]$ & $2.006\,[1.770,2.272]$ & $2.647\,[2.277,3.078]$ \\
Aug. plug-in & $1.060\,[0.798,1.406]$ & $1.227\,[1.090,1.382]$ & $1.502\,[1.302,1.734]$ & $1.250\,[1.098,1.423]$ \\
Global CV & $1.438\,[1.383,1.494]$ & $1.568\,[1.507,1.633]$ & $1.703\,[1.612,1.799]$ & $1.566\,[1.516,1.618]$ \\
Per-signal CV & $0.891\,[0.873,0.909]$ & $1.118\,[1.103,1.134]$ & $1.175\,[1.160,1.190]$ & $1.054\,[1.042,1.066]$ \\
Residual-only & $1.185\,[1.097,1.279]$ & $1.049\,[0.900,1.224]$ & $0.597\,[0.543,0.656]$ & $0.906\,[0.833,0.985]$ \\
Scalar prediction CV & $0.772\,[0.665,0.896]$ & $1.236\,[1.140,1.339]$ & $1.480\,[1.371,1.597]$ & $1.122\,[1.050,1.199]$ \\
Local PPI & $0.590\,[0.564,0.617]$ & $0.664\,[0.647,0.683]$ & $0.610\,[0.593,0.627]$ & $0.621\,[0.607,0.636]$ \\
PPCI & $1.171\,[0.983,1.393]$ & $1.262\,[1.119,1.423]$ & $1.347\,[1.182,1.535]$ & $1.258\,[1.143,1.384]$ \\
StratPPI & $0.676\,[0.581,0.788]$ & $0.850\,[0.727,0.994]$ & $0.863\,[0.794,0.939]$ & $0.792\,[0.726,0.864]$ \\
\method{} (feasible) & $\mathbf{5.400}\,[4.997,5.836]$ & $\mathbf{5.761}\,[5.321,6.238]$ & $\mathbf{5.431}\,[4.895,6.025]$ & $\mathbf{5.528}\,[5.181,5.898]$ \\
\bottomrule
\end{tabular}%
}
\end{table}
}

Across the 24 benchmark--candidate cells and three label budgets,
\method{} attains an overall geometric-mean RE of
$5.528\,[5.181,5.898]$, compared with
$2.647\,[2.277,3.078]$ for the strongest baseline overall, the multivariate
plug-in. Its budget-specific RE is $5.400$, $5.761$, and $5.431$ for
$n_{\mathrm{lab}}=50,100,200$, respectively, and it is the best-performing
method in all 72 cells. Audited cell-level RE, Monte Carlo intervals, and
mean MSEs are reported in Appendix~\ref{app:exp-tables}.

\section{Conclusion}

Cheap auxiliary signals need not be calibrated to be useful. They need only explain local variation in gold
correctness. \method{} turns that observation into a conditional estimator with a simple
implementation, an oracle-optimal local augmentation, and an interpretable gain formula.
Our experiments establish the result for data-efficient performance
profiles; the theory gives corresponding extensions to direct paired model
gaps and deployment-weighted scoring under the same statistical principle.

The scope extends naturally beyond LLM evaluation. The key ingredients are a scarce gold outcome
$Y_i$, abundant cheap proxy signals $S_i$ observable for all items, and a metadata covariate $Z_i$
defining the conditional structure of interest. In \emph{clinical prediction}, gold outcomes might
be long-term patient endpoints while cheap signals are short-term biomarkers, stratified by disease
severity. In \emph{online experimentation}, long-term user retention is expensive to observe while
short-term engagement proxies are immediate, stratified by user segment. In \emph{scientific
benchmarking} beyond NLP---computer vision, speech recognition, code generation---expensive human
annotations can be augmented with automated metrics stratified by input type or difficulty. In each
case, the local centering principle applies without requiring the cheap proxy to be calibrated, and
at the population level, the local $R^2_m(z)$ characterizes which strata can benefit.

\paragraph{Semantic profile spaces.}
Semantic similarity may define an alternative profile, but then it also defines an alternative
estimand. A prespecified low-dimensional embedding projection can serve as a continuous $Z$, while a
prespecified locality-sensitive-hashing (LSH) bucket can serve as a categorical $Z$; the existing
weights then apply unchanged. Direct smoothing in the original high-dimensional embedding space is
outside the present theory. At the population level, $R_m^2(z)$ characterizes the potential
usefulness of cheap signals for a chosen profile; it does not select the profile.

\paragraph{Cascaded raters.}
The same centered construction suggests a multilevel design when cheap signals are available on the
full pool, a more expensive rater is applied to a random subset, and human labels are collected on a
random nested subset. Appendix~\ref{app:cascade} records this algebraic extension. We treat it as a
discussion direction rather than a new main contribution and do not claim a complete optimality or
adaptivity theory for the cascade.

\paragraph{Limitations.}
\method{} improves efficiency only when cheap auxiliary signals have local explanatory power. If a cheap
judge is uninformative or adversarial within a subgroup, the local $R^2$ will be small and the method
should not be expected to help. Finite-sample covariance estimation can also be unstable when the
number of labels in a group is small relative to the number of cheap signals; ridge regularization
and bandwidth enlargement mitigate but do not remove this issue. The asymptotic theory allows an
estimated nuisance coefficient under stability conditions, but small groups can still violate the
large-neighborhood approximation and should be diagnosed empirically.
For a continuous $q$-dimensional profile, the effective local sample size is of order
$nh^q$, so local smoothing inherits the curse of dimensionality. Our theory treats fixed,
low-dimensional $q$ and does not cover a profile dimension that grows with sample size or direct
smoothing in a raw high-dimensional embedding space.
A possible negative use is over-reliance on cheap auxiliary signals
in settings where gold labels are biased, underspecified, or too sparse to detect local failures. We
therefore report theoretical gain curves computed from population $R^2$ in controlled
simulations, sensitivity to ridge and bandwidth choices, and full-label oracle comparisons.

\bibliographystyle{plainnat}
\bibliography{references}

\appendix

\section{Technical Lemmas}
\label{app:lemmas}

This appendix collects elementary linear-algebraic facts about positive
semidefinite covariances and Moore--Penrose pseudoinverses that are used
repeatedly in the proofs of Section~\ref{sec:theory}.

\begin{lemma}
\label{lem:range}
Let $X\in\mathbb R^K$ and $Y\in\mathbb R$ be random variables on a common probability space with $\mathbb E\|X\|_2^2<\infty$ and $\mathbb E Y^2<\infty$.
Let $\Sigma_{XX}=\Var(X)$, $\Sigma_{XY}=\Cov(X,Y)$, and $f(b)=b^\top\Sigma_{XX}b-2b^\top\Sigma_{XY}$ for $b\in\mathbb R^K$.
Write $\Sigma_{XX}^\dagger$ for the Moore--Penrose pseudoinverse of $\Sigma_{XX}$, $\mathrm{range}(\Sigma_{XX})=\{\Sigma_{XX}v:v\in\mathbb R^K\}$ for its column space, and $\ker(\Sigma_{XX})=\{v\in\mathbb R^K:\Sigma_{XX}v=0\}$ for its null space.
Then:
\begin{enumerate}[label=(\roman*),itemsep=0pt]
\item $\Sigma_{XY}\in\mathrm{range}(\Sigma_{XX})$.
\item $f$ is convex and attains a finite minimum on the affine set $\beta^\star+\ker(\Sigma_{XX})$, with $\beta^\star=\Sigma_{XX}^\dagger\Sigma_{XY}$ the unique minimum-Euclidean-norm minimizer.
\item $f(\beta^\star)=-\Sigma_{XY}^\top\Sigma_{XX}^\dagger\Sigma_{XY}$.
\item Among all deterministic $b\in\mathbb R^K$, $\Var(Y-b^\top X)$ is minimized exactly on the same solution set, with minimum value $\Var(Y)-\Sigma_{XY}^\top\Sigma_{XX}^\dagger\Sigma_{XY}$.
\end{enumerate}
\end{lemma}

\begin{proof}
For (i), symmetry of $\Sigma_{XX}$ gives the orthogonal decomposition $\mathbb R^K=\mathrm{range}(\Sigma_{XX})\oplus\ker(\Sigma_{XX})$.
Fix any $v\in\ker(\Sigma_{XX})$.
Then $\Var(v^\top X)=v^\top\Sigma_{XX}v=0$, so $v^\top X$ is almost surely constant, hence $v^\top\Sigma_{XY}=\Cov(v^\top X,Y)=0$.
Thus $\Sigma_{XY}\perp\ker(\Sigma_{XX})$, which gives $\Sigma_{XY}\in\mathrm{range}(\Sigma_{XX})$.

For (ii), the Hessian of $f$ is $2\Sigma_{XX}\succeq0$, so $f$ is convex.
By (i), the equation $\Sigma_{XX}b=\Sigma_{XY}$ has at least one solution.
Pick any one and call it $\beta_0$.
For any $b\in\mathbb R^K$, $b$ solves $\Sigma_{XX}b=\Sigma_{XY}$ if and only if $\Sigma_{XX}(b-\beta_0)=0$, i.e., $b-\beta_0\in\ker(\Sigma_{XX})$, so the solution set equals $\beta_0+\ker(\Sigma_{XX})$.
Each such solution is a stationary point of the convex $f$, hence a global minimizer.
Among solutions, $\Sigma_{XX}^\dagger\Sigma_{XY}$ lies in $\mathrm{range}(\Sigma_{XX})$, which is orthogonal to $\ker(\Sigma_{XX})$, so adding any nonzero element of $\ker(\Sigma_{XX})$ strictly increases its Euclidean norm, identifying $\beta^\star=\Sigma_{XX}^\dagger\Sigma_{XY}$ as the unique minimum-Euclidean-norm minimizer.

For (iii), if $\Sigma_{XX}\beta^\star=\Sigma_{XY}$, then $\beta^{\star\top}\Sigma_{XX}\beta^\star=\beta^{\star\top}\Sigma_{XY}$, so $f(\beta^\star)=-\beta^{\star\top}\Sigma_{XY}$.
Choosing $\beta^\star=\Sigma_{XX}^\dagger\Sigma_{XY}$ and using $(\Sigma_{XX}^\dagger)^\top=\Sigma_{XX}^\dagger$ for symmetric $\Sigma_{XX}$ gives the stated value.

For (iv), $\Var(Y-b^\top X)=\Var(Y)+f(b)$, so the minimizers coincide with those of $f$ and the minimum value is $\Var(Y)+f(\beta^\star)=\Var(Y)-\Sigma_{XY}^\top\Sigma_{XX}^\dagger\Sigma_{XY}$ by (iii).
\end{proof}

\section{Proofs for Section~\ref{sec:theory}}
\label{app:proof}

\subsection{Proof of Proposition~\ref{prop:identity}}
For the identification claim, since $b(z)$ is deterministic, linearity of expectation gives
\[
\mathbb E\!\left[
b(z)^\top\{S_{im}-\mu_{S,m}(z)\}\mid Z_i=z
\right]
=
b(z)^\top
\left\{
\mathbb E(S_{im}\mid Z_i=z)-\mu_{S,m}(z)
\right\}
=0.
\]
Thus subtracting the centered cheap-signal term leaves the conditional mean of $Y_{im}$ unchanged.

For the optimality claim, note that the conditional variance of $Y_{im}-b^\top\{S_{im}-\mu_{S,m}(z)\}$ given $Z_i=z$ equals $\Var(Y_{im}-b^\top S_{im}\mid Z_i=z)$, since $b^\top\mu_{S,m}(z)$ is non-random.
Apply Lemma~\ref{lem:range} conditionally on $Z_i=z$ with $X=S_{im}$ and $Y=Y_{im}$.
The moment hypotheses hold by the assumed $\mathbb E(\|S_{im}\|_2^2\mid Z_i=z)<\infty$ and $Y_{im}\in\{0,1\}$.
Part~(i) gives $\Sigma_{SY,m}(z)\in\mathrm{range}(\Sigma_{SS,m}(z))$, so the normal equations $\Sigma_{SS,m}(z)b=\Sigma_{SY,m}(z)$ have at least one solution.
Part~(iv) identifies this solution set with the minimizers of the conditional variance, and part~(ii) singles out $\beta_m^\star(z)=\Sigma_{SS,m}(z)^\dagger\Sigma_{SY,m}(z)$ as the minimum-norm element.

\subsection{Proof of Theorem~\ref{thm:finite}}
Let $\bar Y_{T,g}=M_g^{-1}\sum_{i\in T_g}Y_{im}$. Rewrite
\[
\widehat\theta_m(g;b_g)
=
\bar R_{L,g}(b_g)+b_g^\top\bar S_{T,g},
\qquad
R_i(b_g)=Y_{im}-b_g^\top S_{im}.
\]
Here $\bar R_{L,g}(b_g)=n_g^{-1}\sum_{i\in L_g}R_i(b_g)$ and
$\bar S_{T,g}=M_g^{-1}\sum_{i\in T_g}S_{im}$.
Conditional on $T_g$, $\bar R_{L,g}(b_g)$ is the mean of a simple random sample without replacement from $\{R_i(b_g):i\in T_g\}$, so $\mathbb E_L\{\bar R_{L,g}(b_g)\mid T_g\}=M_g^{-1}\sum_{i\in T_g}R_i(b_g)$.
Since $b_g^\top\bar S_{T,g}$ is $T_g$-measurable and $M_g^{-1}\sum_{i\in T_g}R_i(b_g)+b_g^\top\bar S_{T,g}=\bar Y_{T,g}$,
\[
\mathbb E_L\{\widehat\theta_m(g;b_g)\mid T_g\}
=
\bar Y_{T,g}.
\]
Taking expectation over the i.i.d.\ benchmark draw gives
$\mathbb E\{\widehat\theta_m(g;b_g)\}=\mathbb E(\bar Y_{T,g})=\theta_m(g)$.

For the variance, let
$S^2_{R,T,g}(b_g)=(M_g-1)^{-1}\sum_{i\in T_g}\{R_i(b_g)-\bar R_{T,g}(b_g)\}^2$, where
$\bar R_{T,g}(b_g)=M_g^{-1}\sum_{i\in T_g}R_i(b_g)$. Conditional on $T_g$,
\[
\Var_L\{\widehat\theta_m(g;b_g)\mid T_g\}
=
\left(1-\frac{n_g}{M_g}\right)\frac{S^2_{R,T,g}(b_g)}{n_g}.
\]
This identity also gives the realized-pool oracle ratio used in the
Metrics paragraph. When the finite-pool variance $S^2_{Y,T,g}$ of $Y$ is
positive, let $b_{T,g}^{\mathrm d}$ minimize $S^2_{R,T,g}(b)$. Writing
$R_{T,g}^{2,\mathrm d}=1-S^2_{R,T,g}(b_{T,g}^{\mathrm d})/
S^2_{Y,T,g}$, the common factor $(1-n_g/M_g)/n_g$ cancels between the naive
and oracle conditional variances, giving
$1/(1-R_{T,g}^{2,\mathrm d})$. Under the joint law, instead take the
population minimizer $\beta_m^\star(g)$, for which
$\sigma^2_{R,g}(\beta_m^\star)=\sigma^2_{Y,m}(g)
\{1-R_m^2(g)\}$. Substitution into the variance formula of
Theorem~\ref{thm:finite}, with $\pi_g=n_g/M_g$, gives the oracle ratio
$1/\{1-(1-\pi_g)R_m^2(g)\}\le 1/\pi_g$.
By the law of total variance,
\[
\Var\{\widehat\theta_m(g;b_g)\}
=
\Var(\bar Y_{T,g})
+
\mathbb E\!\left[\Var_L\{\widehat\theta_m(g;b_g)\mid T_g\}\right].
\]
Since $T_g$ contains $M_g$ i.i.d.\ conditional draws, $\Var(\bar Y_{T,g})=\sigma^2_{Y,m}(g)/M_g$ and
$\mathbb E\{S^2_{R,T,g}(b_g)\}=\sigma^2_{R,g}(b_g)$, which yields the stated variance.

\subsection{Proof of Theorem~\ref{thm:adaptivity}}
The estimator is affine in $b$:
\[
\widehat\theta_m(z;b)
=
\bar Y_{m,L}(z)-b^\top\{\bar S_{m,L}(z)-\bar S_{m,T}(z)\}.
\]
Subtracting the same expression evaluated at $b=\beta_m^\star(z)$ gives
\[
\widehat\theta_m(z;\widehat\beta_m(z))-\widehat\theta_m(z;\beta_m^\star(z))
=
-
\{\widehat\beta_m(z)-\beta_m^\star(z)\}^\top
\{\bar S_{m,L}(z)-\bar S_{m,T}(z)\}.
\]
The rate statement follows immediately from Cauchy--Schwarz.
Specifically, under the stated conditions,
\[
\left|
\widehat\theta_m(z;\widehat\beta_m(z))-\widehat\theta_m(z;\beta_m^\star(z))
\right|
\le
\|\widehat\beta_m(z)-\beta_m^\star(z)\|\,
\|\bar S_{m,L}(z)-\bar S_{m,T}(z)\|
=
o_p((nh^q)^{-1/2}),
\]
so the estimated coefficient contributes only a second-order term.

\section{Proof of Theorem~\ref{thm:oracle-optimality}}
\label{app:oracle-optimality}

Suppress $m,z$ from the notation and fix $b$. Let
\[
E_i=b^\top S_i,\qquad R_i=Y_i-b^\top S_i.
\]
The estimator can be written as
\[
\widehat\theta(z;b)=\widehat\mu_{R,L}(z)+\widehat\mu_{E,T}(z).
\]
All local means below are well defined on the event that the empirical local denominators
$\sum_{i\in L}K_h(Z_i-z)$ and $\sum_{i\in T}K_h(Z_i-z)$ are positive. Under
Assumption~\ref{ass:regularity}, this event has probability tending to one, and it is automatic for
kernels that are positive everywhere, such as the Gaussian. Define the population-smoothed local
means
\[
\mu_{R,h}(z)=
\frac{\mathbb E\{K_h(Z_i-z)R_i\}}{\mathbb E\{K_h(Z_i-z)\}},
\qquad
\mu_{E,h}(z)=
\frac{\mathbb E\{K_h(Z_i-z)E_i\}}{\mathbb E\{K_h(Z_i-z)\}}.
\]
Since $R_i+E_i=Y_i$ and $\mathbb E(Y_i\mid Z_i)=\theta(Z_i)$, the tower property gives the exact
identity $\mu_{R,h}(z)+\mu_{E,h}(z)-\theta(z)=B_h(z)$. Hence, on that event,
\[
\widehat\theta(z;b)-\theta(z)-B_h(z)
=
\{\widehat\mu_{R,L}(z)-\mu_{R,h}(z)\}
+
\{\widehat\mu_{E,T}(z)-\mu_{E,h}(z)\},
\]
with no remainder. Since
$\mathbb E[K_h(Z_i-z)\{R_i-\mu_{R,h}(z)\}]=0$ by construction, the standard ratio
linearization gives
\[
\widehat\mu_{R,L}(z)-\mu_{R,h}(z)
=
\frac{1}{n\,\mathbb E\{K_h(Z_i-z)\}}
\sum_{i\in L}K_h(Z_i-z)\{R_i-\mu_{R,h}(z)\}
+o_p\{(nh^q)^{-1/2}\},
\]
and similarly for $\widehat\mu_{E,T}(z)$ with $(L,n)$ replaced by $(T,M)$. Because $L$ is a
uniform random subset of the i.i.d.\ pool drawn independently of the item values, the array of
value--membership pairs is exchangeable, so we may take $L=\{1,\ldots,n\}$ without loss of
generality. Any linear combination of the two centered sums is then a sum over two independent
blocks of independent terms---items $1,\ldots,n$ and items $n+1,\ldots,M$---and the
Cram\'er--Wold device with the Lyapunov central limit theorem for triangular arrays, using the
$(2+\delta)$-moment bound in Assumption~\ref{ass:regularity}, yields joint asymptotic normality.

The marginal variance limits follow from the same kernel integration. For the cross covariance,
expanding over $L\times T$, independence of items kills the $i\ne j$ contributions and leaves only
the $n$ overlap terms with $i=j\in L$.
Each contributes
$\mathbb E[K_h(Z_i-z)^2\{R_i-\mu_{R,h}(z)\}\{E_i-\mu_{E,h}(z)\}]
=h^{-q}\{R(K)f_Z(z)\Cov(R_i,E_i\mid Z_i=z)+o(1)\}$ by kernel integration under continuity of
the conditional second moments at $z$, since $\mu_{R,h}(z)\to\mu_R(z)$ and
$\mu_{E,h}(z)\to\mu_E(z)$. Normalizing the two sums by
$n\,\mathbb E\{K_h(Z_i-z)\}$ and $M\,\mathbb E\{K_h(Z_i-z)\}$, and using
$\mathbb E\{K_h(Z_i-z)\}\to f_Z(z)$ only at this final step, yields
\[
\Cov\{\widehat\mu_{R,L}(z),\widehat\mu_{E,T}(z)\}
=
\frac{R(K)}{M h^q f_Z(z)}
\Cov(R_i,E_i\mid Z_i=z)
+
o((nh^q)^{-1}).
\]
Hence
\[
\AVar\{\widehat\theta(z;b)\}
=
\frac{R(K)}{h^q f_Z(z)}
\left[
\frac{\Var(R_i\mid Z_i=z)}{n}
+
\frac{\Var(E_i\mid Z_i=z)}{M}
+
\frac{2\Cov(R_i,E_i\mid Z_i=z)}{M}
\right].
\]
Multiplying by $nh^q$ and using $\pi_0=\lim n/M$ gives the variance factor
\[
\Var(Y_i-b^\top S_i\mid Z_i=z)
+
\pi_0\Var(b^\top S_i\mid Z_i=z)
+
2\pi_0\Cov(Y_i-b^\top S_i,b^\top S_i\mid Z_i=z).
\]
Expanding the quadratic yields
\[
\sigma_Y^2
+
(1-\pi_0)\{b^\top\Sigma_{SS}b-2b^\top\Sigma_{SY}\}
=
\mathcal V(z;b).
\]
When $\pi_0<1$, applying Lemma~\ref{lem:range}(i)--(iii) conditionally on $Z_i=z$ to $X=S_i$ and $Y=Y_i$ shows that the convex quadratic $f(b)=b^\top\Sigma_{SS}b-2b^\top\Sigma_{SY}$ is minimized on the solution set of $\Sigma_{SS}b=\Sigma_{SY}$, with minimum-norm minimizer $\Sigma_{SS}^\dagger\Sigma_{SY}$ and minimum value $-\Sigma_{SY}^\top\Sigma_{SS}^\dagger\Sigma_{SY}$.
Substituting into $\mathcal V(z;b)$ gives
\[
\mathcal V(z;\beta_m^\star)
=
\sigma_Y^2-(1-\pi_0)\Sigma_{SY}^\top\Sigma_{SS}^\dagger\Sigma_{SY}
=
\sigma_Y^2\{1-(1-\pi_0)R^2\}.
\]
When $\pi_0=1$, the scaled variance of $\bar S_L(z)-\bar S_T(z)$ carries the
finite-population factor $1-n/M\to0$, so
$\sqrt{nh^q}\{\bar S_L(z)-\bar S_T(z)\}=o_p(1)$ and the augmentation does not contribute to the
first-order variance.

\section{Proof of Theorem~\ref{thm:gain}}
\label{app:gain-proof}

We first consider the continuous case.
By Theorem~\ref{thm:oracle-optimality} with $b=\beta_m^\star(z)$, the oracle estimator $\widehat\theta_m(z;\beta_m^\star(z))$ satisfies
\[
\sqrt{nh^q}
\{\widehat\theta_m(z;\beta_m^\star(z))-\theta_m(z)-B_h(z)\}
\rightsquigarrow
\mathcal N\!\left(0,\frac{R(K)}{f_Z(z)}\mathcal V_m(z;\beta_m^\star(z))\right).
\]
By Lemma~\ref{lem:range} applied conditionally on $Z_i=z$ with $X=S_{im}$ and $Y=Y_{im}$, $\beta_m^\star(z)=\Sigma_{SS,m}(z)^\dagger\Sigma_{SY,m}(z)$ minimizes the local quadratic $f$, and by definition of $R_m^2(z)$,
\[
\Sigma_{SY,m}(z)^\top\Sigma_{SS,m}(z)^\dagger\Sigma_{SY,m}(z)=\sigma^2_{Y,m}(z)R_m^2(z).
\]
Substituting into the variance factor of Theorem~\ref{thm:oracle-optimality} gives
\[
\mathcal V_m(z;\beta_m^\star(z))=\sigma^2_{Y,m}(z)\{1-(1-\pi_0)R_m^2(z)\}.
\]

The feasible estimator replaces $\beta_m^\star(z)$ by $\widehat\beta_m(z)$.
By Theorem~\ref{thm:adaptivity} and Assumption~\ref{ass:nuisance},
\[
\widehat\theta_m(z;\widehat\beta_m(z))-\widehat\theta_m(z;\beta_m^\star(z))
=
-\{\widehat\beta_m(z)-\beta_m^\star(z)\}^\top\{\bar S_{m,L}(z)-\bar S_{m,T}(z)\}=o_p((nh^q)^{-1/2}),
\]
so the feasible and oracle estimators share the same first-order limit distribution.

The naive smoother corresponds to $b=0$ in Theorem~\ref{thm:oracle-optimality}, with $\mathcal V_m(z;0)=\sigma^2_{Y,m}(z)$, so its first-order variance factor is $\{R(K)/f_Z(z)\}\sigma^2_{Y,m}(z)$.
Dividing the naive variance by the LACE variance gives the gain $\operatorname{Gain}_m(z)=1/\{1-(1-\pi_0)R_m^2(z)\}$.

For a fixed discrete group $g$ with $p_g=\mathbb P(Z_i=g)>0$, condition on the group counts $M_g=|\{i\in T:Z_i=g\}|$ and $n_g=|\{i\in L:Z_i=g\}|$.
For fixed $M_g$ and $n_g$, Theorem~\ref{thm:finite} with $b_g=\beta_m^\star(g)$ gives
\[
\Var\,\widehat\theta_m(g;\beta_m^\star(g))
=
\frac{\sigma^2_{Y,m}(g)}{M_g}
+
\left(1-\frac{n_g}{M_g}\right)\frac{\sigma^2_{Y,m}(g)\{1-R_m^2(g)\}}{n_g},
\]
where the residual variance $\sigma^2_{R,g}(\beta_m^\star(g))=\sigma^2_{Y,m}(g)\{1-R_m^2(g)\}$ follows from Lemma~\ref{lem:range}(iv).
Since $M_g/M\to p_g$, $n_g/n\to p_g$, and $n_g/M_g\to\pi_0$ in probability,
\[
\Var\,\widehat\theta_m(g;\beta_m^\star(g))
=
\frac{\sigma^2_{Y,m}(g)}{np_g}\{1-(1-\pi_0)R_m^2(g)\}+o(n^{-1}).
\]
The corresponding grouped CLT follows from the standard CLT for within-group sample means.
The replacement of $\beta_m^\star(g)$ by $\widehat\beta_m(g)$ follows from the same Cauchy--Schwarz argument as in Theorem~\ref{thm:adaptivity}, with rate $\|\bar S_{m,L}(g)-\bar S_{m,T}(g)\|=O_p((np_g)^{-1/2})$ from the standard CLT for sample means.
The within-group estimator has $B_h(g)=0$ since it does not smooth.

\section{Proof of Theorem~\ref{thm:extensions}}
\label{app:extensions-proof}

For the paired gap, apply Theorem~\ref{thm:gain} with $(Y_{im},S_{im})$ replaced by
\[
Y_i^\Delta=D_{i,ab}=Y_{ia}-Y_{ib},
\qquad
S_i^\Delta=H_{i,ab}.
\]
Then $\mathbb E(D_{i,ab}\mid Z_i=z)=\Delta_{ab}(z)$, and the transferred Assumptions~\ref{ass:regularity}--\ref{ass:nuisance} give the stated limit distribution.
The smoothing bias is
\[
B_{\Delta,h}(z)
=
\frac{\mathbb E\{K_h(Z_i-z)\Delta_{ab}(Z_i)\}}{\mathbb E\{K_h(Z_i-z)\}}
-\Delta_{ab}(z).
\]
Lemma~\ref{lem:range}, applied conditionally on $Z_i=z$ with $X=H_{i,ab}$ and $Y=D_{i,ab}$, identifies the optimal coefficient as
\[
\beta^\star_{ab}(z)=\Sigma_{HH,ab}(z)^\dagger\Sigma_{HD,ab}(z),
\]
and yields the local coefficient of determination $R^2_{\Delta,ab}(z)$ appearing in the variance factor.

For the deployment score, abbreviate $\rho_i=r_Q(Z_i)$,
$\Psi_Q=\Psi_m(Q)$, and $\mu_Q=\mathbb E_Q(S_{im})$. The covariate-shift
definition of $Q$, together with $\mathbb E_P\rho_i=1$, gives
\[
\Psi_Q=\mathbb E_P(\rho_iY_{im}),
\qquad
\mu_Q=\mathbb E_P(\rho_iS_{im}).
\]
Let
\[
X_i=\rho_i(S_{im}-\mu_Q)=\phi^Q_{S,i},
\qquad
U_i=\rho_i(Y_{im}-\Psi_Q)=\phi^Q_{Y,i},
\]
and write $\mathbb P_A f=|A|^{-1}\sum_{i\in A}f_i$.
Then $\mathbb E_PX_i=0$, $\mathbb E_PU_i=0$,
$\Gamma^Q_{SS,m}=\mathbb E_P(X_iX_i^\top)$, and
$\Gamma^Q_{SY,m}=\mathbb E_P(X_iU_i)$.

We first establish consistency of the feasible coefficient. Conditional on any realized index set
$L$ independent of the data, the observations indexed by $L$ are $n$ i.i.d. draws from $P$.
Consequently, the weak law applies to both $\mathbb P_T$ and $\mathbb P_L$. In particular,
\[
\mathbb P_T\rho_i\xrightarrow{p}1,
\qquad
\mathbb P_L\rho_i\xrightarrow{p}1,
\]
so both self-normalizing denominators are positive with probability tending to one; define the
corresponding ratios arbitrarily on the complementary event. Moreover,
\[
d_M:=\bar S^Q_{m,T}-\mu_Q
=\frac{\mathbb P_TX_i}{\mathbb P_T\rho_i}=o_p(1),
\qquad
e_n:=\bar Y^Q_{m,L}-\Psi_Q
=\frac{\mathbb P_LU_i}{\mathbb P_L\rho_i}=o_p(1).
\]
The estimated influence quantities therefore satisfy
\[
\widehat\phi^Q_{S,i}=X_i-\rho_i d_M,
\qquad
\widehat\phi^Q_{Y,i}=U_i-\rho_i e_n.
\]
Expanding their empirical second moments gives
\[
\begin{aligned}
\widehat\Gamma^Q_{SS,m}
={}&\mathbb P_T(X_iX_i^\top)
-\{\mathbb P_T(\rho_iX_i)\}d_M^\top
-d_M\{\mathbb P_T(\rho_iX_i)\}^\top
+\{\mathbb P_T(\rho_i^2)\}d_Md_M^\top,
\\
\widehat\Gamma^Q_{SY,m}
={}&\mathbb P_L(X_iU_i)
-\{\mathbb P_L(\rho_iX_i)\}e_n
-d_M\mathbb P_L(\rho_iU_i)
+\{\mathbb P_L(\rho_i^2)\}d_Me_n.
\end{aligned}
\]
Boundedness of $\rho_i$ and the assumed second moments make all displayed empirical moments
integrable; in particular,
$\mathbb E_P\|X_iU_i\|_2\le
\{\mathbb E_P\|X_i\|_2^2\}^{1/2}\{\mathbb E_PU_i^2\}^{1/2}<\infty$.
The weak law and Slutsky's theorem thus imply
\[
\widehat\Gamma^Q_{SS,m}\xrightarrow{p}\Gamma^Q_{SS,m},
\qquad
\widehat\Gamma^Q_{SY,m}\xrightarrow{p}\Gamma^Q_{SY,m}.
\]

It remains to verify that a singular $\Gamma^Q_{SS,m}$ causes no null-space instability. Let
$\mathcal R=\operatorname{Range}(\Gamma^Q_{SS,m})$. For every
$v\in\operatorname{Null}(\Gamma^Q_{SS,m})$,
\[
0=v^\top\Gamma^Q_{SS,m}v=\mathbb E_P\{(v^\top X_i)^2\}.
\]
Taking a finite orthonormal basis of the null space and intersecting the corresponding probability-one
events shows that $X_i\in\mathcal R$ almost surely. Hence, on the event
$\mathbb P_T\rho_i>0$, the exact identity
$d_M=\mathbb P_TX_i/\mathbb P_T\rho_i$ gives $d_M\in\mathcal R$, and therefore
$\widehat\phi^Q_{S,i}\in\mathcal R$. It follows that
\[
\operatorname{Range}(\widehat\Gamma^Q_{SS,m})\subseteq\mathcal R,
\qquad
\widehat\Gamma^Q_{SY,m}\in\mathcal R.
\]
Let $d=\operatorname{rank}(\Gamma^Q_{SS,m})$. If $d=0$, then $X_i=0$ almost surely,
$\widehat\phi^Q_{S,i}=0$, and
$\widehat\beta_m(Q)=0=\beta^\star_{Q,m}$ with probability tending to one.
If $d>0$, let the columns of $V\in\mathbb R^{K\times d}$ form an orthonormal basis of
$\mathcal R$, and let the columns of $W$ form an orthonormal basis of $\mathcal R^\perp$.
Set
\[
G=V^\top\Gamma^Q_{SS,m}V,
\quad
\widehat G=V^\top\widehat\Gamma^Q_{SS,m}V,
\quad
g=V^\top\Gamma^Q_{SY,m},
\quad
\widehat g=V^\top\widehat\Gamma^Q_{SY,m}.
\]
The matrix $G$ is positive definite, $\widehat G\xrightarrow{p}G$, and
$\widehat g\xrightarrow{p}g$. The range inclusions above give the exact block decomposition
\[
\widehat\Gamma^Q_{SS,m}+\lambda_{Q,n}I_K
=V(\widehat G+\lambda_{Q,n}I_d)V^\top
+\lambda_{Q,n}WW^\top,
\qquad
\widehat\Gamma^Q_{SY,m}=V\widehat g,
\]
and hence
\[
\{\widehat\Gamma^Q_{SS,m}+\lambda_{Q,n}I_K\}^{-1}
=V(\widehat G+\lambda_{Q,n}I_d)^{-1}V^\top
+\lambda_{Q,n}^{-1}WW^\top.
\]
The potentially divergent null-space term vanishes when this inverse acts on
$\widehat\Gamma^Q_{SY,m}=V\widehat g$. Since $\lambda_{Q,n}\to0$, the continuous mapping theorem
therefore yields
\[
\widehat\beta_m(Q)
=V(\widehat G+\lambda_{Q,n}I_d)^{-1}\widehat g
\xrightarrow{p}
VG^{-1}g
=(\Gamma^Q_{SS,m})^\dagger\Gamma^Q_{SY,m}
=\beta^\star_{Q,m}.
\]

For $A=L$ for the outcome mean and $A\in\{L,T\}$ for the signal mean, write $N_A=|A|$.
The self-normalized weighted means satisfy the H\'ajek expansions
\[
\bar Y^Q_{m,A}-\Psi_Q
=
\frac{1}{N_A}\sum_{i\in A}\rho_i\{Y_{im}-\Psi_Q\}+o_p(N_A^{-1/2})
=
\mathbb P_A\phi^Q_{Y,i}+o_p(N_A^{-1/2}),
\]
and similarly $\bar S^Q_{m,A}-\mu_Q=\mathbb P_A\phi^Q_{S,i}+o_p(N_A^{-1/2})$.
These are first-order delta-method expansions for the ratio $\sum_A r_QV_i/\sum_A r_Q$, using $\mathbb E_P r_Q(Z_i)=1$ and boundedness of $r_Q$.

For a fixed coefficient $\beta$, the deployment LACE estimator is $\widehat\Psi_m(Q;\beta)=\bar Y^Q_{m,L}-\beta^\top\{\bar S^Q_{m,L}-\bar S^Q_{m,T}\}$.
Substituting the two H\'ajek expansions gives
\[
\widehat\Psi_m(Q;\beta)-\Psi_m(Q)
=
\mathbb P_{n,L}\phi^Q_{R,i}(\beta)
+
\mathbb P_{M,T}\phi^Q_{E,i}(\beta)
+
o_p(n^{-1/2}),
\]
where $\phi^Q_{R,i}(\beta)=\phi^Q_{Y,i}-\beta^\top\phi^Q_{S,i}$ and $\phi^Q_{E,i}(\beta)=\beta^\top\phi^Q_{S,i}$.

Since $L\subset T$, the covariance between the two empirical averages has the $1/M$ overlap form
\[
\Cov\{\mathbb P_{n,L}\phi^Q_R(\beta),\mathbb P_{M,T}\phi^Q_E(\beta)\}
=
\frac{1}{M}\Cov_P\{\phi^Q_R(\beta),\phi^Q_E(\beta)\}.
\]
Multiplying by $n$ and using $n/M\to\pi_0$, the first-order variance factor for fixed $\beta$ is
\[
\Var_P\{\phi^Q_R(\beta)\}+\pi_0\Var_P\{\phi^Q_E(\beta)\}+2\pi_0\Cov_P\{\phi^Q_R(\beta),\phi^Q_E(\beta)\}.
\]
The multivariate central limit theorem applied jointly to the labeled observations and the remaining
full-pool observations gives asymptotic normality with this variance.
Taking $\beta=\beta^\star_{Q,m}$ and applying Lemma~\ref{lem:range} with $X=\phi^Q_{S,i}$ and $Y=\phi^Q_{Y,i}$ gives $\Cov_P\{\phi^Q_R(\beta^\star_{Q,m}),\phi^Q_E(\beta^\star_{Q,m})\}=0$, $\Var_P\{\phi^Q_E(\beta^\star_{Q,m})\}=(\Gamma^Q_{SY,m})^\top(\Gamma^Q_{SS,m})^\dagger\Gamma^Q_{SY,m}$, and $\Var_P\{\phi^Q_R(\beta^\star_{Q,m})\}=\Var_P(\phi^Q_{Y,i})-(\Gamma^Q_{SY,m})^\top(\Gamma^Q_{SS,m})^\dagger\Gamma^Q_{SY,m}$.
Substituting yields the first-order variance $\Var_P(\phi^Q_{Y,i})\{1-(1-\pi_0)R^2_{Q,m}\}$, matching the stated limit.

Finally,
\[
\widehat\Psi_m(Q;\widehat\beta_m(Q))-\widehat\Psi_m(Q;\beta^\star_{Q,m})
=
-\{\widehat\beta_m(Q)-\beta^\star_{Q,m}\}^\top\{\bar S^Q_{m,L}-\bar S^Q_{m,T}\}.
\]
Since $\|\bar S^Q_{m,L}-\bar S^Q_{m,T}\|=O_p(n^{-1/2})$ and $\|\widehat\beta_m(Q)-\beta^\star_{Q,m}\|=o_p(1)$, this difference is $o_p(n^{-1/2})$. Hence the feasible deployment estimator has the same first-order limit distribution.

\section{A cascaded auxiliary-rating construction}
\label{app:cascade}
Suppose cheap signals $S_i$ are observed for the full pool $T$, an expensive-rater signal $E_i$ is
observed on a random subset $A\subset T$, and gold outcomes $Y_i$ are observed on a random nested
subset $L\subset A$. Let $H_i=(E_i,S_i^\top)^\top$. Writing bars for the same profile-local means as
in Section~\ref{sec:method}, a natural two-stage centered estimator is
\[
\widehat\theta_{\mathrm{cascade}}(z)
=
\bar Y_L(z)
-\widehat\beta_1(z)^\top\{\bar H_L(z)-\bar H_A(z)\}
-\widehat\beta_0(z)^\top\{\bar S_A(z)-\bar S_T(z)\}.
\]
The first difference uses the expensive-rated subset to reduce the sampling noise of the human-rated
subset; the second uses the full pool of cheap signals to reduce the sampling noise of the
expensive-rated subset. Neither auxiliary rater is required to be calibrated because it enters
through a centered difference. For equal-weight group profiles, simple-random-sampling identities
make these differences exactly centered. For normalized local kernel means, the corresponding claim
is first-order, through the usual H\'ajek expansions under analogous regularity conditions and
nonvanishing nested sampling fractions, rather than an exact finite-sample identity.

One practical implementation fits both coefficients on $L$ using the same local ridge construction
as \method{}: regress $Y$ locally on $H=(E,S)$ for $\widehat\beta_1(z)$ and locally on $S$ for
$\widehat\beta_0(z)$. This documents an implementable algebraic extension; we do not assert that
these coefficients are jointly variance-optimal for the cascade, nor establish a separate
adaptivity or limit-distribution theorem here.

\section{Additional Experimental Tables}
\label{app:exp-tables}

Row labels in all tables correspond to the estimator definitions in Section~\ref{estimators_compared}.

\begin{table}[p]
\centering\scriptsize
\setlength{\tabcolsep}{3pt}
\caption{RE $[95\%\ \mathrm{MC\ CI}]$ on MATH-500 (Claude Haiku 3, continuous $Z$, $B=100$); mean MSE $\times 10^3$ is in parentheses. Intervals use paired split-level delta-method influence values. The target is the realized full-pool gold kernel profile; uncertainty is over matched labeled-subset draws conditional on the observed benchmark pool.}
\label{tab:real:math500:haiku-3}
\resizebox{\textwidth}{!}{%
\begin{tabular}{lccc}
\toprule
Estimator & $n_{\mathrm{lab}}=50$ & $n_{\mathrm{lab}}=100$ & $n_{\mathrm{lab}}=200$ \\
\midrule
Naive & $1.00\,[1.00,1.00]\ (9.81)$ & $1.00\,[1.00,1.00]\ (5.00)$ & $1.00\,[1.00,1.00]\ (1.77)$ \\
Plug-in (judge) & $2.59\,[0.96,6.98]\ (3.79)$ & $2.43\,[1.56,3.79]\ (2.06)$ & $2.53\,[1.72,3.72]\ (0.70)$ \\
Plug-in (multi) & $2.77\,[1.34,5.73]\ (3.54)$ & $2.67\,[1.86,3.83]\ (1.87)$ & $2.37\,[1.77,3.18]\ (0.75)$ \\
Aug. plug-in & $1.45\,[1.05,2.00]\ (6.76)$ & $1.44\,[1.20,1.73]\ (3.48)$ & $1.70\,[1.22,2.36]\ (1.04)$ \\
Global CV & $1.86\,[1.43,2.41]\ (5.27)$ & $1.81\,[1.63,2.01]\ (2.76)$ & $1.90\,[1.55,2.33]\ (0.93)$ \\
Per-signal CV & $1.24\,[1.16,1.33]\ (7.93)$ & $1.39\,[1.31,1.47]\ (3.60)$ & $1.41\,[1.33,1.50]\ (1.25)$ \\
Residual-only & $0.57\,[0.33,1.00]\ (17.33)$ & $0.37\,[0.27,0.51]\ (13.67)$ & $0.11\,[0.07,0.17]\ (16.81)$ \\
Scalar prediction CV & $1.38\,[1.08,1.76]\ (7.09)$ & $1.42\,[1.21,1.67]\ (3.52)$ & $1.66\,[1.26,2.19]\ (1.06)$ \\
Local PPI & $0.59\,[0.50,0.67]\ (16.74)$ & $0.62\,[0.55,0.70]\ (8.02)$ & $0.51\,[0.42,0.60]\ (3.48)$ \\
PPCI & $1.46\,[1.09,1.96]\ (6.72)$ & $1.40\,[1.17,1.68]\ (3.58)$ & $1.52\,[1.13,2.04]\ (1.16)$ \\
StratPPI & $0.24\,[0.14,0.42]\ (40.69)$ & $0.46\,[0.24,0.87]\ (10.84)$ & $0.43\,[0.30,0.61]\ (4.10)$ \\
\method{} (feasible) & $\mathbf{6.14}\,[4.24,8.88]\ (1.60)$ & $\mathbf{5.78}\,[5.05,6.62]\ (0.86)$ & $\mathbf{6.07}\,[4.81,7.65]\ (0.29)$ \\
\bottomrule
\end{tabular}%
}
\end{table}

\begin{table}[p]
\centering\scriptsize
\setlength{\tabcolsep}{3pt}
\caption{RE $[95\%\ \mathrm{MC\ CI}]$ on ScienceQA (Claude Haiku 3, continuous $Z$, $B=100$); mean MSE $\times 10^3$ is in parentheses. Intervals use paired split-level delta-method influence values. The target is the realized full-pool gold kernel profile; uncertainty is over matched labeled-subset draws conditional on the observed benchmark pool.}
\label{tab:real:scienceqa:haiku-3}
\resizebox{\textwidth}{!}{%
\begin{tabular}{lccc}
\toprule
Estimator & $n_{\mathrm{lab}}=50$ & $n_{\mathrm{lab}}=100$ & $n_{\mathrm{lab}}=200$ \\
\midrule
Naive & $1.00\,[1.00,1.00]\ (1.71)$ & $1.00\,[1.00,1.00]\ (0.74)$ & $1.00\,[1.00,1.00]\ (0.31)$ \\
Plug-in (judge) & $2.59\,[1.37,4.89]\ (0.66)$ & $2.10\,[1.52,2.91]\ (0.35)$ & $1.62\,[1.17,2.25]\ (0.19)$ \\
Plug-in (multi) & $1.96\,[0.95,4.05]\ (0.88)$ & $1.42\,[0.94,2.15]\ (0.52)$ & $1.24\,[0.90,1.71]\ (0.25)$ \\
Aug. plug-in & $0.74\,[0.34,1.61]\ (2.30)$ & $0.73\,[0.53,1.00]\ (1.01)$ & $1.07\,[0.78,1.46]\ (0.29)$ \\
Global CV & $1.43\,[1.23,1.66]\ (1.20)$ & $1.35\,[1.16,1.57]\ (0.55)$ & $1.53\,[1.22,1.92]\ (0.20)$ \\
Per-signal CV & $1.06\,[0.99,1.14]\ (1.62)$ & $1.09\,[1.02,1.17]\ (0.68)$ & $1.17\,[1.10,1.24]\ (0.27)$ \\
Residual-only & $1.34\,[1.00,1.80]\ (1.28)$ & $1.06\,[0.82,1.36]\ (0.70)$ & $0.64\,[0.40,1.03]\ (0.48)$ \\
Scalar prediction CV & $0.65\,[0.49,0.85]\ (2.65)$ & $1.03\,[0.81,1.30]\ (0.72)$ & $1.25\,[0.95,1.65]\ (0.25)$ \\
Local PPI & $0.63\,[0.54,0.73]\ (2.71)$ & $0.68\,[0.58,0.77]\ (1.09)$ & $0.65\,[0.55,0.75]\ (0.48)$ \\
PPCI & $0.72\,[0.36,1.42]\ (2.39)$ & $0.64\,[0.44,0.92]\ (1.15)$ & $0.80\,[0.52,1.24]\ (0.39)$ \\
StratPPI & $0.38\,[0.21,0.68]\ (4.53)$ & $0.65\,[0.49,0.87]\ (1.13)$ & $0.81\,[0.61,1.07]\ (0.38)$ \\
\method{} (feasible) & $\mathbf{5.96}\,[3.75,9.47]\ (0.29)$ & $\mathbf{5.70}\,[2.20,14.74]\ (0.13)$ & $\mathbf{5.13}\,[3.42,7.69]\ (0.06)$ \\
\bottomrule
\end{tabular}%
}
\end{table}

\begin{table}[p]
\centering\scriptsize
\setlength{\tabcolsep}{3pt}
\caption{RE $[95\%\ \mathrm{MC\ CI}]$ on MMLU (Claude Haiku 3, continuous $Z$, $B=100$); mean MSE $\times 10^3$ is in parentheses. Intervals use paired split-level delta-method influence values. The target is the realized full-pool gold kernel profile; uncertainty is over matched labeled-subset draws conditional on the observed benchmark pool.}
\label{tab:real:mmlu:haiku-3}
\resizebox{\textwidth}{!}{%
\begin{tabular}{lccc}
\toprule
Estimator & $n_{\mathrm{lab}}=50$ & $n_{\mathrm{lab}}=100$ & $n_{\mathrm{lab}}=200$ \\
\midrule
Naive & $1.00\,[1.00,1.00]\ (28.09)$ & $1.00\,[1.00,1.00]\ (9.59)$ & $1.00\,[1.00,1.00]\ (4.17)$ \\
Plug-in (judge) & $3.23\,[1.67,6.25]\ (8.69)$ & $2.65\,[1.46,4.80]\ (3.61)$ & $1.91\,[1.31,2.78]\ (2.18)$ \\
Plug-in (multi) & $3.09\,[0.66,14.51]\ (9.10)$ & $3.42\,[1.01,11.59]\ (2.80)$ & $2.51\,[1.36,4.62]\ (1.66)$ \\
Aug. plug-in & $0.69\,[0.34,1.39]\ (40.94)$ & $1.06\,[0.69,1.62]\ (9.08)$ & $1.00\,[0.48,2.09]\ (4.17)$ \\
Global CV & $1.12\,[1.06,1.18]\ (25.03)$ & $1.29\,[1.05,1.58]\ (7.41)$ & $1.18\,[1.05,1.33]\ (3.53)$ \\
Per-signal CV & $0.96\,[0.90,1.02]\ (29.18)$ & $1.00\,[0.94,1.06]\ (9.60)$ & $1.01\,[0.96,1.06]\ (4.12)$ \\
Residual-only & $1.04\,[0.89,1.22]\ (26.90)$ & $0.92\,[0.65,1.30]\ (10.38)$ & $0.73\,[0.56,0.96]\ (5.68)$ \\
Scalar prediction CV & $0.08\,[0.01,1.13]\ (373.99)$ & $1.10\,[0.89,1.36]\ (8.72)$ & $1.02\,[0.87,1.20]\ (4.10)$ \\
Local PPI & $0.61\,[0.52,0.70]\ (45.98)$ & $0.58\,[0.52,0.65]\ (16.47)$ & $0.53\,[0.47,0.59]\ (7.88)$ \\
PPCI & $0.94\,[0.46,1.94]\ (30.01)$ & $1.25\,[0.84,1.87]\ (7.66)$ & $1.06\,[0.72,1.57]\ (3.92)$ \\
StratPPI & $0.97\,[0.39,2.41]\ (29.01)$ & $1.15\,[0.71,1.87]\ (8.32)$ & $1.09\,[0.84,1.42]\ (3.82)$ \\
\method{} (feasible) & $\mathbf{4.37}\,[3.90,4.90]\ (6.43)$ & $\mathbf{4.65}\,[3.90,5.54]\ (2.06)$ & $\mathbf{4.28}\,[3.89,4.71]\ (0.97)$ \\
\bottomrule
\end{tabular}%
}
\end{table}

\begin{table}[p]
\centering\scriptsize
\setlength{\tabcolsep}{3pt}
\caption{RE $[95\%\ \mathrm{MC\ CI}]$ on WinoGrande (Claude Haiku 3, continuous $Z$, $B=100$); mean MSE $\times 10^3$ is in parentheses. Intervals use paired split-level delta-method influence values. The target is the realized full-pool gold kernel profile; uncertainty is over matched labeled-subset draws conditional on the observed benchmark pool.}
\label{tab:real:winogrande:haiku-3}
\resizebox{\textwidth}{!}{%
\begin{tabular}{lccc}
\toprule
Estimator & $n_{\mathrm{lab}}=50$ & $n_{\mathrm{lab}}=100$ & $n_{\mathrm{lab}}=200$ \\
\midrule
Naive & $1.00\,[1.00,1.00]\ (14.40)$ & $1.00\,[1.00,1.00]\ (5.51)$ & $1.00\,[1.00,1.00]\ (2.21)$ \\
Plug-in (judge) & $4.52\,[0.15,139.57]\ (3.18)$ & $3.49\,[1.42,8.57]\ (1.58)$ & $3.47\,[1.26,9.53]\ (0.64)$ \\
Plug-in (multi) & $5.00\,[0.25,98.06]\ (2.88)$ & $3.56\,[0.71,17.92]\ (1.55)$ & $3.08\,[1.26,7.55]\ (0.72)$ \\
Aug. plug-in & $1.55\,[1.25,1.92]\ (9.31)$ & $1.56\,[0.45,5.43]\ (3.54)$ & $2.16\,[1.60,2.92]\ (1.03)$ \\
Global CV & $1.72\,[1.61,1.84]\ (8.35)$ & $1.92\,[1.72,2.15]\ (2.87)$ & $2.21\,[1.60,3.05]\ (1.00)$ \\
Per-signal CV & $1.18\,[1.14,1.23]\ (12.23)$ & $1.19\,[1.14,1.24]\ (4.63)$ & $1.23\,[1.19,1.27]\ (1.81)$ \\
Residual-only & $1.44\,[1.31,1.58]\ (10.00)$ & $1.34\,[1.19,1.51]\ (4.13)$ & $0.80\,[0.50,1.28]\ (2.77)$ \\
Scalar prediction CV & $1.42\,[1.09,1.85]\ (10.11)$ & $1.50\,[0.99,2.27]\ (3.67)$ & $2.02\,[1.35,3.02]\ (1.10)$ \\
Local PPI & $0.71\,[0.64,0.78]\ (20.18)$ & $0.69\,[0.62,0.77]\ (7.95)$ & $0.66\,[0.59,0.72]\ (3.36)$ \\
PPCI & $1.69\,[1.37,2.09]\ (8.52)$ & $1.61\,[0.30,8.59]\ (3.43)$ & $2.07\,[1.56,2.75]\ (1.07)$ \\
StratPPI & $0.97\,[0.70,1.33]\ (14.79)$ & $1.24\,[0.94,1.63]\ (4.43)$ & $1.42\,[1.03,1.95]\ (1.56)$ \\
\method{} (feasible) & $\mathbf{6.17}\,[5.71,6.67]\ (2.33)$ & $\mathbf{6.74}\,[6.08,7.47]\ (0.82)$ & $\mathbf{7.00}\,[4.17,11.76]\ (0.32)$ \\
\bottomrule
\end{tabular}%
}
\end{table}

\begin{table}[p]
\centering\scriptsize
\setlength{\tabcolsep}{3pt}
\caption{RE $[95\%\ \mathrm{MC\ CI}]$ on HellaSwag (Claude Haiku 3, continuous $Z$, $B=100$); mean MSE $\times 10^3$ is in parentheses. Intervals use paired split-level delta-method influence values. The target is the realized full-pool gold kernel profile; uncertainty is over matched labeled-subset draws conditional on the observed benchmark pool.}
\label{tab:real:hellaswag:haiku-3}
\resizebox{\textwidth}{!}{%
\begin{tabular}{lccc}
\toprule
Estimator & $n_{\mathrm{lab}}=50$ & $n_{\mathrm{lab}}=100$ & $n_{\mathrm{lab}}=200$ \\
\midrule
Naive & $1.00\,[1.00,1.00]\ (6.70)$ & $1.00\,[1.00,1.00]\ (3.23)$ & $1.00\,[1.00,1.00]\ (1.05)$ \\
Plug-in (judge) & $2.41\,[0.20,28.63]\ (2.79)$ & $2.29\,[0.57,9.15]\ (1.41)$ & $2.40\,[0.71,8.16]\ (0.44)$ \\
Plug-in (multi) & $2.65\,[0.61,11.50]\ (2.53)$ & $2.39\,[0.19,30.40]\ (1.35)$ & $2.25\,[1.01,5.02]\ (0.47)$ \\
Aug. plug-in & $1.12\,[0.53,2.36]\ (5.98)$ & $1.31\,[0.99,1.74]\ (2.46)$ & $1.65\,[1.13,2.41]\ (0.63)$ \\
Global CV & $1.34\,[1.27,1.41]\ (4.99)$ & $1.57\,[1.47,1.67]\ (2.05)$ & $1.73\,[1.49,2.00]\ (0.60)$ \\
Per-signal CV & $1.00\,[0.95,1.06]\ (6.69)$ & $1.15\,[1.10,1.20]\ (2.82)$ & $1.16\,[1.11,1.21]\ (0.90)$ \\
Residual-only & $1.15\,[0.97,1.36]\ (5.81)$ & $0.74\,[0.25,2.19]\ (4.37)$ & $0.17\,[0.14,0.21]\ (6.09)$ \\
Scalar prediction CV & $0.97\,[0.70,1.34]\ (6.92)$ & $1.21\,[0.91,1.61]\ (2.67)$ & $1.51\,[1.15,1.98]\ (0.69)$ \\
Local PPI & $0.74\,[0.65,0.83]\ (9.04)$ & $0.74\,[0.65,0.84]\ (4.34)$ & $0.69\,[0.58,0.80]\ (1.52)$ \\
PPCI & $1.12\,[0.64,1.95]\ (6.00)$ & $1.24\,[0.91,1.69]\ (2.60)$ & $1.45\,[0.80,2.64]\ (0.72)$ \\
StratPPI & $0.82\,[0.55,1.23]\ (8.13)$ & $0.98\,[0.76,1.26]\ (3.28)$ & $1.02\,[0.84,1.24]\ (1.03)$ \\
\method{} (feasible) & $\mathbf{5.31}\,[4.26,6.62]\ (1.26)$ & $\mathbf{5.77}\,[5.17,6.44]\ (0.56)$ & $\mathbf{5.98}\,[4.69,7.63]\ (0.17)$ \\
\bottomrule
\end{tabular}%
}
\end{table}

\begin{table}[p]
\centering\scriptsize
\setlength{\tabcolsep}{3pt}
\caption{RE $[95\%\ \mathrm{MC\ CI}]$ on TruthfulQA (Claude Haiku 3, continuous $Z$, $B=100$); mean MSE $\times 10^3$ is in parentheses. Intervals use paired split-level delta-method influence values. The target is the realized full-pool gold kernel profile; uncertainty is over matched labeled-subset draws conditional on the observed benchmark pool.}
\label{tab:real:truthfulqa:haiku-3}
\resizebox{\textwidth}{!}{%
\begin{tabular}{lccc}
\toprule
Estimator & $n_{\mathrm{lab}}=50$ & $n_{\mathrm{lab}}=100$ & $n_{\mathrm{lab}}=200$ \\
\midrule
Naive & $1.00\,[1.00,1.00]\ (20.94)$ & $1.00\,[1.00,1.00]\ (9.29)$ & $1.00\,[1.00,1.00]\ (3.82)$ \\
Plug-in (judge) & $2.85\,[0.00,7723.82]\ (7.35)$ & $3.20\,[0.17,61.80]\ (2.90)$ & $3.65\,[0.30,44.55]\ (1.05)$ \\
Plug-in (multi) & $3.10\,[0.00,1999.15]\ (6.75)$ & $3.55\,[0.04,342.26]\ (2.62)$ & $3.40\,[0.01,897.68]\ (1.12)$ \\
Aug. plug-in & $1.65\,[0.29,9.27]\ (12.69)$ & $2.10\,[0.54,8.19]\ (4.42)$ & $2.85\,[0.55,14.76]\ (1.34)$ \\
Global CV & $1.48\,[1.30,1.68]\ (14.15)$ & $1.72\,[1.39,2.13]\ (5.40)$ & $2.35\,[1.34,4.12]\ (1.63)$ \\
Per-signal CV & $1.12\,[1.05,1.19]\ (18.69)$ & $1.25\,[1.19,1.32]\ (7.43)$ & $1.48\,[1.42,1.55]\ (2.58)$ \\
Residual-only & $0.85\,[0.45,1.62]\ (24.63)$ & $0.58\,[0.38,0.88]\ (16.02)$ & $0.32\,[0.19,0.53]\ (11.95)$ \\
Scalar prediction CV & $1.55\,[1.11,2.16]\ (13.51)$ & $2.05\,[1.35,3.12]\ (4.53)$ & $2.75\,[1.42,5.31]\ (1.39)$ \\
Local PPI & $0.49\,[0.38,0.60]\ (42.75)$ & $0.58\,[0.53,0.63]\ (15.92)$ & $0.55\,[0.50,0.60]\ (6.95)$ \\
PPCI & $2.20\,[0.26,18.56]\ (9.52)$ & $2.65\,[0.56,12.46]\ (3.51)$ & $3.15\,[0.62,16.12]\ (1.21)$ \\
StratPPI & $0.95\,[0.01,130.32]\ (22.04)$ & $1.10\,[0.08,15.15]\ (8.45)$ & $0.88\,[0.15,5.24]\ (4.35)$ \\
\method{} (feasible) & $\mathbf{4.35}\,[4.04,4.68]\ (4.81)$ & $\mathbf{5.10}\,[4.57,5.69]\ (1.82)$ & $\mathbf{5.75}\,[4.85,6.82]\ (0.67)$ \\
\bottomrule
\end{tabular}%
}
\end{table}

\begin{table}[p]
\centering\scriptsize
\setlength{\tabcolsep}{3pt}
\caption{RE $[95\%\ \mathrm{MC\ CI}]$ on GSM8K (Claude Haiku 3, continuous $Z$, $B=100$); mean MSE $\times 10^3$ is in parentheses. Intervals use paired split-level delta-method influence values. The target is the realized full-pool gold kernel profile; uncertainty is over matched labeled-subset draws conditional on the observed benchmark pool.}
\label{tab:real:gsm8k:haiku-3}
\resizebox{\textwidth}{!}{%
\begin{tabular}{lccc}
\toprule
Estimator & $n_{\mathrm{lab}}=50$ & $n_{\mathrm{lab}}=100$ & $n_{\mathrm{lab}}=200$ \\
\midrule
Naive & $1.00\,[1.00,1.00]\ (7.65)$ & $1.00\,[1.00,1.00]\ (2.50)$ & $1.00\,[1.00,1.00]\ (0.85)$ \\
Plug-in (judge) & $3.71\,[1.63,8.43]\ (2.06)$ & $2.30\,[1.46,3.61]\ (1.09)$ & $1.28\,[0.79,2.07]\ (0.67)$ \\
Plug-in (multi) & $5.84\,[1.66,20.53]\ (1.31)$ & $4.10\,[2.39,7.03]\ (0.61)$ & $1.68\,[0.88,3.19]\ (0.51)$ \\
Aug. plug-in & $1.25\,[0.18,8.85]\ (6.12)$ & $1.34\,[0.79,2.28]\ (1.87)$ & $1.62\,[1.11,2.37]\ (0.53)$ \\
Global CV & $1.45\,[1.32,1.59]\ (5.27)$ & $1.63\,[1.48,1.79]\ (1.53)$ & $1.73\,[1.38,2.17]\ (0.49)$ \\
Per-signal CV & $0.96\,[0.85,1.08]\ (7.97)$ & $1.18\,[1.07,1.30]\ (2.12)$ & $1.26\,[1.13,1.40]\ (0.68)$ \\
Residual-only & $1.68\,[0.60,4.69]\ (4.54)$ & $1.06\,[0.05,23.42]\ (2.36)$ & $0.36\,[0.23,0.57]\ (2.37)$ \\
Scalar prediction CV & $1.09\,[0.28,4.21]\ (7.04)$ & $1.29\,[0.87,1.91]\ (1.94)$ & $1.42\,[0.84,2.41]\ (0.60)$ \\
Local PPI & $0.65\,[0.55,0.75]\ (11.76)$ & $0.65\,[0.54,0.75]\ (3.87)$ & $0.70\,[0.60,0.80]\ (1.21)$ \\
PPCI & $1.33\,[0.20,8.87]\ (5.74)$ & $1.30\,[0.63,2.69]\ (1.92)$ & $1.50\,[0.86,2.61]\ (0.57)$ \\
StratPPI & $0.50\,[0.30,0.83]\ (15.27)$ & $0.71\,[0.51,0.99]\ (3.54)$ & $0.79\,[0.41,1.53]\ (1.07)$ \\
\method{} (feasible) & $\mathbf{6.88}\,[1.99,23.82]\ (1.11)$ & $\mathbf{6.90}\,[3.73,12.77]\ (0.36)$ & $\mathbf{5.43}\,[0.91,32.40]\ (0.16)$ \\
\bottomrule
\end{tabular}%
}
\end{table}

\begin{table}[p]
\centering\scriptsize
\setlength{\tabcolsep}{3pt}
\caption{RE $[95\%\ \mathrm{MC\ CI}]$ on ARC (Claude Haiku 3, continuous $Z$, $B=100$); mean MSE $\times 10^3$ is in parentheses. Intervals use paired split-level delta-method influence values. The target is the realized full-pool gold kernel profile; uncertainty is over matched labeled-subset draws conditional on the observed benchmark pool.}
\label{tab:real:arc:haiku-3}
\resizebox{\textwidth}{!}{%
\begin{tabular}{lccc}
\toprule
Estimator & $n_{\mathrm{lab}}=50$ & $n_{\mathrm{lab}}=100$ & $n_{\mathrm{lab}}=200$ \\
\midrule
Naive & $1.00\,[1.00,1.00]\ (7.04)$ & $1.00\,[1.00,1.00]\ (3.17)$ & $1.00\,[1.00,1.00]\ (1.26)$ \\
Plug-in (judge) & $3.29\,[1.10,9.82]\ (2.14)$ & $2.82\,[1.15,6.93]\ (1.12)$ & $1.68\,[0.89,3.16]\ (0.75)$ \\
Plug-in (multi) & $3.31\,[0.30,36.46]\ (2.12)$ & $2.79\,[0.40,19.31]\ (1.14)$ & $1.58\,[0.74,3.39]\ (0.80)$ \\
Aug. plug-in & $0.98\,[0.05,17.94]\ (7.16)$ & $1.15\,[0.60,2.19]\ (2.75)$ & $1.27\,[0.87,1.85]\ (0.99)$ \\
Global CV & $1.39\,[1.24,1.56]\ (5.07)$ & $1.42\,[1.29,1.56]\ (2.23)$ & $1.67\,[1.42,1.96]\ (0.76)$ \\
Per-signal CV & $0.89\,[0.81,0.97]\ (7.88)$ & $1.16\,[1.07,1.25]\ (2.74)$ & $1.21\,[1.12,1.31]\ (1.04)$ \\
Residual-only & $1.38\,[0.96,1.99]\ (5.08)$ & $1.40\,[0.96,2.04]\ (2.26)$ & $0.85\,[0.37,1.95]\ (1.48)$ \\
Scalar prediction CV & $0.91\,[0.63,1.31]\ (7.73)$ & $1.20\,[0.72,1.99]\ (2.63)$ & $1.22\,[0.75,1.99]\ (1.03)$ \\
Local PPI & $0.58\,[0.48,0.69]\ (12.05)$ & $0.69\,[0.58,0.80]\ (4.58)$ & $0.66\,[0.56,0.76]\ (1.91)$ \\
PPCI & $1.06\,[0.12,9.65]\ (6.61)$ & $1.10\,[0.51,2.38]\ (2.87)$ & $1.15\,[0.82,1.61]\ (1.10)$ \\
StratPPI & $1.44\,[0.91,2.29]\ (4.87)$ & $1.14\,[0.73,1.78]\ (2.78)$ & $0.79\,[0.59,1.05]\ (1.59)$ \\
\method{} (feasible) & $\mathbf{5.71}\,[3.58,9.11]\ (1.23)$ & $\mathbf{5.97}\,[4.39,8.12]\ (0.53)$ & $\mathbf{5.34}\,[2.13,13.38]\ (0.24)$ \\
\bottomrule
\end{tabular}%
}
\end{table}

\begin{table}[p]
\centering\scriptsize
\setlength{\tabcolsep}{3pt}
\caption{RE $[95\%\ \mathrm{MC\ CI}]$ on MATH-500 (Ministral 3B, continuous $Z$, $B=100$); mean MSE $\times 10^3$ is in parentheses. Intervals use paired split-level delta-method influence values. The target is the realized full-pool gold kernel profile; uncertainty is over matched labeled-subset draws conditional on the observed benchmark pool.}
\label{tab:real:math500:ministral-3b}
\resizebox{\textwidth}{!}{%
\begin{tabular}{lccc}
\toprule
Estimator & $n_{\mathrm{lab}}=50$ & $n_{\mathrm{lab}}=100$ & $n_{\mathrm{lab}}=200$ \\
\midrule
Naive & $1.00\,[1.00,1.00]\ (7.46)$ & $1.00\,[1.00,1.00]\ (3.50)$ & $1.00\,[1.00,1.00]\ (1.63)$ \\
Plug-in (judge) & $1.91\,[1.15,3.18]\ (3.91)$ & $2.08\,[1.40,3.10]\ (1.68)$ & $2.31\,[1.54,3.46]\ (0.71)$ \\
Plug-in (multi) & $1.85\,[0.63,5.42]\ (4.04)$ & $2.18\,[1.47,3.23]\ (1.60)$ & $2.20\,[1.17,4.15]\ (0.74)$ \\
Aug. plug-in & $0.96\,[0.74,1.24]\ (7.79)$ & $1.16\,[0.82,1.64]\ (3.02)$ & $1.46\,[0.87,2.45]\ (1.12)$ \\
Global CV & $1.28\,[1.13,1.45]\ (5.82)$ & $1.44\,[1.18,1.76]\ (2.43)$ & $1.53\,[1.28,1.84]\ (1.07)$ \\
Per-signal CV & $0.81\,[0.74,0.89]\ (9.17)$ & $1.05\,[0.97,1.14]\ (3.34)$ & $1.21\,[1.13,1.29]\ (1.35)$ \\
Residual-only & $0.81\,[0.62,1.05]\ (9.18)$ & $0.58\,[0.43,0.78]\ (6.05)$ & $0.40\,[0.21,0.75]\ (4.09)$ \\
Scalar prediction CV & $0.64\,[0.47,0.86]\ (11.59)$ & $0.92\,[0.73,1.16]\ (3.80)$ & $1.35\,[1.04,1.75]\ (1.21)$ \\
Local PPI & $0.69\,[0.61,0.76]\ (10.84)$ & $0.70\,[0.64,0.76]\ (5.00)$ & $0.75\,[0.66,0.84]\ (2.17)$ \\
PPCI & $0.97\,[0.74,1.27]\ (7.67)$ & $1.15\,[0.82,1.62]\ (3.05)$ & $1.39\,[1.02,1.89]\ (1.17)$ \\
StratPPI & $0.28\,[0.19,0.40]\ (27.11)$ & $0.37\,[0.27,0.51]\ (9.32)$ & $0.58\,[0.46,0.73]\ (2.82)$ \\
\method{} (feasible) & $\mathbf{4.19}\,[3.52,4.99]\ (1.78)$ & $\mathbf{4.73}\,[4.05,5.52]\ (0.74)$ & $\mathbf{5.20}\,[4.61,5.86]\ (0.31)$ \\
\bottomrule
\end{tabular}%
}
\end{table}

\begin{table}[p]
\centering\scriptsize
\setlength{\tabcolsep}{3pt}
\caption{RE $[95\%\ \mathrm{MC\ CI}]$ on ScienceQA (Ministral 3B, continuous $Z$, $B=100$); mean MSE $\times 10^3$ is in parentheses. Intervals use paired split-level delta-method influence values. The target is the realized full-pool gold kernel profile; uncertainty is over matched labeled-subset draws conditional on the observed benchmark pool.}
\label{tab:real:scienceqa:ministral-3b}
\resizebox{\textwidth}{!}{%
\begin{tabular}{lccc}
\toprule
Estimator & $n_{\mathrm{lab}}=50$ & $n_{\mathrm{lab}}=100$ & $n_{\mathrm{lab}}=200$ \\
\midrule
Naive & $1.00\,[1.00,1.00]\ (2.30)$ & $1.00\,[1.00,1.00]\ (1.13)$ & $1.00\,[1.00,1.00]\ (0.38)$ \\
Plug-in (judge) & $3.34\,[0.09,123.13]\ (0.69)$ & $2.57\,[0.46,14.27]\ (0.44)$ & $2.25\,[0.97,5.22]\ (0.17)$ \\
Plug-in (multi) & $2.82\,[0.78,10.25]\ (0.81)$ & $2.43\,[0.92,6.41]\ (0.47)$ & $1.77\,[0.91,3.45]\ (0.21)$ \\
Aug. plug-in & $0.86\,[0.56,1.32]\ (2.68)$ & $1.04\,[0.74,1.46]\ (1.09)$ & $1.09\,[0.73,1.63]\ (0.35)$ \\
Global CV & $1.38\,[1.27,1.50]\ (1.67)$ & $1.45\,[1.28,1.65]\ (0.78)$ & $1.53\,[1.29,1.82]\ (0.25)$ \\
Per-signal CV & $0.71\,[0.64,0.79]\ (3.22)$ & $1.18\,[1.09,1.28]\ (0.96)$ & $1.15\,[1.07,1.23]\ (0.33)$ \\
Residual-only & $1.46\,[1.01,2.11]\ (1.58)$ & $1.63\,[1.27,2.09]\ (0.69)$ & $1.16\,[0.93,1.45]\ (0.32)$ \\
Scalar prediction CV & $0.36\,[0.22,0.59]\ (6.45)$ & $1.18\,[0.85,1.64]\ (0.96)$ & $1.20\,[0.79,1.83]\ (0.31)$ \\
Local PPI & $0.72\,[0.57,0.86]\ (3.21)$ & $0.69\,[0.55,0.83]\ (1.63)$ & $0.52\,[0.43,0.62]\ (0.72)$ \\
PPCI & $0.83\,[0.50,1.37]\ (2.76)$ & $0.96\,[0.71,1.30]\ (1.18)$ & $0.85\,[0.56,1.29]\ (0.44)$ \\
StratPPI & $0.59\,[0.38,0.92]\ (3.91)$ & $0.68\,[0.16,2.80]\ (1.66)$ & $0.94\,[0.68,1.30]\ (0.40)$ \\
\method{} (feasible) & $\mathbf{5.72}\,[4.41,7.42]\ (0.40)$ & $\mathbf{6.56}\,[5.29,8.14]\ (0.17)$ & $\mathbf{5.67}\,[4.24,7.58]\ (0.07)$ \\
\bottomrule
\end{tabular}%
}
\end{table}

\begin{table}[p]
\centering\scriptsize
\setlength{\tabcolsep}{3pt}
\caption{RE $[95\%\ \mathrm{MC\ CI}]$ on MMLU (Ministral 3B, continuous $Z$, $B=100$); mean MSE $\times 10^3$ is in parentheses. Intervals use paired split-level delta-method influence values. The target is the realized full-pool gold kernel profile; uncertainty is over matched labeled-subset draws conditional on the observed benchmark pool.}
\label{tab:real:mmlu:ministral-3b}
\resizebox{\textwidth}{!}{%
\begin{tabular}{lccc}
\toprule
Estimator & $n_{\mathrm{lab}}=50$ & $n_{\mathrm{lab}}=100$ & $n_{\mathrm{lab}}=200$ \\
\midrule
Naive & $1.00\,[1.00,1.00]\ (18.17)$ & $1.00\,[1.00,1.00]\ (6.65)$ & $1.00\,[1.00,1.00]\ (3.20)$ \\
Plug-in (judge) & $3.36\,[1.83,6.17]\ (5.40)$ & $2.33\,[1.67,3.25]\ (2.86)$ & $1.81\,[1.50,2.18]\ (1.77)$ \\
Plug-in (multi) & $4.19\,[1.62,10.81]\ (4.34)$ & $2.77\,[1.89,4.06]\ (2.40)$ & $2.38\,[1.90,2.98]\ (1.34)$ \\
Aug. plug-in & $0.69\,[0.31,1.54]\ (26.32)$ & $0.97\,[0.61,1.54]\ (6.83)$ & $1.34\,[1.01,1.78]\ (2.38)$ \\
Global CV & $1.20\,[1.10,1.30]\ (15.15)$ & $1.28\,[0.68,2.40]\ (5.18)$ & $1.41\,[1.21,1.65]\ (2.27)$ \\
Per-signal CV & $0.84\,[0.78,0.91]\ (21.52)$ & $0.99\,[0.91,1.07]\ (6.72)$ & $1.09\,[1.03,1.16]\ (2.94)$ \\
Residual-only & $1.19\,[0.98,1.45]\ (15.32)$ & $1.13\,[0.84,1.52]\ (5.89)$ & $0.74\,[0.44,1.24]\ (4.31)$ \\
Scalar prediction CV & $1.06\,[0.92,1.22]\ (17.21)$ & $1.11\,[0.89,1.38]\ (5.99)$ & $1.22\,[0.97,1.53]\ (2.63)$ \\
Local PPI & $0.63\,[0.57,0.70]\ (28.71)$ & $0.65\,[0.57,0.73]\ (10.27)$ & $0.69\,[0.62,0.75]\ (4.66)$ \\
PPCI & $1.02\,[0.51,2.04]\ (17.79)$ & $1.16\,[0.70,1.92]\ (5.72)$ & $1.40\,[1.08,1.82]\ (2.29)$ \\
StratPPI & $0.88\,[0.67,1.15]\ (20.61)$ & $0.87\,[0.50,1.50]\ (7.63)$ & $0.78\,[0.62,0.98]\ (4.12)$ \\
\method{} (feasible) & $\mathbf{4.70}\,[4.12,5.37]\ (3.87)$ & $\mathbf{4.88}\,[3.58,6.64]\ (1.36)$ & $\mathbf{4.86}\,[4.05,5.83]\ (0.66)$ \\
\bottomrule
\end{tabular}%
}
\end{table}

\begin{table}[p]
\centering\scriptsize
\setlength{\tabcolsep}{3pt}
\caption{RE $[95\%\ \mathrm{MC\ CI}]$ on WinoGrande (Ministral 3B, continuous $Z$, $B=100$); mean MSE $\times 10^3$ is in parentheses. Intervals use paired split-level delta-method influence values. The target is the realized full-pool gold kernel profile; uncertainty is over matched labeled-subset draws conditional on the observed benchmark pool.}
\label{tab:real:winogrande:ministral-3b}
\resizebox{\textwidth}{!}{%
\begin{tabular}{lccc}
\toprule
Estimator & $n_{\mathrm{lab}}=50$ & $n_{\mathrm{lab}}=100$ & $n_{\mathrm{lab}}=200$ \\
\midrule
Naive & $1.00\,[1.00,1.00]\ (14.72)$ & $1.00\,[1.00,1.00]\ (6.86)$ & $1.00\,[1.00,1.00]\ (2.40)$ \\
Plug-in (judge) & $3.72\,[2.60,5.33]\ (3.95)$ & $3.83\,[2.18,6.73]\ (1.79)$ & $2.91\,[2.30,3.68]\ (0.83)$ \\
Plug-in (multi) & $3.81\,[2.76,5.26]\ (3.86)$ & $4.15\,[2.54,6.77]\ (1.65)$ & $2.80\,[2.14,3.66]\ (0.86)$ \\
Aug. plug-in & $1.33\,[1.00,1.76]\ (11.07)$ & $1.90\,[1.38,2.62]\ (3.61)$ & $2.22\,[1.82,2.71]\ (1.08)$ \\
Global CV & $1.84\,[1.47,2.31]\ (7.98)$ & $2.20\,[1.81,2.68]\ (3.12)$ & $2.20\,[1.81,2.67]\ (1.09)$ \\
Per-signal CV & $1.11\,[1.05,1.17]\ (13.27)$ & $1.21\,[1.14,1.28]\ (5.67)$ & $1.27\,[1.21,1.33]\ (1.89)$ \\
Residual-only & $1.29\,[1.07,1.55]\ (11.40)$ & $1.45\,[1.03,2.03]\ (4.72)$ & $1.47\,[1.18,1.83]\ (1.64)$ \\
Scalar prediction CV & $1.29\,[1.10,1.51]\ (11.42)$ & $1.70\,[1.15,2.52]\ (4.03)$ & $2.12\,[1.21,3.72]\ (1.14)$ \\
Local PPI & $0.65\,[0.52,0.77]\ (22.71)$ & $0.78\,[0.71,0.85]\ (8.81)$ & $0.75\,[0.66,0.83]\ (3.21)$ \\
PPCI & $1.43\,[1.08,1.90]\ (10.26)$ & $1.95\,[1.43,2.66]\ (3.52)$ & $2.14\,[1.65,2.77]\ (1.12)$ \\
StratPPI & $0.83\,[0.58,1.19]\ (17.70)$ & $1.10\,[0.79,1.54]\ (6.24)$ & $0.98\,[0.72,1.33]\ (2.45)$ \\
\method{} (feasible) & $\mathbf{5.37}\,[4.48,6.43]\ (2.74)$ & $\mathbf{6.40}\,[4.73,8.66]\ (1.07)$ & $\mathbf{7.38}\,[6.29,8.66]\ (0.33)$ \\
\bottomrule
\end{tabular}%
}
\end{table}

\begin{table}[p]
\centering\scriptsize
\setlength{\tabcolsep}{3pt}
\caption{RE $[95\%\ \mathrm{MC\ CI}]$ on HellaSwag (Ministral 3B, continuous $Z$, $B=100$); mean MSE $\times 10^3$ is in parentheses. Intervals use paired split-level delta-method influence values. The target is the realized full-pool gold kernel profile; uncertainty is over matched labeled-subset draws conditional on the observed benchmark pool.}
\label{tab:real:hellaswag:ministral-3b}
\resizebox{\textwidth}{!}{%
\begin{tabular}{lccc}
\toprule
Estimator & $n_{\mathrm{lab}}=50$ & $n_{\mathrm{lab}}=100$ & $n_{\mathrm{lab}}=200$ \\
\midrule
Naive & $1.00\,[1.00,1.00]\ (9.73)$ & $1.00\,[1.00,1.00]\ (5.25)$ & $1.00\,[1.00,1.00]\ (1.62)$ \\
Plug-in (judge) & $5.02\,[2.53,9.97]\ (1.94)$ & $5.99\,[2.54,14.12]\ (0.88)$ & $4.74\,[2.91,7.72]\ (0.34)$ \\
Plug-in (multi) & $3.98\,[1.64,9.65]\ (2.45)$ & $4.24\,[2.05,8.75]\ (1.24)$ & $3.05\,[1.87,4.99]\ (0.53)$ \\
Aug. plug-in & $1.93\,[0.84,4.45]\ (5.05)$ & $3.33\,[1.40,7.93]\ (1.58)$ & $3.29\,[0.79,13.64]\ (0.49)$ \\
Global CV & $3.02\,[2.19,4.16]\ (3.23)$ & $4.15\,[3.28,5.25]\ (1.26)$ & $3.76\,[2.17,6.51]\ (0.43)$ \\
Per-signal CV & $1.36\,[1.25,1.48]\ (7.15)$ & $1.53\,[1.44,1.63]\ (3.43)$ & $1.49\,[1.39,1.60]\ (1.09)$ \\
Residual-only & $1.68\,[0.78,3.61]\ (5.78)$ & $1.34\,[0.78,2.31]\ (3.93)$ & $0.56\,[0.40,0.78]\ (2.91)$ \\
Scalar prediction CV & $1.97\,[1.34,2.90]\ (4.94)$ & $3.11\,[1.25,7.73]\ (1.69)$ & $3.24\,[1.49,7.04]\ (0.50)$ \\
Local PPI & $0.83\,[0.73,0.93]\ (11.69)$ & $0.94\,[0.83,1.05]\ (5.60)$ & $0.92\,[0.81,1.03]\ (1.76)$ \\
PPCI & $1.94\,[0.85,4.42]\ (5.02)$ & $3.15\,[1.43,6.93]\ (1.67)$ & $2.92\,[0.44,19.20]\ (0.55)$ \\
StratPPI & $1.26\,[0.94,1.70]\ (7.72)$ & $1.64\,[1.15,2.33]\ (3.21)$ & $1.39\,[1.09,1.78]\ (1.17)$ \\
\method{} (feasible) & $\mathbf{7.95}\,[4.39,14.40]\ (1.22)$ & $\mathbf{10.09}\,[5.42,18.80]\ (0.52)$ & $\mathbf{11.03}\,[7.75,15.70]\ (0.15)$ \\
\bottomrule
\end{tabular}%
}
\end{table}

\begin{table}[p]
\centering\scriptsize
\setlength{\tabcolsep}{3pt}
\caption{RE $[95\%\ \mathrm{MC\ CI}]$ on TruthfulQA (Ministral 3B, continuous $Z$, $B=100$); mean MSE $\times 10^3$ is in parentheses. Intervals use paired split-level delta-method influence values. The target is the realized full-pool gold kernel profile; uncertainty is over matched labeled-subset draws conditional on the observed benchmark pool.}
\label{tab:real:truthfulqa:ministral-3b}
\resizebox{\textwidth}{!}{%
\begin{tabular}{lccc}
\toprule
Estimator & $n_{\mathrm{lab}}=50$ & $n_{\mathrm{lab}}=100$ & $n_{\mathrm{lab}}=200$ \\
\midrule
Naive & $1.00\,[1.00,1.00]\ (7.21)$ & $1.00\,[1.00,1.00]\ (4.24)$ & $1.00\,[1.00,1.00]\ (1.85)$ \\
Plug-in (judge) & $3.17\,[1.90,5.30]\ (2.28)$ & $4.49\,[3.27,6.17]\ (0.95)$ & $3.64\,[2.86,4.63]\ (0.51)$ \\
Plug-in (multi) & $2.68\,[1.49,4.83]\ (2.70)$ & $4.48\,[2.53,7.92]\ (0.95)$ & $3.35\,[2.14,5.24]\ (0.55)$ \\
Aug. plug-in & $1.08\,[0.50,2.35]\ (6.67)$ & $1.03\,[0.57,1.86]\ (4.13)$ & $1.76\,[1.40,2.21]\ (1.05)$ \\
Global CV & $1.34\,[0.75,2.39]\ (5.37)$ & $1.51\,[1.11,2.05]\ (2.81)$ & $1.72\,[1.21,2.44]\ (1.07)$ \\
Per-signal CV & $0.90\,[0.85,0.96]\ (8.01)$ & $1.08\,[1.02,1.14]\ (3.92)$ & $1.09\,[1.04,1.14]\ (1.70)$ \\
Residual-only & $0.91\,[0.79,1.05]\ (7.94)$ & $0.88\,[0.72,1.07]\ (4.83)$ & $0.40\,[0.23,0.70]\ (4.68)$ \\
Scalar prediction CV & $1.15\,[0.82,1.62]\ (6.25)$ & $1.46\,[1.02,2.09]\ (2.91)$ & $1.73\,[1.05,2.84]\ (1.07)$ \\
Local PPI & $0.37\,[0.22,0.51]\ (19.73)$ & $0.64\,[0.53,0.75]\ (6.62)$ & $0.54\,[0.46,0.63]\ (3.40)$ \\
PPCI & $1.32\,[0.77,2.26]\ (5.46)$ & $1.47\,[1.07,2.02]\ (2.89)$ & $1.83\,[1.41,2.37]\ (1.01)$ \\
StratPPI & $1.25\,[0.28,5.68]\ (5.78)$ & $1.24\,[0.36,4.26]\ (3.42)$ & $1.27\,[0.91,1.77]\ (1.46)$ \\
\method{} (feasible) & $\mathbf{4.89}\,[4.30,5.56]\ (1.48)$ & $\mathbf{6.23}\,[3.75,10.35]\ (0.68)$ & $\mathbf{5.48}\,[4.23,7.10]\ (0.34)$ \\
\bottomrule
\end{tabular}%
}
\end{table}

\begin{table}[p]
\centering\scriptsize
\setlength{\tabcolsep}{3pt}
\caption{RE $[95\%\ \mathrm{MC\ CI}]$ on GSM8K (Ministral 3B, continuous $Z$, $B=100$); mean MSE $\times 10^3$ is in parentheses. Intervals use paired split-level delta-method influence values. The target is the realized full-pool gold kernel profile; uncertainty is over matched labeled-subset draws conditional on the observed benchmark pool.}
\label{tab:real:gsm8k:ministral-3b}
\resizebox{\textwidth}{!}{%
\begin{tabular}{lccc}
\toprule
Estimator & $n_{\mathrm{lab}}=50$ & $n_{\mathrm{lab}}=100$ & $n_{\mathrm{lab}}=200$ \\
\midrule
Naive & $1.00\,[1.00,1.00]\ (4.68)$ & $1.00\,[1.00,1.00]\ (2.03)$ & $1.00\,[1.00,1.00]\ (0.70)$ \\
Plug-in (judge) & $2.77\,[0.30,25.82]\ (1.69)$ & $3.03\,[0.87,10.55]\ (0.67)$ & $1.95\,[0.59,6.45]\ (0.36)$ \\
Plug-in (multi) & $2.82\,[1.62,4.92]\ (1.66)$ & $3.09\,[0.99,9.64]\ (0.66)$ & $2.36\,[1.05,5.29]\ (0.30)$ \\
Aug. plug-in & $0.77\,[0.36,1.64]\ (6.06)$ & $0.86\,[0.55,1.36]\ (2.37)$ & $1.20\,[0.76,1.91]\ (0.58)$ \\
Global CV & $1.19\,[1.10,1.29]\ (3.94)$ & $1.31\,[1.18,1.45]\ (1.54)$ & $1.47\,[1.19,1.81]\ (0.47)$ \\
Per-signal CV & $0.83\,[0.77,0.90]\ (5.61)$ & $1.12\,[1.00,1.25]\ (1.81)$ & $1.16\,[1.05,1.28]\ (0.61)$ \\
Residual-only & $1.23\,[1.01,1.49]\ (3.81)$ & $1.40\,[0.92,2.13]\ (1.44)$ & $0.93\,[0.55,1.57]\ (0.75)$ \\
Scalar prediction CV & $0.66\,[0.40,1.09]\ (7.10)$ & $0.88\,[0.64,1.21]\ (2.29)$ & $1.16\,[0.75,1.80]\ (0.60)$ \\
Local PPI & $0.66\,[0.59,0.73]\ (7.11)$ & $0.72\,[0.62,0.82]\ (2.83)$ & $0.77\,[0.67,0.87]\ (0.91)$ \\
PPCI & $0.82\,[0.42,1.62]\ (5.68)$ & $0.86\,[0.56,1.32]\ (2.36)$ & $1.06\,[0.74,1.52]\ (0.66)$ \\
StratPPI & $0.52\,[0.42,0.64]\ (8.91)$ & $0.60\,[0.25,1.44]\ (3.38)$ & $0.82\,[0.41,1.64]\ (0.85)$ \\
\method{} (feasible) & $\mathbf{4.94}\,[3.78,6.45]\ (0.95)$ & $\mathbf{5.40}\,[4.23,6.90]\ (0.38)$ & $\mathbf{4.46}\,[2.88,6.91]\ (0.16)$ \\
\bottomrule
\end{tabular}%
}
\end{table}

\begin{table}[p]
\centering\scriptsize
\setlength{\tabcolsep}{3pt}
\caption{RE $[95\%\ \mathrm{MC\ CI}]$ on ARC (Ministral 3B, continuous $Z$, $B=100$); mean MSE $\times 10^3$ is in parentheses. Intervals use paired split-level delta-method influence values. The target is the realized full-pool gold kernel profile; uncertainty is over matched labeled-subset draws conditional on the observed benchmark pool.}
\label{tab:real:arc:ministral-3b}
\resizebox{\textwidth}{!}{%
\begin{tabular}{lccc}
\toprule
Estimator & $n_{\mathrm{lab}}=50$ & $n_{\mathrm{lab}}=100$ & $n_{\mathrm{lab}}=200$ \\
\midrule
Naive & $1.00\,[1.00,1.00]\ (4.44)$ & $1.00\,[1.00,1.00]\ (1.82)$ & $1.00\,[1.00,1.00]\ (0.65)$ \\
Plug-in (judge) & $5.14\,[2.38,11.08]\ (0.86)$ & $4.48\,[2.81,7.15]\ (0.41)$ & $3.89\,[2.66,5.70]\ (0.17)$ \\
Plug-in (multi) & $5.79\,[0.38,87.23]\ (0.77)$ & $5.09\,[2.28,11.39]\ (0.36)$ & $4.73\,[1.46,15.31]\ (0.14)$ \\
Aug. plug-in & $1.28\,[0.41,4.03]\ (3.48)$ & $1.49\,[0.43,5.19]\ (1.23)$ & $1.67\,[0.91,3.06]\ (0.39)$ \\
Global CV & $1.32\,[1.21,1.44]\ (3.37)$ & $1.47\,[1.26,1.72]\ (1.24)$ & $1.76\,[1.21,2.57]\ (0.37)$ \\
Per-signal CV & $0.90\,[0.83,0.98]\ (4.95)$ & $1.06\,[0.98,1.15]\ (1.72)$ & $1.18\,[1.08,1.29]\ (0.55)$ \\
Residual-only & $1.47\,[1.29,1.67]\ (3.01)$ & $1.13\,[0.88,1.45]\ (1.61)$ & $0.93\,[0.57,1.53]\ (0.70)$ \\
Scalar prediction CV & $0.71\,[0.40,1.27]\ (6.24)$ & $1.11\,[0.81,1.52]\ (1.65)$ & $1.37\,[1.12,1.68]\ (0.48)$ \\
Local PPI & $0.46\,[0.37,0.54]\ (9.73)$ & $0.47\,[0.42,0.52]\ (3.85)$ & $0.37\,[0.33,0.41]\ (1.76)$ \\
PPCI & $1.39\,[0.50,3.88]\ (3.20)$ & $1.47\,[0.66,3.26]\ (1.24)$ & $1.40\,[0.69,2.85]\ (0.47)$ \\
StratPPI & $1.07\,[0.78,1.48]\ (4.16)$ & $0.99\,[0.51,1.91]\ (1.83)$ & $0.58\,[0.37,0.91]\ (1.13)$ \\
\method{} (feasible) & $\mathbf{5.93}\,[5.40,6.52]\ (0.75)$ & $\mathbf{5.34}\,[4.67,6.11]\ (0.34)$ & $\mathbf{5.73}\,[3.65,8.99]\ (0.11)$ \\
\bottomrule
\end{tabular}%
}
\end{table}

\begin{table}[p]
\centering\scriptsize
\setlength{\tabcolsep}{3pt}
\caption{RE $[95\%\ \mathrm{MC\ CI}]$ on MATH-500 (Qwen3 32B, continuous $Z$, $B=100$); mean MSE $\times 10^3$ is in parentheses. Intervals use paired split-level delta-method influence values. The target is the realized full-pool gold kernel profile; uncertainty is over matched labeled-subset draws conditional on the observed benchmark pool.}
\label{tab:real:math500:qwen3-32b}
\resizebox{\textwidth}{!}{%
\begin{tabular}{lccc}
\toprule
Estimator & $n_{\mathrm{lab}}=50$ & $n_{\mathrm{lab}}=100$ & $n_{\mathrm{lab}}=200$ \\
\midrule
Naive & $1.00\,[1.00,1.00]\ (6.14)$ & $1.00\,[1.00,1.00]\ (2.66)$ & $1.00\,[1.00,1.00]\ (0.98)$ \\
Plug-in (judge) & $2.42\,[1.32,4.44]\ (2.54)$ & $2.57\,[1.81,3.66]\ (1.03)$ & $2.01\,[1.61,2.50]\ (0.49)$ \\
Plug-in (multi) & $2.49\,[1.30,4.76]\ (2.46)$ & $2.71\,[1.84,3.98]\ (0.98)$ & $1.94\,[1.51,2.49]\ (0.50)$ \\
Aug. plug-in & $1.32\,[0.94,1.86]\ (4.64)$ & $1.37\,[1.04,1.81]\ (1.94)$ & $1.47\,[1.17,1.84]\ (0.67)$ \\
Global CV & $1.49\,[1.26,1.76]\ (4.13)$ & $1.50\,[1.32,1.71]\ (1.78)$ & $1.54\,[1.34,1.77]\ (0.63)$ \\
Per-signal CV & $0.89\,[0.82,0.97]\ (6.90)$ & $1.10\,[1.04,1.16]\ (2.42)$ & $1.11\,[1.05,1.17]\ (0.88)$ \\
Residual-only & $0.90\,[0.64,1.27]\ (6.82)$ & $0.77\,[0.58,1.03]\ (3.46)$ & $0.21\,[0.16,0.28]\ (4.69)$ \\
Scalar prediction CV & $1.02\,[0.84,1.24]\ (6.00)$ & $1.27\,[0.80,2.02]\ (2.09)$ & $1.33\,[1.02,1.74]\ (0.74)$ \\
Local PPI & $0.81\,[0.70,0.92]\ (7.61)$ & $0.79\,[0.69,0.89]\ (3.36)$ & $0.69\,[0.59,0.80]\ (1.41)$ \\
PPCI & $1.34\,[0.97,1.86]\ (4.58)$ & $1.34\,[1.01,1.78]\ (1.98)$ & $1.36\,[1.02,1.81]\ (0.72)$ \\
StratPPI & $0.24\,[0.16,0.37]\ (25.19)$ & $0.62\,[0.47,0.81]\ (4.28)$ & $0.63\,[0.51,0.77]\ (1.55)$ \\
\method{} (feasible) & $\mathbf{5.47}\,[4.46,6.70]\ (1.12)$ & $\mathbf{5.38}\,[4.58,6.32]\ (0.49)$ & $\mathbf{4.32}\,[3.49,5.35]\ (0.23)$ \\
\bottomrule
\end{tabular}%
}
\end{table}

\begin{table}[p]
\centering\scriptsize
\setlength{\tabcolsep}{3pt}
\caption{RE $[95\%\ \mathrm{MC\ CI}]$ on ScienceQA (Qwen3 32B, continuous $Z$, $B=100$); mean MSE $\times 10^3$ is in parentheses. Intervals use paired split-level delta-method influence values. The target is the realized full-pool gold kernel profile; uncertainty is over matched labeled-subset draws conditional on the observed benchmark pool.}
\label{tab:real:scienceqa:qwen3-32b}
\resizebox{\textwidth}{!}{%
\begin{tabular}{lccc}
\toprule
Estimator & $n_{\mathrm{lab}}=50$ & $n_{\mathrm{lab}}=100$ & $n_{\mathrm{lab}}=200$ \\
\midrule
Naive & $1.00\,[1.00,1.00]\ (0.92)$ & $1.00\,[1.00,1.00]\ (0.42)$ & $1.00\,[1.00,1.00]\ (0.16)$ \\
Plug-in (judge) & $3.07\,[1.25,7.56]\ (0.30)$ & $2.34\,[1.54,3.56]\ (0.18)$ & $1.30\,[0.91,1.85]\ (0.13)$ \\
Plug-in (multi) & $2.78\,[1.53,5.07]\ (0.33)$ & $1.43\,[0.79,2.59]\ (0.29)$ & $0.76\,[0.52,1.11]\ (0.22)$ \\
Aug. plug-in & $0.83\,[0.62,1.10]\ (1.12)$ & $0.77\,[0.56,1.05]\ (0.55)$ & $0.88\,[0.59,1.32]\ (0.18)$ \\
Global CV & $1.19\,[1.16,1.22]\ (0.78)$ & $1.34\,[1.21,1.48]\ (0.31)$ & $1.44\,[1.15,1.80]\ (0.11)$ \\
Per-signal CV & $0.24\,[0.17,0.33]\ (3.90)$ & $0.85\,[0.78,0.93]\ (0.49)$ & $1.11\,[1.05,1.18]\ (0.15)$ \\
Residual-only & $1.45\,[1.20,1.76]\ (0.64)$ & $1.49\,[1.21,1.83]\ (0.28)$ & $1.53\,[1.30,1.80]\ (0.11)$ \\
Scalar prediction CV & $0.52\,[0.43,0.62]\ (1.78)$ & $0.83\,[0.47,1.47]\ (0.50)$ & $1.10\,[0.83,1.46]\ (0.15)$ \\
Local PPI & $0.51\,[0.38,0.65]\ (1.78)$ & $0.54\,[0.46,0.62]\ (0.77)$ & $0.47\,[0.40,0.54]\ (0.34)$ \\
PPCI & $0.80\,[0.60,1.07]\ (1.16)$ & $0.66\,[0.49,0.88]\ (0.63)$ & $0.60\,[0.31,1.15]\ (0.27)$ \\
StratPPI & $0.28\,[0.12,0.63]\ (3.31)$ & $0.71\,[0.57,0.88]\ (0.59)$ & $0.95\,[0.71,1.27]\ (0.17)$ \\
\method{} (feasible) & $\mathbf{5.71}\,[4.83,6.75]\ (0.16)$ & $\mathbf{5.90}\,[4.99,6.97]\ (0.07)$ & $\mathbf{5.12}\,[4.30,6.10]\ (0.03)$ \\
\bottomrule
\end{tabular}%
}
\end{table}

\begin{table}[p]
\centering\scriptsize
\setlength{\tabcolsep}{3pt}
\caption{RE $[95\%\ \mathrm{MC\ CI}]$ on MMLU (Qwen3 32B, continuous $Z$, $B=100$); mean MSE $\times 10^3$ is in parentheses. Intervals use paired split-level delta-method influence values. The target is the realized full-pool gold kernel profile; uncertainty is over matched labeled-subset draws conditional on the observed benchmark pool.}
\label{tab:real:mmlu:qwen3-32b}
\resizebox{\textwidth}{!}{%
\begin{tabular}{lccc}
\toprule
Estimator & $n_{\mathrm{lab}}=50$ & $n_{\mathrm{lab}}=100$ & $n_{\mathrm{lab}}=200$ \\
\midrule
Naive & $1.00\,[1.00,1.00]\ (17.93)$ & $1.00\,[1.00,1.00]\ (5.00)$ & $1.00\,[1.00,1.00]\ (2.02)$ \\
Plug-in (judge) & $2.09\,[0.59,7.46]\ (8.57)$ & $1.20\,[0.76,1.90]\ (4.17)$ & $0.61\,[0.45,0.84]\ (3.31)$ \\
Plug-in (multi) & $2.48\,[0.96,6.39]\ (7.22)$ & $1.52\,[1.07,2.16]\ (3.30)$ & $0.73\,[0.56,0.95]\ (2.77)$ \\
Aug. plug-in & $0.89\,[0.01,137.87]\ (20.25)$ & $0.78\,[0.45,1.34]\ (6.38)$ & $1.11\,[0.66,1.86]\ (1.81)$ \\
Global CV & $1.17\,[1.11,1.24]\ (15.28)$ & $1.19\,[1.08,1.31]\ (4.22)$ & $1.18\,[1.10,1.27]\ (1.71)$ \\
Per-signal CV & $0.97\,[0.92,1.03]\ (18.43)$ & $1.02\,[0.97,1.07]\ (4.90)$ & $1.03\,[0.97,1.09]\ (1.96)$ \\
Residual-only & $1.33\,[0.94,1.89]\ (13.49)$ & $1.30\,[0.63,2.66]\ (3.83)$ & $1.04\,[0.72,1.49]\ (1.94)$ \\
Scalar prediction CV & $0.29\,[0.13,0.64]\ (61.47)$ & $0.79\,[0.54,1.17]\ (6.37)$ & $1.04\,[0.82,1.33]\ (1.94)$ \\
Local PPI & $0.45\,[0.24,0.66]\ (39.95)$ & $0.66\,[0.59,0.72]\ (7.61)$ & $0.59\,[0.54,0.64]\ (3.41)$ \\
PPCI & $1.23\,[0.31,4.86]\ (14.56)$ & $0.94\,[0.55,1.59]\ (5.31)$ & $1.09\,[0.72,1.65]\ (1.85)$ \\
StratPPI & $1.05\,[0.62,1.77]\ (17.10)$ & $0.86\,[0.60,1.24]\ (5.79)$ & $0.72\,[0.43,1.19]\ (2.80)$ \\
\method{} (feasible) & $\mathbf{5.00}\,[4.09,6.11]\ (3.59)$ & $\mathbf{4.27}\,[2.88,6.33]\ (1.17)$ & $\mathbf{2.73}\,[1.67,4.46]\ (0.74)$ \\
\bottomrule
\end{tabular}%
}
\end{table}

\begin{table}[p]
\centering\scriptsize
\setlength{\tabcolsep}{3pt}
\caption{RE $[95\%\ \mathrm{MC\ CI}]$ on WinoGrande (Qwen3 32B, continuous $Z$, $B=100$); mean MSE $\times 10^3$ is in parentheses. Intervals use paired split-level delta-method influence values. The target is the realized full-pool gold kernel profile; uncertainty is over matched labeled-subset draws conditional on the observed benchmark pool.}
\label{tab:real:winogrande:qwen3-32b}
\resizebox{\textwidth}{!}{%
\begin{tabular}{lccc}
\toprule
Estimator & $n_{\mathrm{lab}}=50$ & $n_{\mathrm{lab}}=100$ & $n_{\mathrm{lab}}=200$ \\
\midrule
Naive & $1.00\,[1.00,1.00]\ (10.35)$ & $1.00\,[1.00,1.00]\ (4.59)$ & $1.00\,[1.00,1.00]\ (1.80)$ \\
Plug-in (judge) & $5.03\,[1.94,13.05]\ (2.06)$ & $3.54\,[1.84,6.79]\ (1.30)$ & $3.69\,[1.95,6.99]\ (0.49)$ \\
Plug-in (multi) & $4.92\,[0.81,29.84]\ (2.10)$ & $3.68\,[2.34,5.78]\ (1.25)$ & $3.14\,[1.78,5.54]\ (0.57)$ \\
Aug. plug-in & $1.42\,[0.87,2.31]\ (7.27)$ & $1.69\,[1.23,2.31]\ (2.72)$ & $1.98\,[1.46,2.68]\ (0.91)$ \\
Global CV & $1.66\,[1.36,2.02]\ (6.24)$ & $2.04\,[1.65,2.53]\ (2.25)$ & $2.08\,[1.43,3.03]\ (0.86)$ \\
Per-signal CV & $1.02\,[0.97,1.08]\ (10.14)$ & $1.11\,[1.07,1.15]\ (4.13)$ & $1.16\,[1.11,1.21]\ (1.55)$ \\
Residual-only & $1.23\,[1.07,1.42]\ (8.39)$ & $1.31\,[0.85,2.02]\ (3.51)$ & $1.34\,[0.90,1.99]\ (1.34)$ \\
Scalar prediction CV & $1.41\,[0.82,2.41]\ (7.31)$ & $1.62\,[1.17,2.25]\ (2.84)$ & $1.88\,[1.40,2.53]\ (0.96)$ \\
Local PPI & $0.61\,[0.45,0.78]\ (16.85)$ & $0.74\,[0.67,0.81]\ (6.21)$ & $0.75\,[0.66,0.83]\ (2.41)$ \\
PPCI & $1.52\,[0.97,2.38]\ (6.79)$ & $1.69\,[1.16,2.46]\ (2.72)$ & $1.89\,[1.36,2.62]\ (0.95)$ \\
StratPPI & $1.02\,[0.73,1.42]\ (10.12)$ & $1.18\,[0.82,1.69]\ (3.91)$ & $1.45\,[0.96,2.19]\ (1.24)$ \\
\method{} (feasible) & $\mathbf{5.66}\,[4.72,6.79]\ (1.83)$ & $\mathbf{6.33}\,[3.02,13.28]\ (0.73)$ & $\mathbf{7.96}\,[5.54,11.44]\ (0.23)$ \\
\bottomrule
\end{tabular}%
}
\end{table}

\begin{table}[p]
\centering\scriptsize
\setlength{\tabcolsep}{3pt}
\caption{RE $[95\%\ \mathrm{MC\ CI}]$ on HellaSwag (Qwen3 32B, continuous $Z$, $B=100$); mean MSE $\times 10^3$ is in parentheses. Intervals use paired split-level delta-method influence values. The target is the realized full-pool gold kernel profile; uncertainty is over matched labeled-subset draws conditional on the observed benchmark pool.}
\label{tab:real:hellaswag:qwen3-32b}
\resizebox{\textwidth}{!}{%
\begin{tabular}{lccc}
\toprule
Estimator & $n_{\mathrm{lab}}=50$ & $n_{\mathrm{lab}}=100$ & $n_{\mathrm{lab}}=200$ \\
\midrule
Naive & $1.00\,[1.00,1.00]\ (8.01)$ & $1.00\,[1.00,1.00]\ (4.57)$ & $1.00\,[1.00,1.00]\ (1.47)$ \\
Plug-in (judge) & $1.57\,[1.07,2.31]\ (5.11)$ & $1.44\,[1.13,1.84]\ (3.17)$ & $1.30\,[1.09,1.55]\ (1.13)$ \\
Plug-in (multi) & $1.30\,[0.85,1.98]\ (6.16)$ & $1.23\,[0.85,1.78]\ (3.73)$ & $1.02\,[0.80,1.31]\ (1.44)$ \\
Aug. plug-in & $0.61\,[0.48,0.78]\ (13.11)$ & $0.78\,[0.61,1.00]\ (5.90)$ & $0.95\,[0.82,1.10]\ (1.54)$ \\
Global CV & $1.03\,[1.02,1.04]\ (7.79)$ & $1.04\,[1.02,1.06]\ (4.39)$ & $1.07\,[1.03,1.12]\ (1.37)$ \\
Per-signal CV & $0.93\,[0.89,0.97]\ (8.66)$ & $1.00\,[0.97,1.03]\ (4.59)$ & $1.01\,[0.99,1.03]\ (1.45)$ \\
Residual-only & $1.01\,[0.97,1.06]\ (7.90)$ & $0.98\,[0.89,1.08]\ (4.66)$ & $0.54\,[0.39,0.74]\ (2.72)$ \\
Scalar prediction CV & $0.60\,[0.44,0.83]\ (13.33)$ & $0.92\,[0.84,1.01]\ (4.99)$ & $1.02\,[0.95,1.09]\ (1.44)$ \\
Local PPI & $0.60\,[0.54,0.65]\ (13.46)$ & $0.65\,[0.58,0.71]\ (7.08)$ & $0.59\,[0.53,0.66]\ (2.48)$ \\
PPCI & $0.61\,[0.48,0.78]\ (13.23)$ & $0.74\,[0.57,0.96]\ (6.21)$ & $0.88\,[0.74,1.04]\ (1.66)$ \\
StratPPI & $0.68\,[0.59,0.79]\ (11.83)$ & $0.72\,[0.62,0.84]\ (6.32)$ & $0.64\,[0.52,0.79]\ (2.30)$ \\
\method{} (feasible) & $\mathbf{3.51}\,[3.43,3.59]\ (2.29)$ & $\mathbf{3.39}\,[3.28,3.51]\ (1.35)$ & $\mathbf{3.01}\,[2.87,3.16]\ (0.49)$ \\
\bottomrule
\end{tabular}%
}
\end{table}

\begin{table}[p]
\centering\scriptsize
\setlength{\tabcolsep}{3pt}
\caption{RE $[95\%\ \mathrm{MC\ CI}]$ on TruthfulQA (Qwen3 32B, continuous $Z$, $B=100$); mean MSE $\times 10^3$ is in parentheses. Intervals use paired split-level delta-method influence values. The target is the realized full-pool gold kernel profile; uncertainty is over matched labeled-subset draws conditional on the observed benchmark pool.}
\label{tab:real:truthfulqa:qwen3-32b}
\resizebox{\textwidth}{!}{%
\begin{tabular}{lccc}
\toprule
Estimator & $n_{\mathrm{lab}}=50$ & $n_{\mathrm{lab}}=100$ & $n_{\mathrm{lab}}=200$ \\
\midrule
Naive & $1.00\,[1.00,1.00]\ (22.10)$ & $1.00\,[1.00,1.00]\ (9.79)$ & $1.00\,[1.00,1.00]\ (3.81)$ \\
Plug-in (judge) & $3.85\,[2.40,6.18]\ (5.74)$ & $3.46\,[2.12,5.66]\ (2.83)$ & $1.39\,[0.77,2.52]\ (2.75)$ \\
Plug-in (multi) & $4.87\,[1.19,19.93]\ (4.54)$ & $4.38\,[2.69,7.14]\ (2.24)$ & $1.74\,[1.10,2.76]\ (2.19)$ \\
Aug. plug-in & $1.89\,[0.89,4.02]\ (11.68)$ & $2.47\,[1.29,4.71]\ (3.97)$ & $4.79\,[2.52,9.09]\ (0.80)$ \\
Global CV & $2.16\,[1.84,2.54]\ (10.22)$ & $2.00\,[1.70,2.35]\ (4.89)$ & $3.81\,[2.31,6.30]\ (1.00)$ \\
Per-signal CV & $1.02\,[0.95,1.10]\ (21.74)$ & $1.21\,[1.13,1.29]\ (8.09)$ & $1.32\,[1.24,1.41]\ (2.89)$ \\
Residual-only & $0.90\,[0.61,1.33]\ (24.67)$ & $0.75\,[0.49,1.14]\ (13.08)$ & $0.21\,[0.08,0.52]\ (18.23)$ \\
Scalar prediction CV & $1.42\,[0.99,2.04]\ (15.61)$ & $2.13\,[1.49,3.05]\ (4.60)$ & $3.88\,[2.56,5.89]\ (0.98)$ \\
Local PPI & $0.49\,[0.30,0.69]\ (44.82)$ & $0.87\,[0.75,0.99]\ (11.25)$ & $0.77\,[0.63,0.92]\ (4.92)$ \\
PPCI & $2.62\,[1.63,4.21]\ (8.43)$ & $3.33\,[1.97,5.64]\ (2.94)$ & $3.57\,[1.82,7.02]\ (1.07)$ \\
StratPPI & $2.04\,[0.57,7.30]\ (10.82)$ & $1.36\,[0.27,6.79]\ (7.22)$ & $0.82\,[0.62,1.09]\ (4.64)$ \\
\method{} (feasible) & $\mathbf{5.36}\,[4.04,7.12]\ (4.12)$ & $\mathbf{7.08}\,[5.54,9.04]\ (1.38)$ & $\mathbf{9.89}\,[6.15,15.90]\ (0.39)$ \\
\bottomrule
\end{tabular}%
}
\end{table}

\begin{table}[p]
\centering\scriptsize
\setlength{\tabcolsep}{3pt}
\caption{RE $[95\%\ \mathrm{MC\ CI}]$ on GSM8K (Qwen3 32B, continuous $Z$, $B=100$); mean MSE $\times 10^3$ is in parentheses. Intervals use paired split-level delta-method influence values. The target is the realized full-pool gold kernel profile; uncertainty is over matched labeled-subset draws conditional on the observed benchmark pool.}
\label{tab:real:gsm8k:qwen3-32b}
\resizebox{\textwidth}{!}{%
\begin{tabular}{lccc}
\toprule
Estimator & $n_{\mathrm{lab}}=50$ & $n_{\mathrm{lab}}=100$ & $n_{\mathrm{lab}}=200$ \\
\midrule
Naive & $1.00\,[1.00,1.00]\ (2.19)$ & $1.00\,[1.00,1.00]\ (1.17)$ & $1.00\,[1.00,1.00]\ (0.34)$ \\
Plug-in (judge) & $2.49\,[1.55,4.01]\ (0.88)$ & $2.29\,[1.71,3.08]\ (0.51)$ & $1.11\,[0.84,1.47]\ (0.31)$ \\
Plug-in (multi) & $3.90\,[1.18,12.90]\ (0.56)$ & $3.29\,[1.77,6.11]\ (0.35)$ & $1.37\,[0.79,2.38]\ (0.25)$ \\
Aug. plug-in & $0.83\,[0.44,1.57]\ (2.65)$ & $0.84\,[0.41,1.73]\ (1.38)$ & $0.79\,[0.07,8.45]\ (0.43)$ \\
Global CV & $1.24\,[1.15,1.34]\ (1.76)$ & $1.27\,[1.17,1.38]\ (0.92)$ & $1.22\,[0.89,1.67]\ (0.28)$ \\
Per-signal CV & $0.56\,[0.50,0.62]\ (3.89)$ & $1.07\,[1.01,1.13]\ (1.09)$ & $1.02\,[0.96,1.09]\ (0.33)$ \\
Residual-only & $1.49\,[1.10,2.02]\ (1.47)$ & $1.49\,[1.05,2.11]\ (0.78)$ & $0.97\,[0.71,1.33]\ (0.35)$ \\
Scalar prediction CV & $0.52\,[0.45,0.60]\ (4.18)$ & $0.84\,[0.71,0.99]\ (1.38)$ & $0.88\,[0.68,1.14]\ (0.38)$ \\
Local PPI & $0.56\,[0.48,0.64]\ (3.99)$ & $0.62\,[0.53,0.70]\ (1.89)$ & $0.50\,[0.42,0.58]\ (0.69)$ \\
PPCI & $0.86\,[0.49,1.51]\ (2.53)$ & $0.81\,[0.44,1.48]\ (1.43)$ & $0.65\,[0.16,2.63]\ (0.52)$ \\
StratPPI & $0.30\,[0.11,0.82]\ (7.33)$ & $0.51\,[0.41,0.64]\ (2.28)$ & $0.52\,[0.38,0.71]\ (0.65)$ \\
\method{} (feasible) & $\mathbf{6.20}\,[3.26,11.78]\ (0.35)$ & $\mathbf{5.90}\,[4.24,8.21]\ (0.20)$ & $\mathbf{3.85}\,[2.71,5.46]\ (0.09)$ \\
\bottomrule
\end{tabular}%
}
\end{table}

\begin{table}[p]
\centering\scriptsize
\setlength{\tabcolsep}{3pt}
\caption{RE $[95\%\ \mathrm{MC\ CI}]$ on ARC (Qwen3 32B, continuous $Z$, $B=100$); mean MSE $\times 10^3$ is in parentheses. Intervals use paired split-level delta-method influence values. The target is the realized full-pool gold kernel profile; uncertainty is over matched labeled-subset draws conditional on the observed benchmark pool.}
\label{tab:real:arc:qwen3-32b}
\resizebox{\textwidth}{!}{%
\begin{tabular}{lccc}
\toprule
Estimator & $n_{\mathrm{lab}}=50$ & $n_{\mathrm{lab}}=100$ & $n_{\mathrm{lab}}=200$ \\
\midrule
Naive & $1.00\,[1.00,1.00]\ (1.85)$ & $1.00\,[1.00,1.00]\ (0.81)$ & $1.00\,[1.00,1.00]\ (0.26)$ \\
Plug-in (judge) & $4.41\,[2.91,6.69]\ (0.42)$ & $4.00\,[2.43,6.59]\ (0.20)$ & $2.30\,[1.81,2.93]\ (0.11)$ \\
Plug-in (multi) & $3.65\,[2.64,5.04]\ (0.51)$ & $3.28\,[2.10,5.12]\ (0.25)$ & $1.60\,[1.05,2.44]\ (0.16)$ \\
Aug. plug-in & $0.75\,[0.45,1.26]\ (2.47)$ & $0.84\,[0.56,1.25]\ (0.97)$ & $0.82\,[0.54,1.24]\ (0.32)$ \\
Global CV & $1.26\,[0.93,1.70]\ (1.47)$ & $1.42\,[1.27,1.58]\ (0.57)$ & $1.21\,[1.06,1.39]\ (0.21)$ \\
Per-signal CV & $0.79\,[0.74,0.85]\ (2.34)$ & $1.03\,[0.96,1.10]\ (0.79)$ & $1.04\,[1.00,1.08]\ (0.25)$ \\
Residual-only & $1.56\,[1.28,1.90]\ (1.18)$ & $1.67\,[1.21,2.30]\ (0.49)$ & $1.28\,[1.02,1.61]\ (0.20)$ \\
Scalar prediction CV & $0.41\,[0.16,1.04]\ (4.48)$ & $0.86\,[0.51,1.45]\ (0.95)$ & $1.06\,[0.82,1.38]\ (0.24)$ \\
Local PPI & $0.41\,[0.30,0.52]\ (4.52)$ & $0.47\,[0.39,0.54]\ (1.74)$ & $0.37\,[0.30,0.43]\ (0.71)$ \\
PPCI & $0.80\,[0.49,1.30]\ (2.33)$ & $0.77\,[0.49,1.22]\ (1.06)$ & $0.63\,[0.41,0.97]\ (0.41)$ \\
StratPPI & $0.55\,[0.25,1.22]\ (3.36)$ & $0.84\,[0.14,5.08]\ (0.97)$ & $1.63\,[0.87,3.06]\ (0.16)$ \\
\method{} (feasible) & $\mathbf{5.95}\,[4.95,7.15]\ (0.31)$ & $\mathbf{6.70}\,[4.35,10.32]\ (0.12)$ & $\mathbf{5.17}\,[4.45,6.01]\ (0.05)$ \\
\bottomrule
\end{tabular}%
}
\end{table}

\begin{figure}[htbp]
\centering
\begin{subfigure}[b]{0.12\textwidth}
\centering
\vspace{1.2cm}
\textbf{MATH-500}
\end{subfigure}
\hfill
\begin{subfigure}[b]{0.27\textwidth}
\centering
\includegraphics[width=\textwidth]{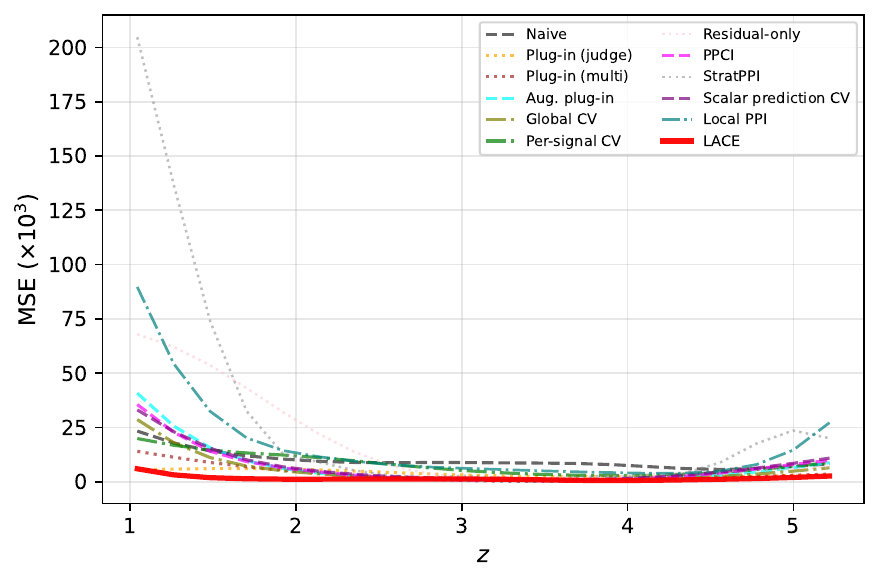}
\end{subfigure}
\hfill
\begin{subfigure}[b]{0.27\textwidth}
\centering
\includegraphics[width=\textwidth]{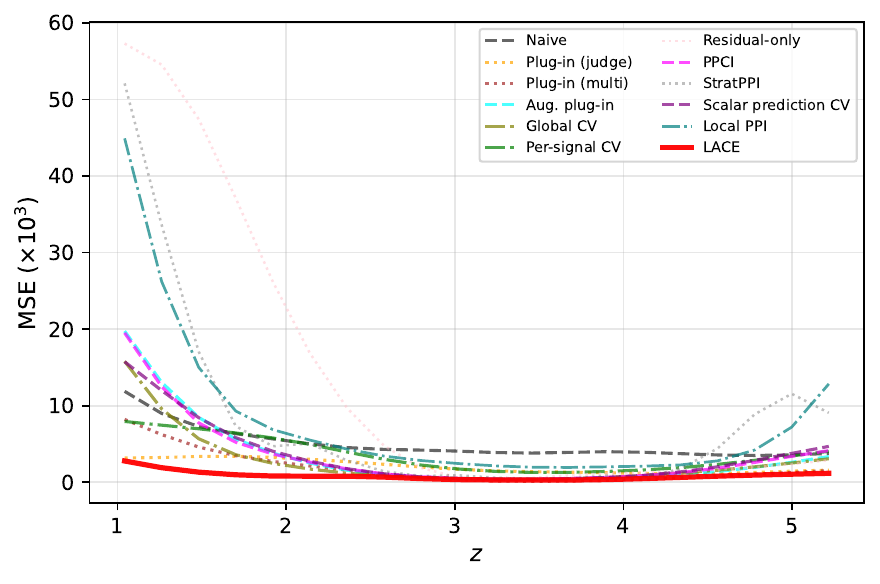}
\end{subfigure}
\hfill
\begin{subfigure}[b]{0.27\textwidth}
\centering
\includegraphics[width=\textwidth]{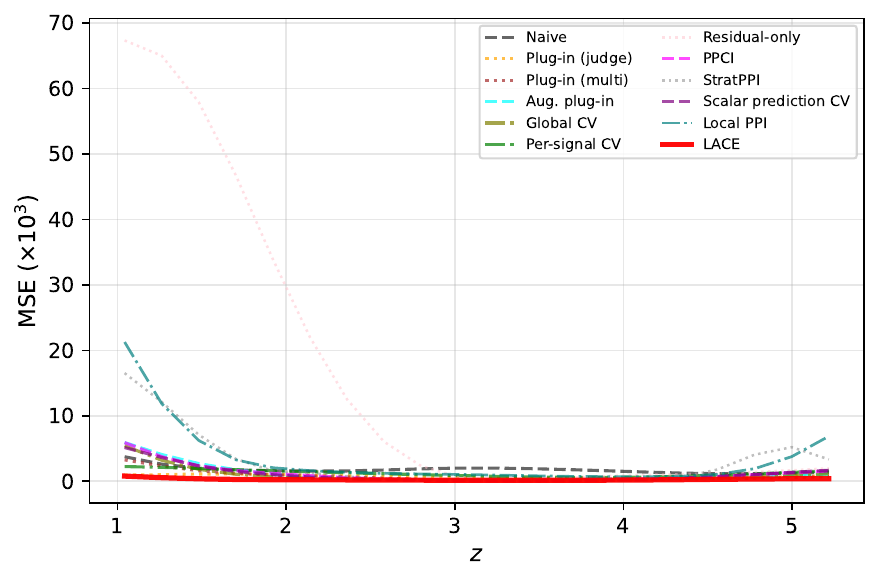}
\end{subfigure}
\par\smallskip
\begin{subfigure}[b]{0.12\textwidth}
\centering
\vspace{1.2cm}
\textbf{ScienceQA}
\end{subfigure}
\hfill
\begin{subfigure}[b]{0.27\textwidth}
\centering
\includegraphics[width=\textwidth]{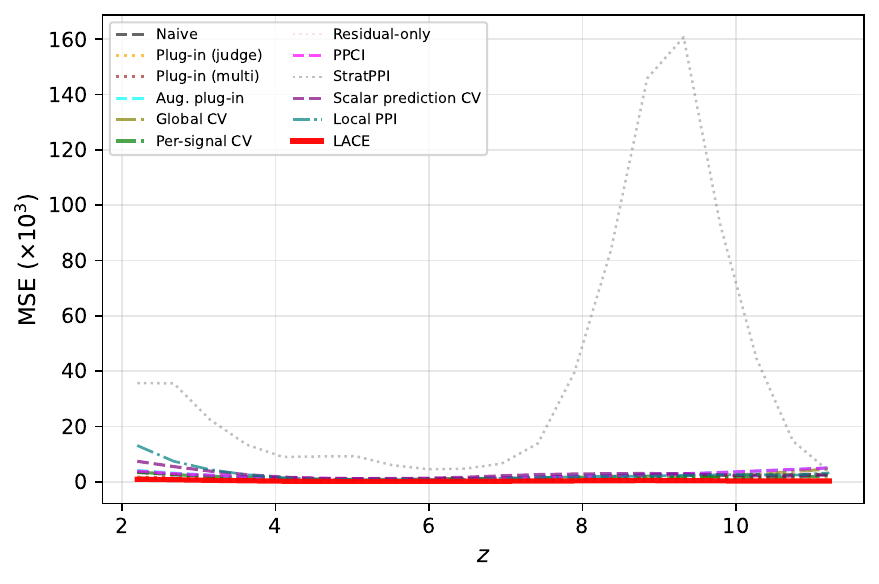}
\end{subfigure}
\hfill
\begin{subfigure}[b]{0.27\textwidth}
\centering
\includegraphics[width=\textwidth]{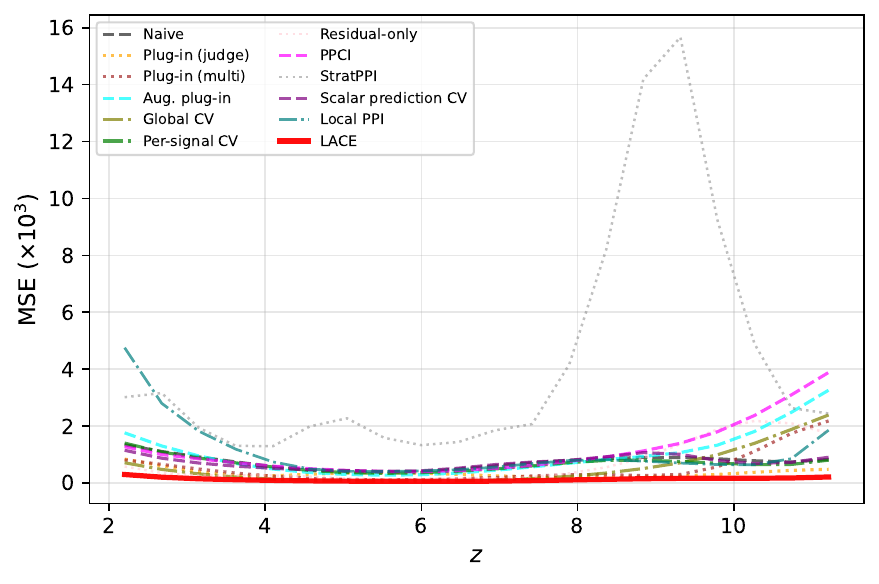}
\end{subfigure}
\hfill
\begin{subfigure}[b]{0.27\textwidth}
\centering
\includegraphics[width=\textwidth]{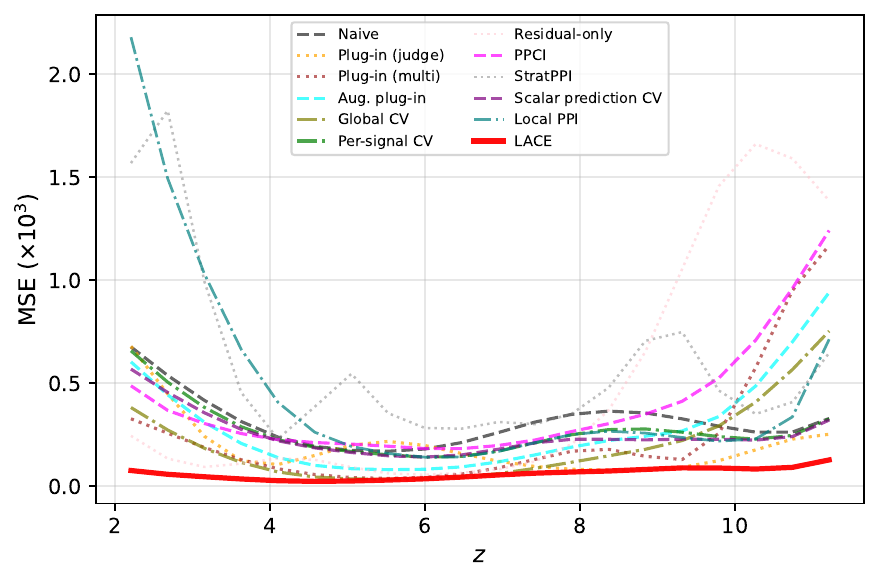}
\end{subfigure}
\par\smallskip
\begin{subfigure}[b]{0.12\textwidth}
\centering
\vspace{1.2cm}
\textbf{MMLU}
\end{subfigure}
\hfill
\begin{subfigure}[b]{0.27\textwidth}
\centering
\includegraphics[width=\textwidth]{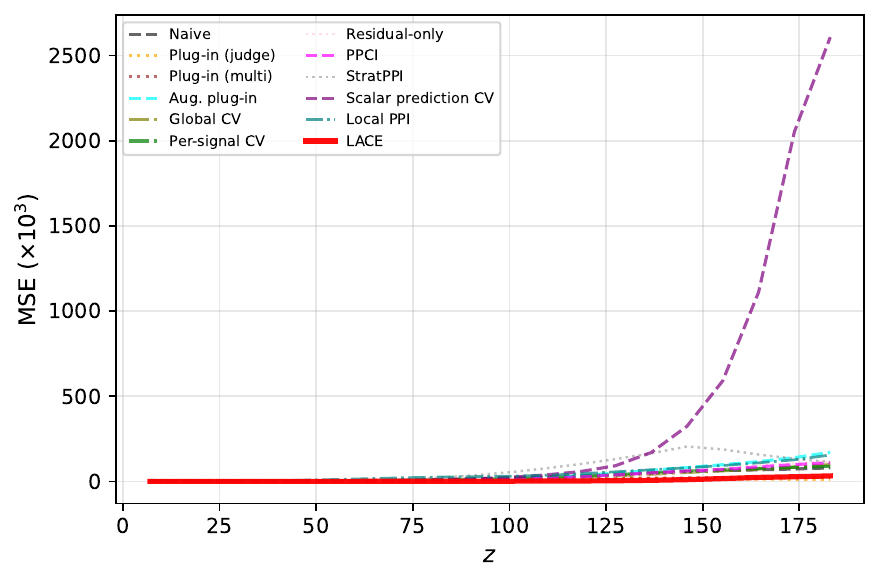}
\end{subfigure}
\hfill
\begin{subfigure}[b]{0.27\textwidth}
\centering
\includegraphics[width=\textwidth]{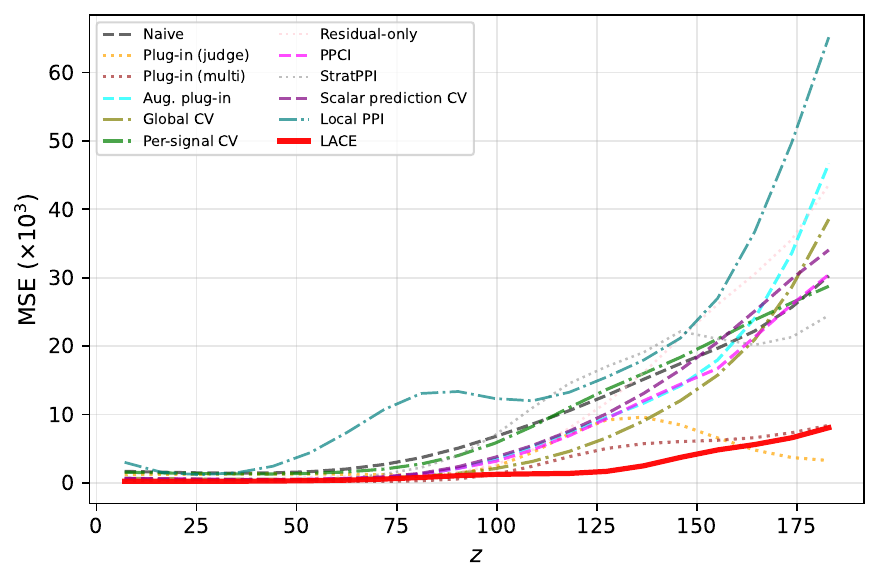}
\end{subfigure}
\hfill
\begin{subfigure}[b]{0.27\textwidth}
\centering
\includegraphics[width=\textwidth]{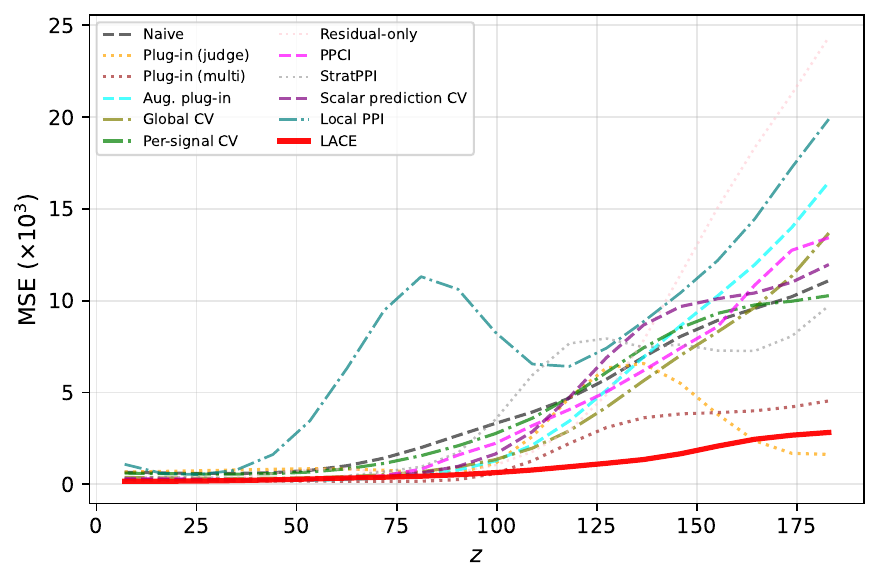}
\end{subfigure}
\par\smallskip
\begin{subfigure}[b]{0.12\textwidth}
\centering
\vspace{1.2cm}
\textbf{WinoGrande}
\end{subfigure}
\hfill
\begin{subfigure}[b]{0.27\textwidth}
\centering
\includegraphics[width=\textwidth]{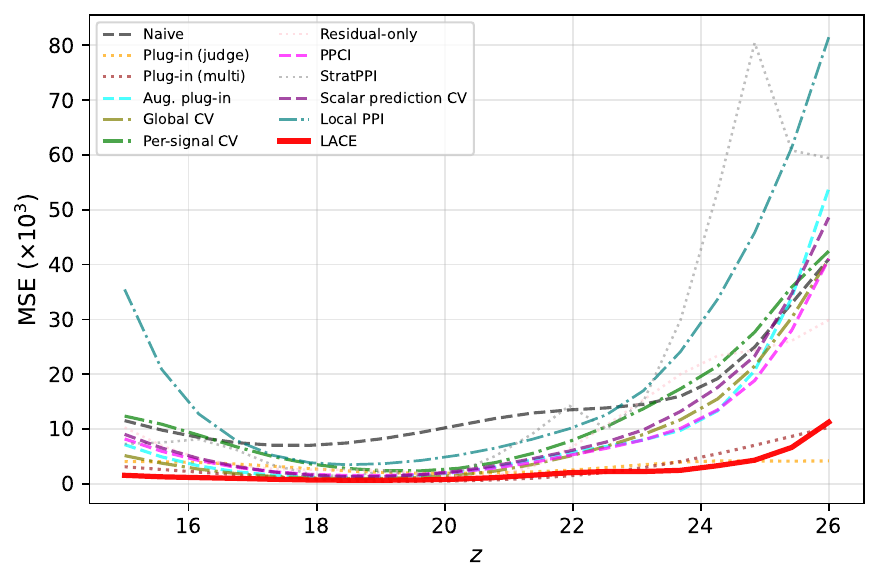}
\end{subfigure}
\hfill
\begin{subfigure}[b]{0.27\textwidth}
\centering
\includegraphics[width=\textwidth]{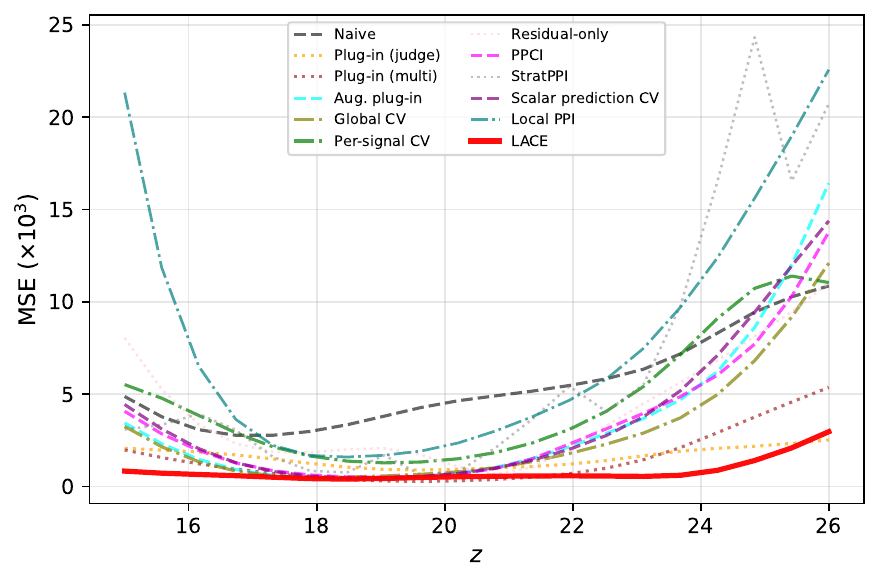}
\end{subfigure}
\hfill
\begin{subfigure}[b]{0.27\textwidth}
\centering
\includegraphics[width=\textwidth]{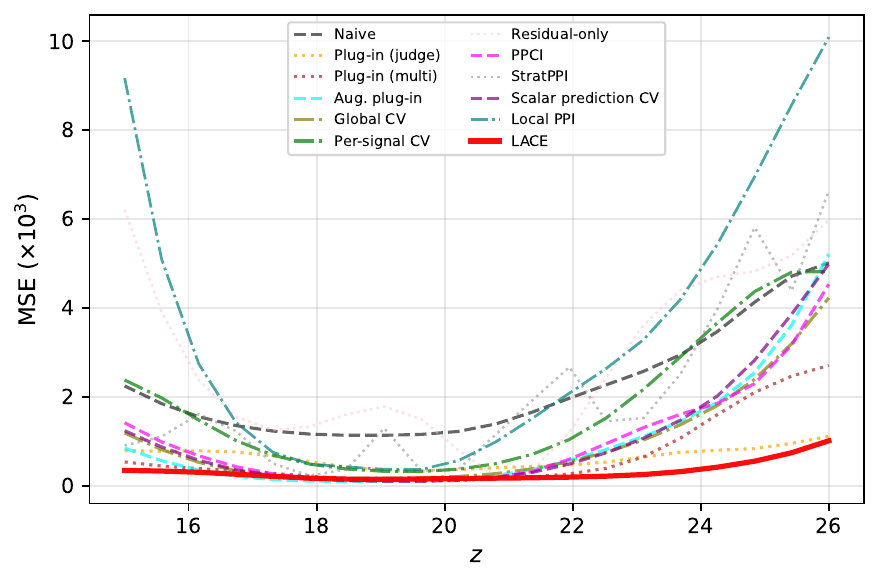}
\end{subfigure}
\par\smallskip
\begin{subfigure}[b]{0.12\textwidth}
\centering
\vspace{1.2cm}
\textbf{HellaSwag}
\end{subfigure}
\hfill
\begin{subfigure}[b]{0.27\textwidth}
\centering
\includegraphics[width=\textwidth]{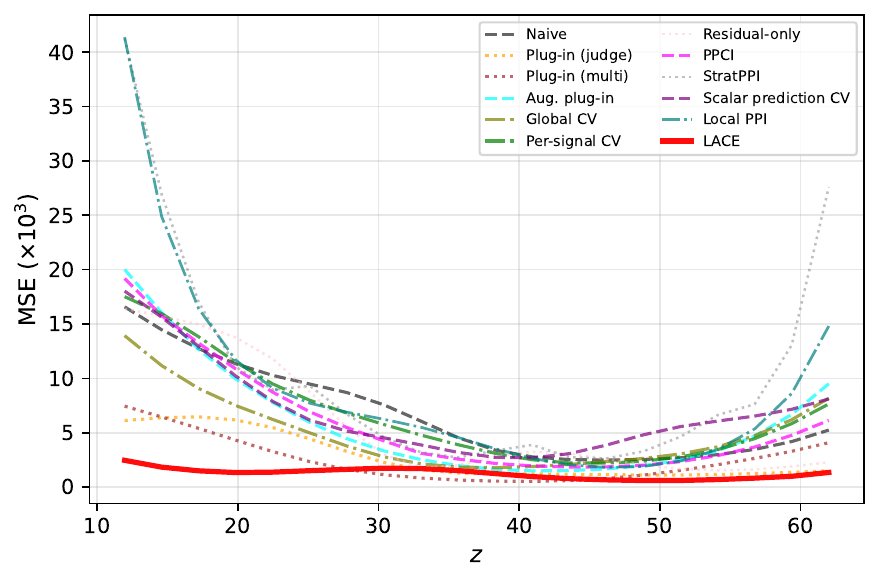}
\end{subfigure}
\hfill
\begin{subfigure}[b]{0.27\textwidth}
\centering
\includegraphics[width=\textwidth]{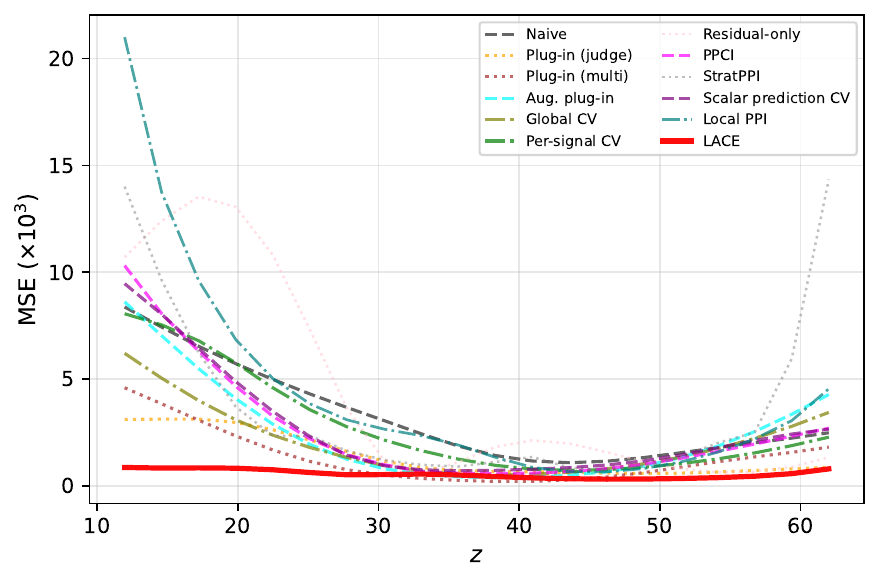}
\end{subfigure}
\hfill
\begin{subfigure}[b]{0.27\textwidth}
\centering
\includegraphics[width=\textwidth]{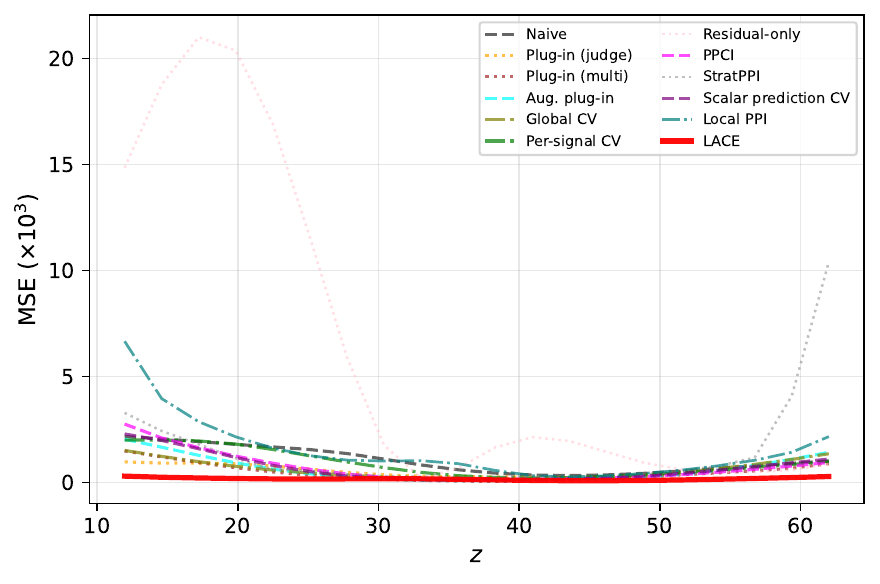}
\end{subfigure}
\par\smallskip
\begin{subfigure}[b]{0.12\textwidth}
\centering
\vspace{1.2cm}
\textbf{TruthfulQA}
\end{subfigure}
\hfill
\begin{subfigure}[b]{0.27\textwidth}
\centering
\includegraphics[width=\textwidth]{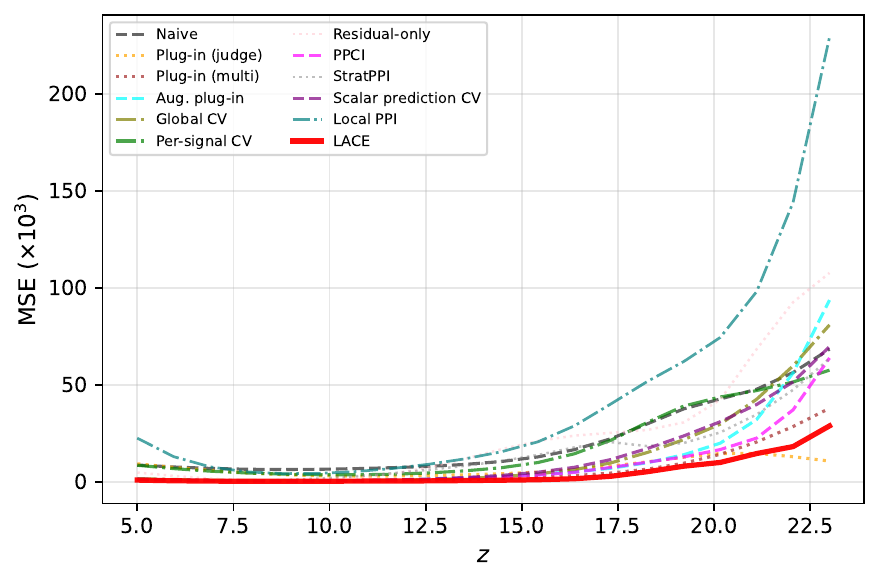}
\end{subfigure}
\hfill
\begin{subfigure}[b]{0.27\textwidth}
\centering
\includegraphics[width=\textwidth]{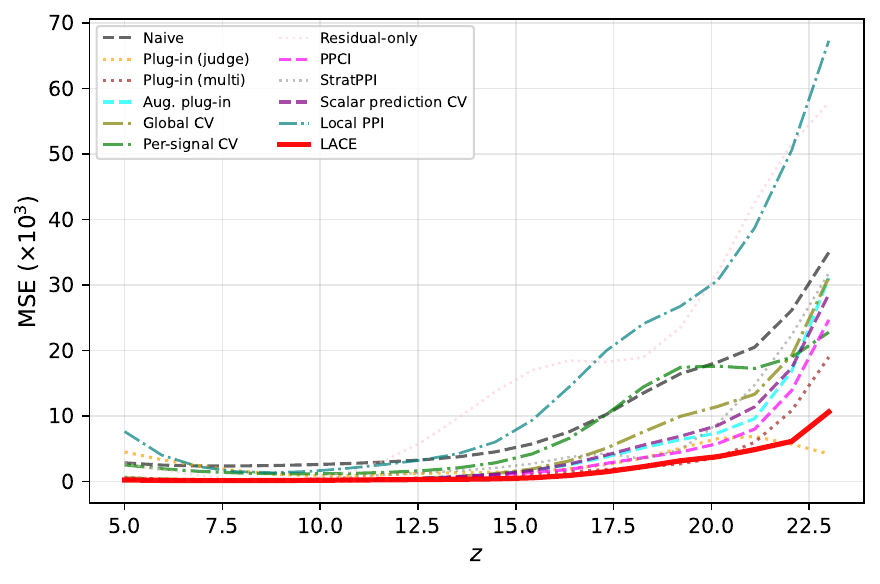}
\end{subfigure}
\hfill
\begin{subfigure}[b]{0.27\textwidth}
\centering
\includegraphics[width=\textwidth]{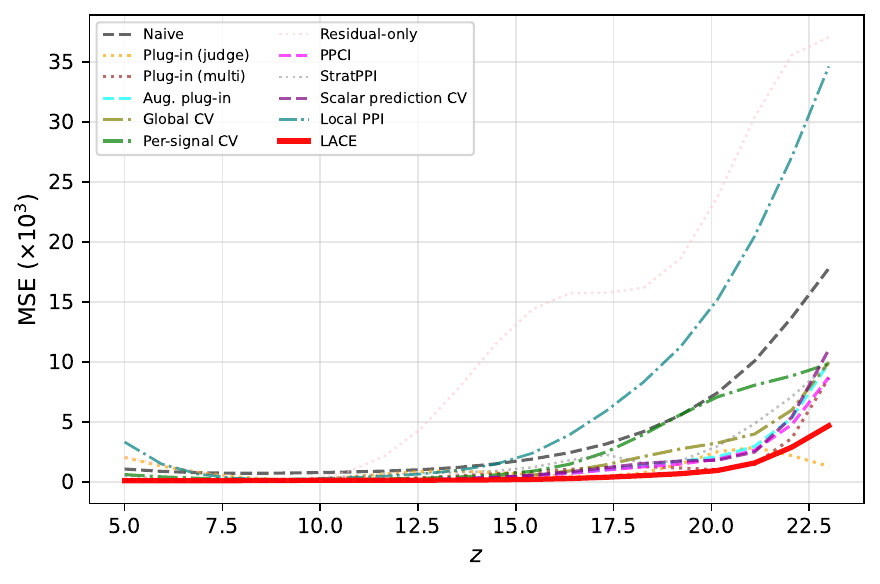}
\end{subfigure}
\par\smallskip
\begin{subfigure}[b]{0.12\textwidth}
\centering
\vspace{1.2cm}
\textbf{GSM8K}
\end{subfigure}
\hfill
\begin{subfigure}[b]{0.27\textwidth}
\centering
\includegraphics[width=\textwidth]{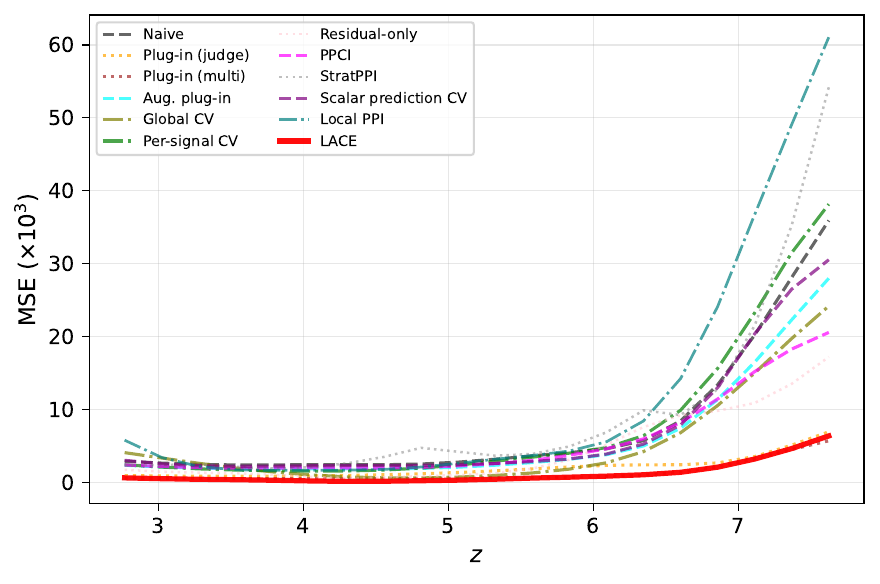}
\end{subfigure}
\hfill
\begin{subfigure}[b]{0.27\textwidth}
\centering
\includegraphics[width=\textwidth]{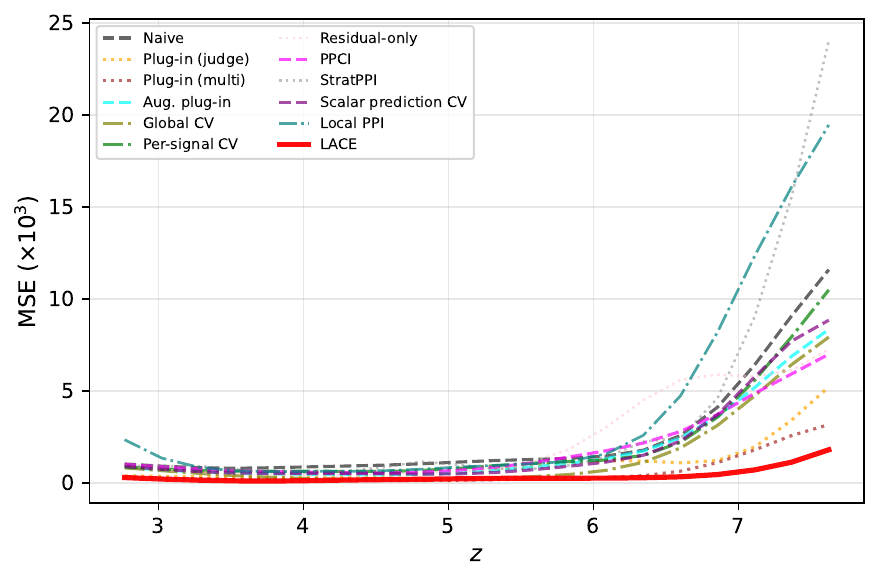}
\end{subfigure}
\hfill
\begin{subfigure}[b]{0.27\textwidth}
\centering
\includegraphics[width=\textwidth]{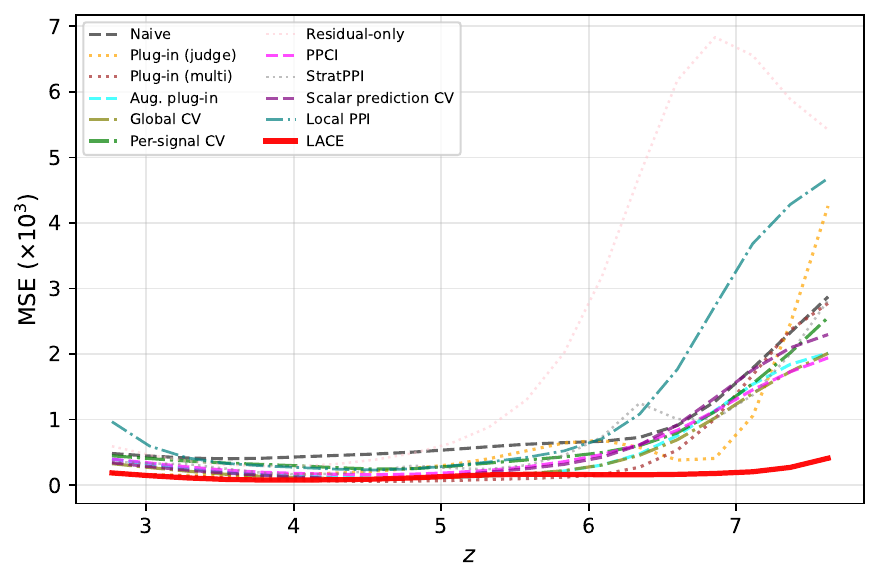}
\end{subfigure}
\par\smallskip
\begin{subfigure}[b]{0.12\textwidth}
\centering
\vspace{1.2cm}
\textbf{ARC}
\end{subfigure}
\hfill
\begin{subfigure}[b]{0.27\textwidth}
\centering
\includegraphics[width=\textwidth]{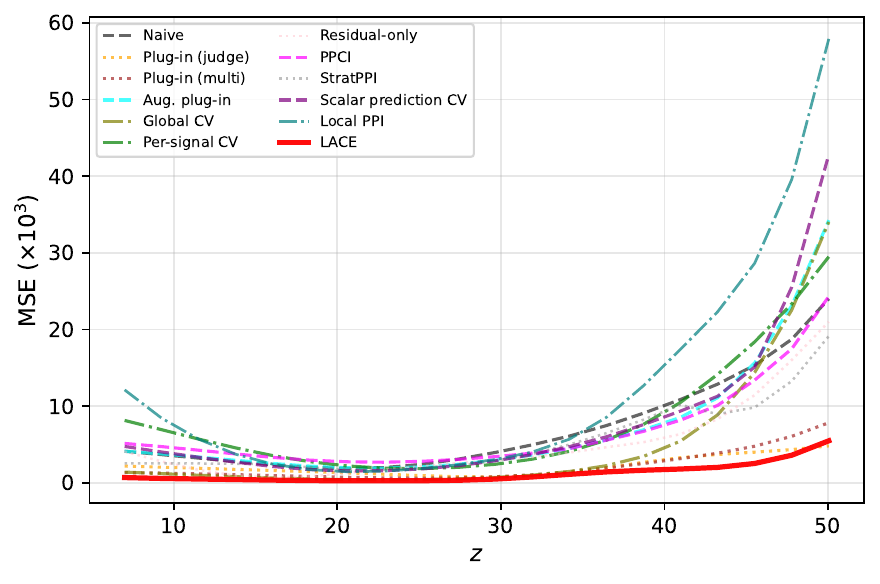}
\end{subfigure}
\hfill
\begin{subfigure}[b]{0.27\textwidth}
\centering
\includegraphics[width=\textwidth]{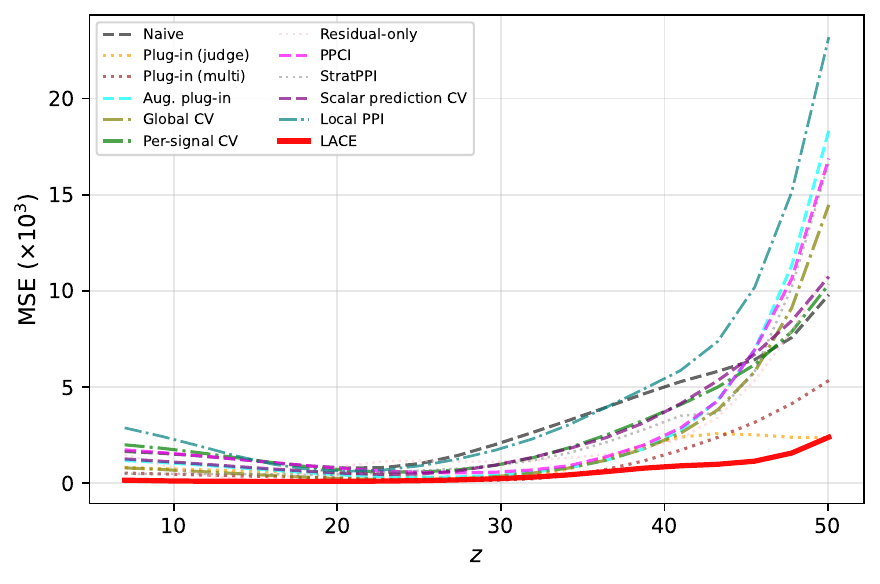}
\end{subfigure}
\hfill
\begin{subfigure}[b]{0.27\textwidth}
\centering
\includegraphics[width=\textwidth]{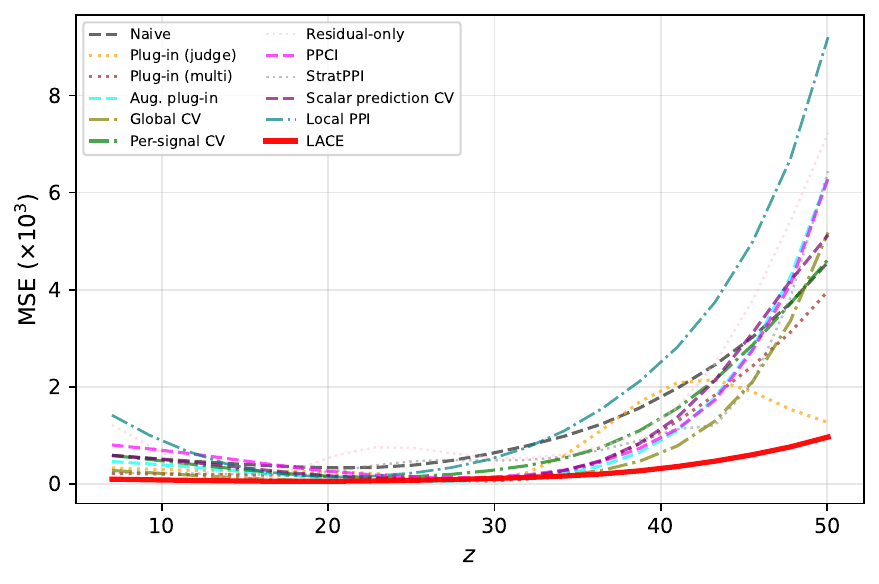}
\end{subfigure}
\caption{Pointwise MSE ($\times 10^3$) for Claude Haiku 3. Each row:
one dataset (label on left); columns are $n=50,100,200$.}
\label{fig:mse-haiku-3}
\end{figure}

\begin{figure}[htbp]
\centering
\begin{subfigure}[b]{0.12\textwidth}
\centering
\vspace{1.2cm}
\textbf{MATH-500}
\end{subfigure}
\hfill
\begin{subfigure}[b]{0.27\textwidth}
\centering
\includegraphics[width=\textwidth]{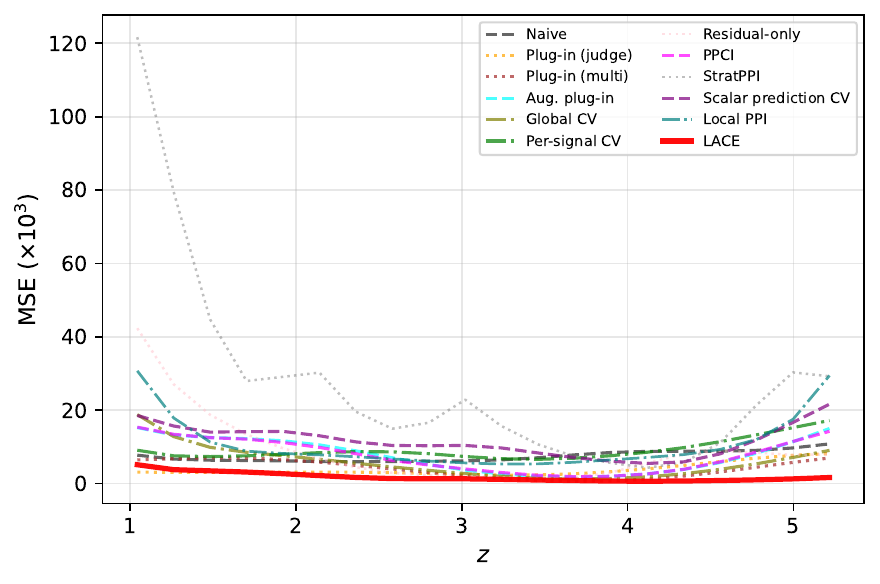}
\end{subfigure}
\hfill
\begin{subfigure}[b]{0.27\textwidth}
\centering
\includegraphics[width=\textwidth]{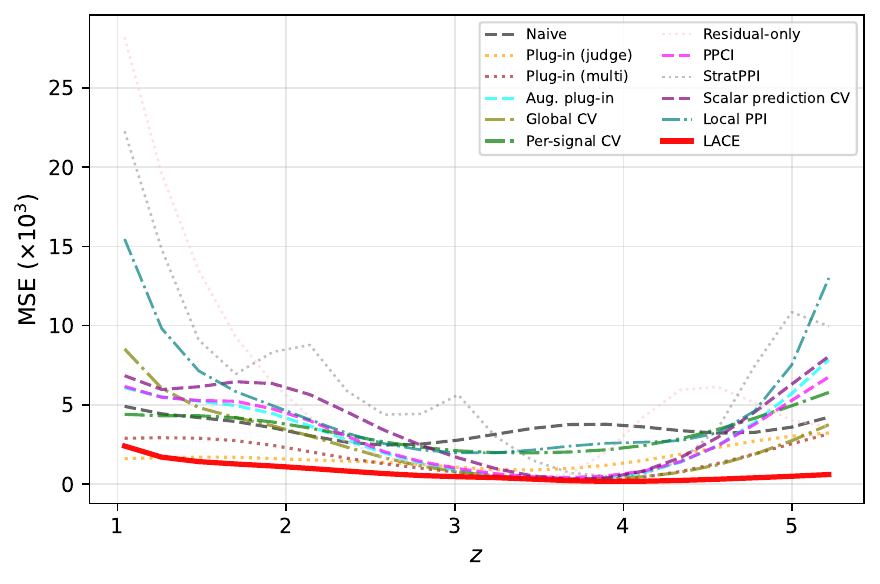}
\end{subfigure}
\hfill
\begin{subfigure}[b]{0.27\textwidth}
\centering
\includegraphics[width=\textwidth]{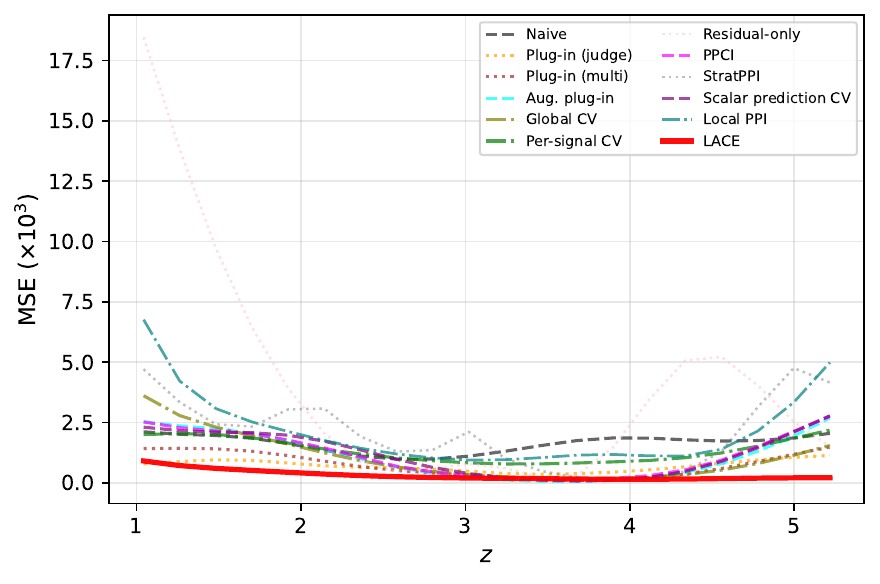}
\end{subfigure}
\par\smallskip
\begin{subfigure}[b]{0.12\textwidth}
\centering
\vspace{1.2cm}
\textbf{ScienceQA}
\end{subfigure}
\hfill
\begin{subfigure}[b]{0.27\textwidth}
\centering
\includegraphics[width=\textwidth]{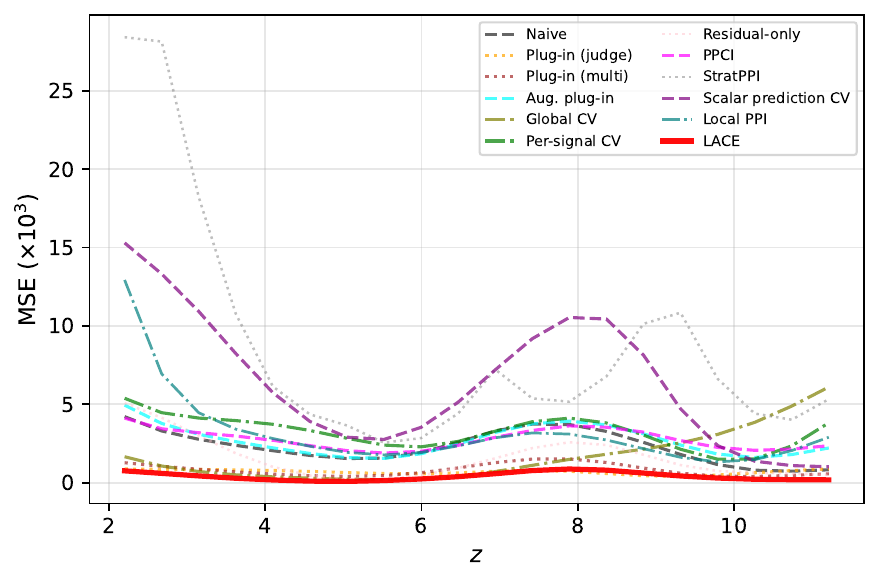}
\end{subfigure}
\hfill
\begin{subfigure}[b]{0.27\textwidth}
\centering
\includegraphics[width=\textwidth]{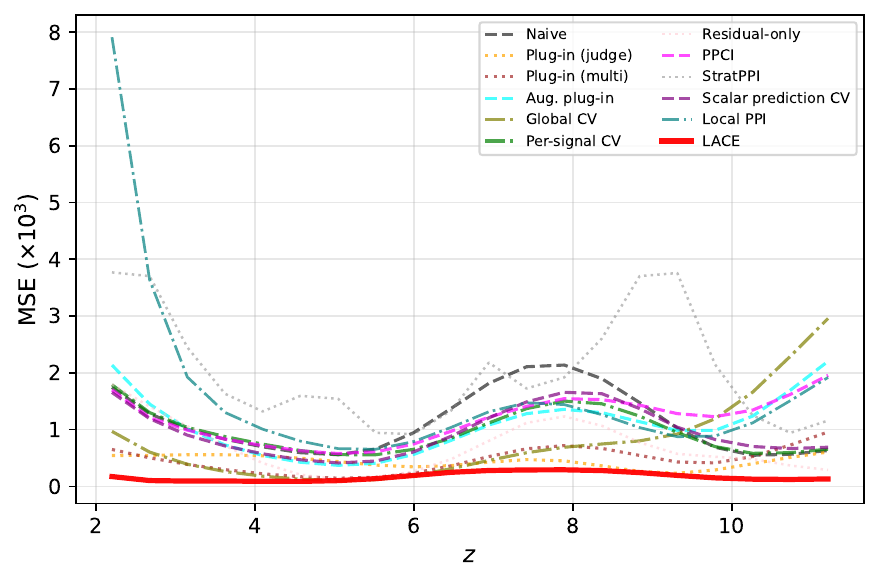}
\end{subfigure}
\hfill
\begin{subfigure}[b]{0.27\textwidth}
\centering
\includegraphics[width=\textwidth]{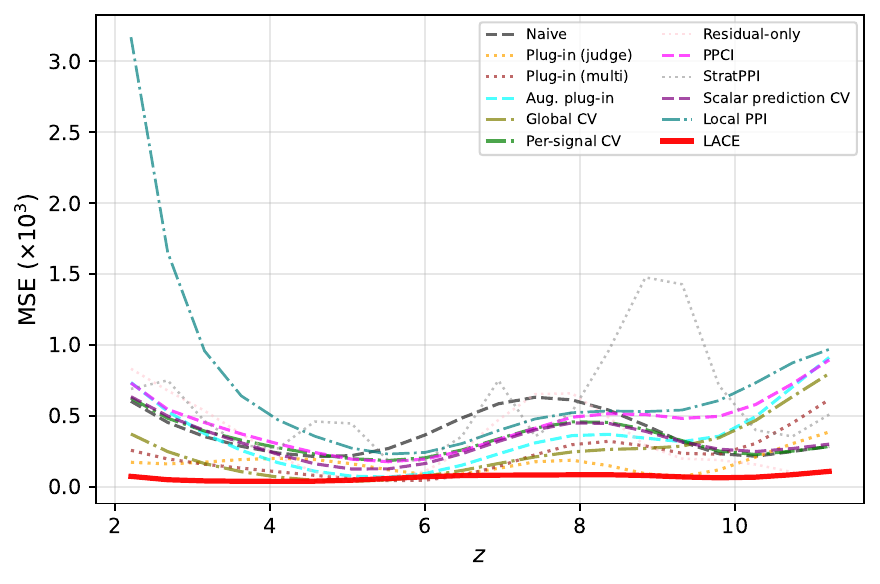}
\end{subfigure}
\par\smallskip
\begin{subfigure}[b]{0.12\textwidth}
\centering
\vspace{1.2cm}
\textbf{MMLU}
\end{subfigure}
\hfill
\begin{subfigure}[b]{0.27\textwidth}
\centering
\includegraphics[width=\textwidth]{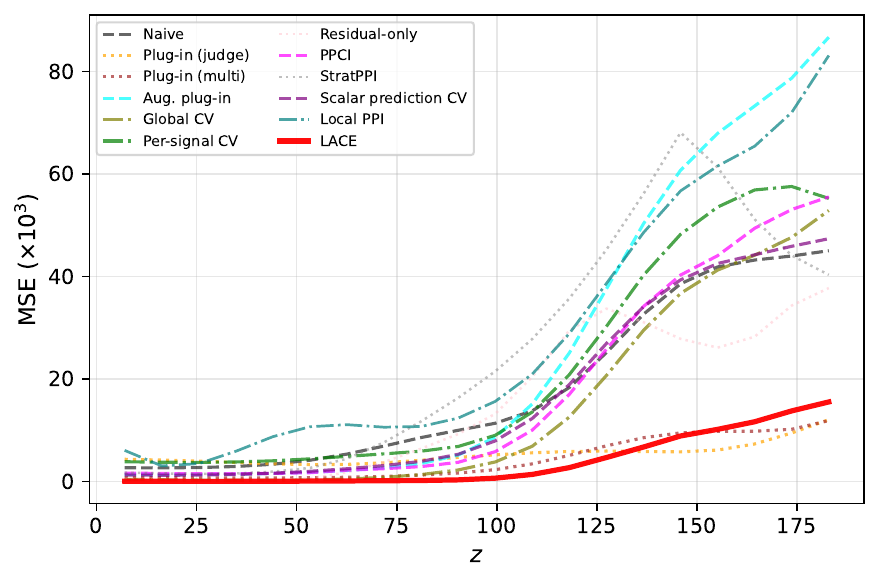}
\end{subfigure}
\hfill
\begin{subfigure}[b]{0.27\textwidth}
\centering
\includegraphics[width=\textwidth]{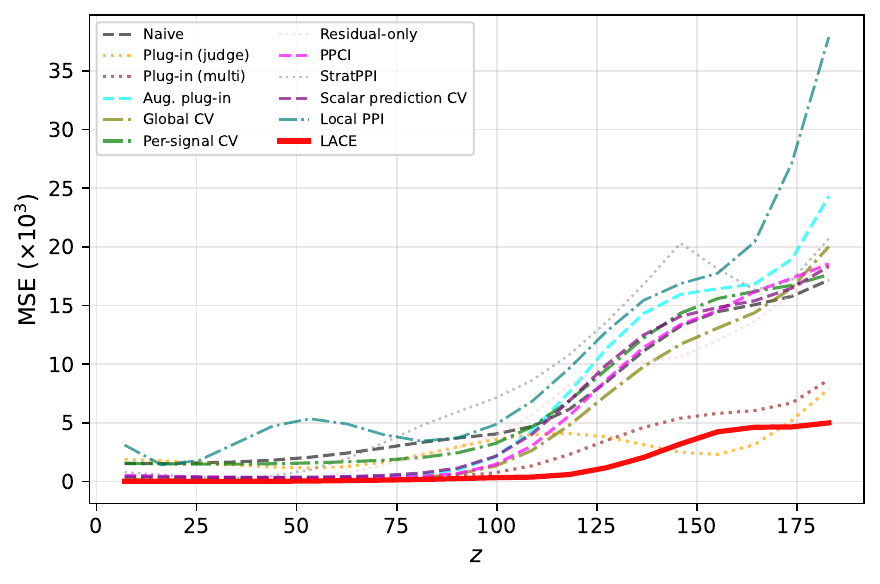}
\end{subfigure}
\hfill
\begin{subfigure}[b]{0.27\textwidth}
\centering
\includegraphics[width=\textwidth]{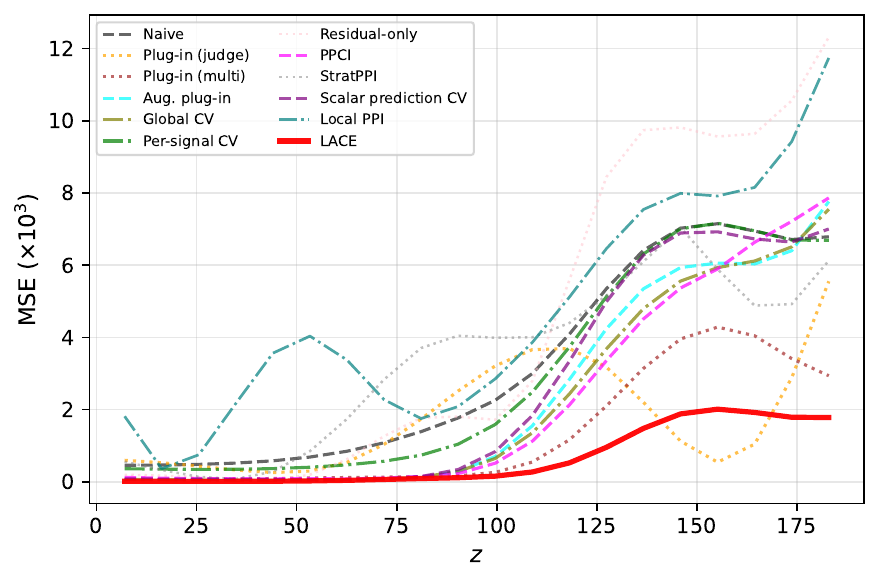}
\end{subfigure}
\par\smallskip
\begin{subfigure}[b]{0.12\textwidth}
\centering
\vspace{1.2cm}
\textbf{WinoGrande}
\end{subfigure}
\hfill
\begin{subfigure}[b]{0.27\textwidth}
\centering
\includegraphics[width=\textwidth]{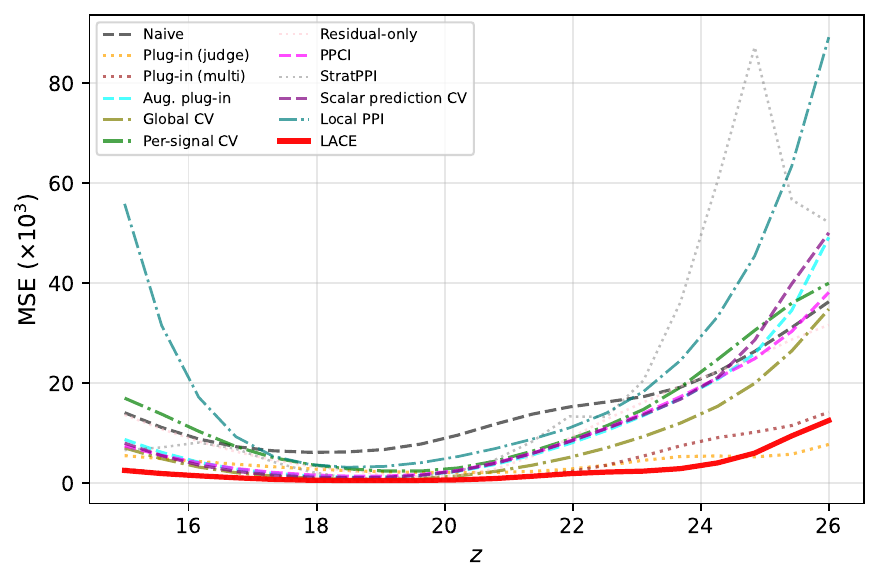}
\end{subfigure}
\hfill
\begin{subfigure}[b]{0.27\textwidth}
\centering
\includegraphics[width=\textwidth]{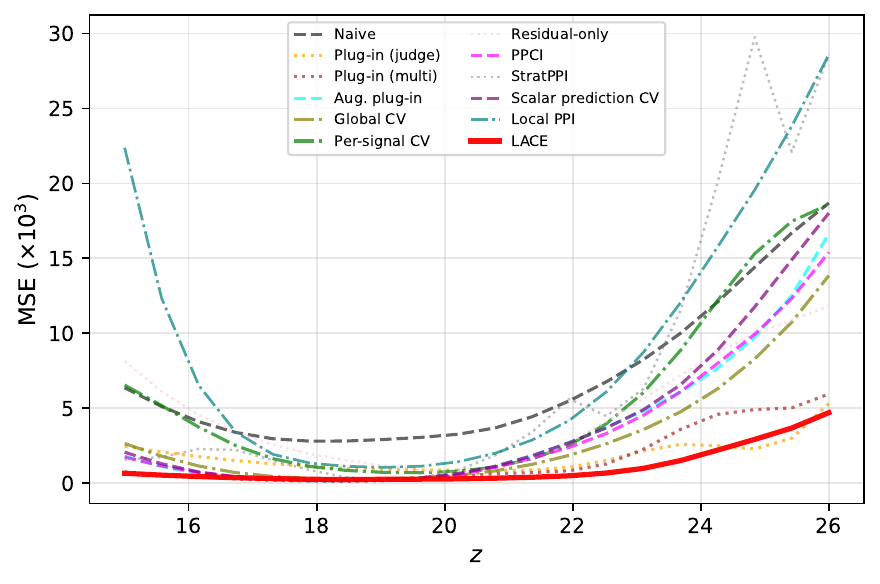}
\end{subfigure}
\hfill
\begin{subfigure}[b]{0.27\textwidth}
\centering
\includegraphics[width=\textwidth]{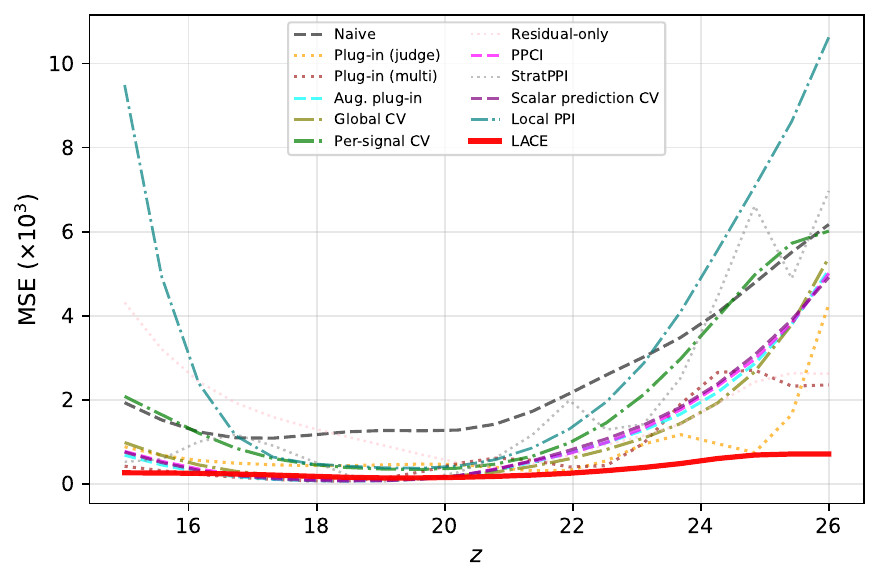}
\end{subfigure}
\par\smallskip
\begin{subfigure}[b]{0.12\textwidth}
\centering
\vspace{1.2cm}
\textbf{HellaSwag}
\end{subfigure}
\hfill
\begin{subfigure}[b]{0.27\textwidth}
\centering
\includegraphics[width=\textwidth]{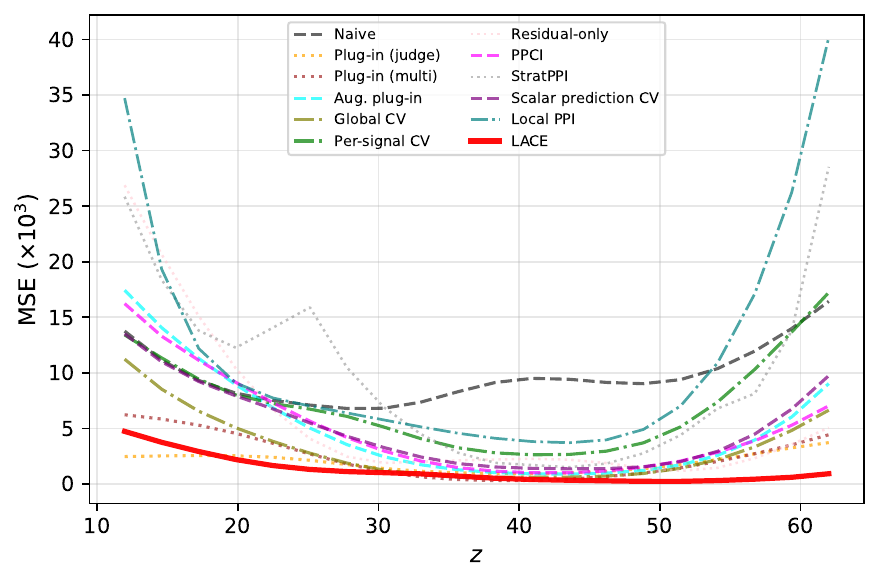}
\end{subfigure}
\hfill
\begin{subfigure}[b]{0.27\textwidth}
\centering
\includegraphics[width=\textwidth]{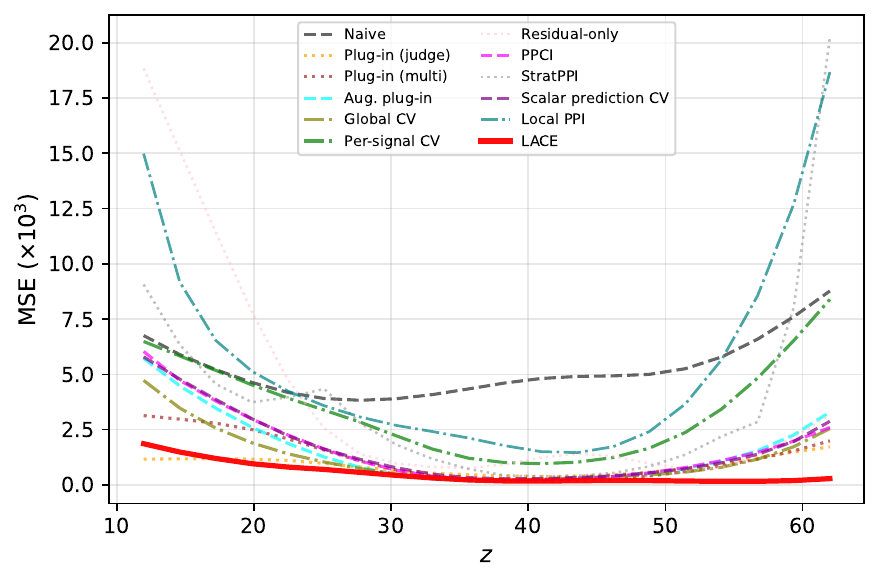}
\end{subfigure}
\hfill
\begin{subfigure}[b]{0.27\textwidth}
\centering
\includegraphics[width=\textwidth]{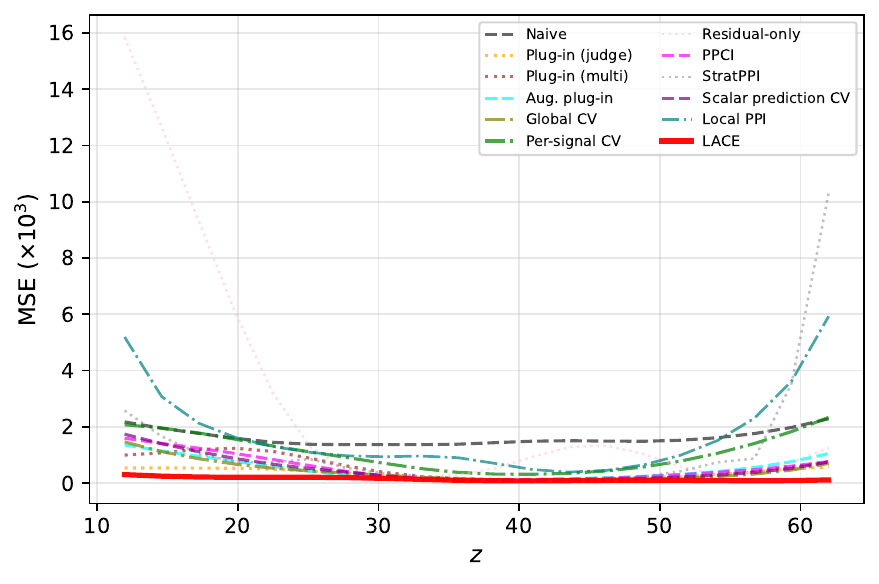}
\end{subfigure}
\par\smallskip
\begin{subfigure}[b]{0.12\textwidth}
\centering
\vspace{1.2cm}
\textbf{TruthfulQA}
\end{subfigure}
\hfill
\begin{subfigure}[b]{0.27\textwidth}
\centering
\includegraphics[width=\textwidth]{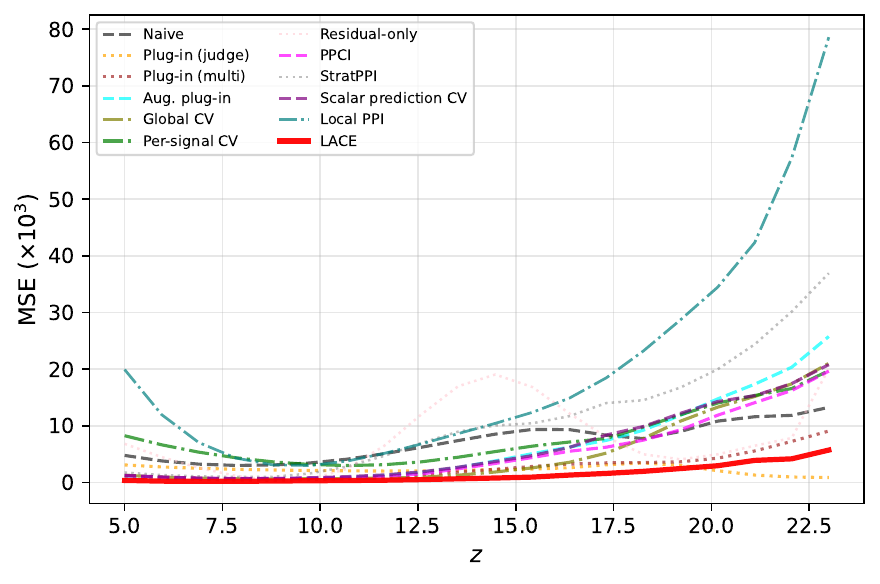}
\end{subfigure}
\hfill
\begin{subfigure}[b]{0.27\textwidth}
\centering
\includegraphics[width=\textwidth]{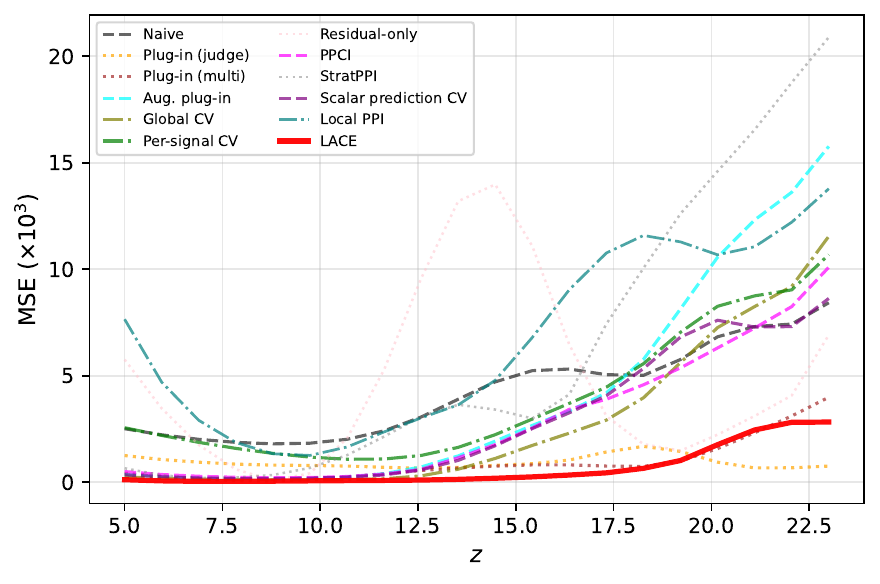}
\end{subfigure}
\hfill
\begin{subfigure}[b]{0.27\textwidth}
\centering
\includegraphics[width=\textwidth]{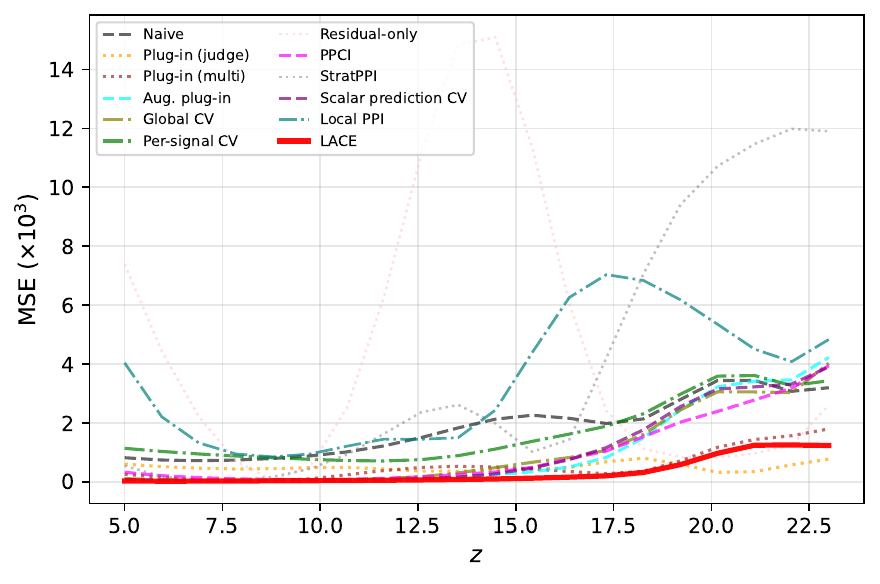}
\end{subfigure}
\par\smallskip
\begin{subfigure}[b]{0.12\textwidth}
\centering
\vspace{1.2cm}
\textbf{GSM8K}
\end{subfigure}
\hfill
\begin{subfigure}[b]{0.27\textwidth}
\centering
\includegraphics[width=\textwidth]{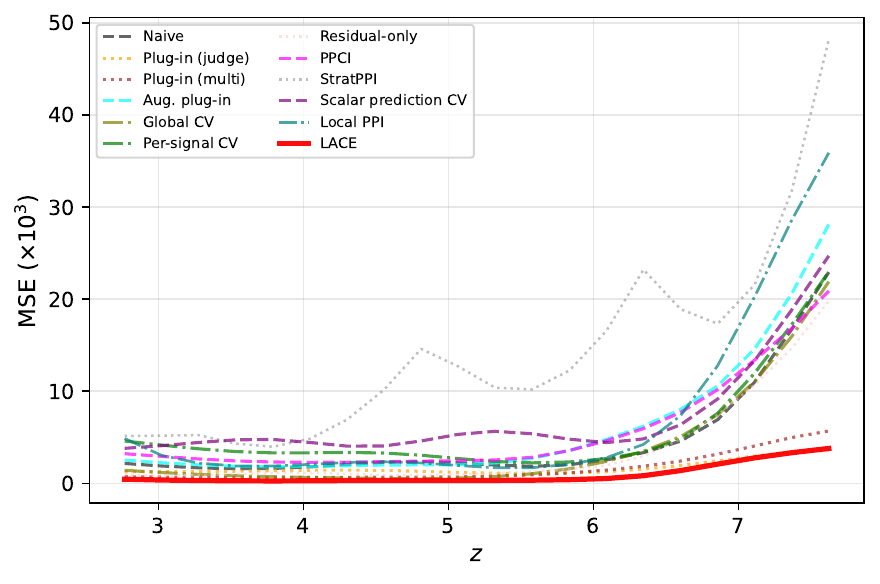}
\end{subfigure}
\hfill
\begin{subfigure}[b]{0.27\textwidth}
\centering
\includegraphics[width=\textwidth]{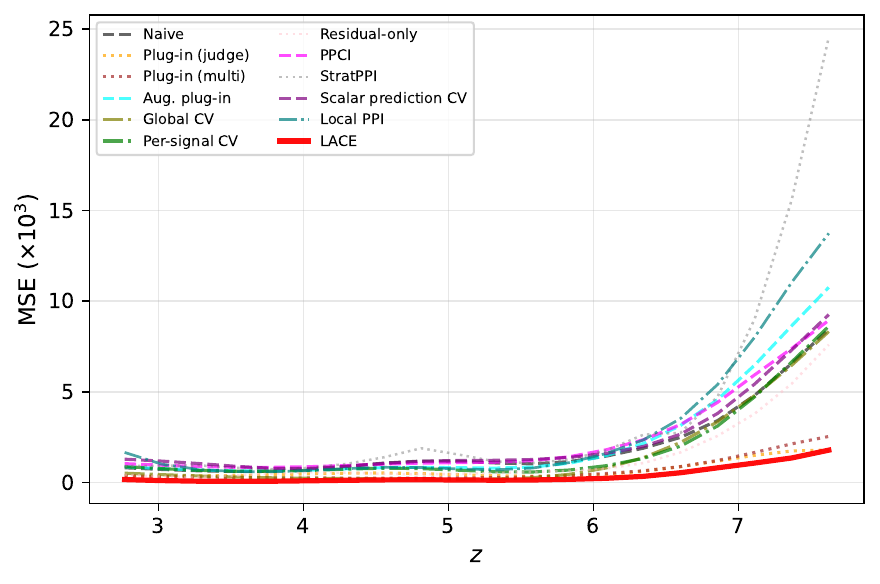}
\end{subfigure}
\hfill
\begin{subfigure}[b]{0.27\textwidth}
\centering
\includegraphics[width=\textwidth]{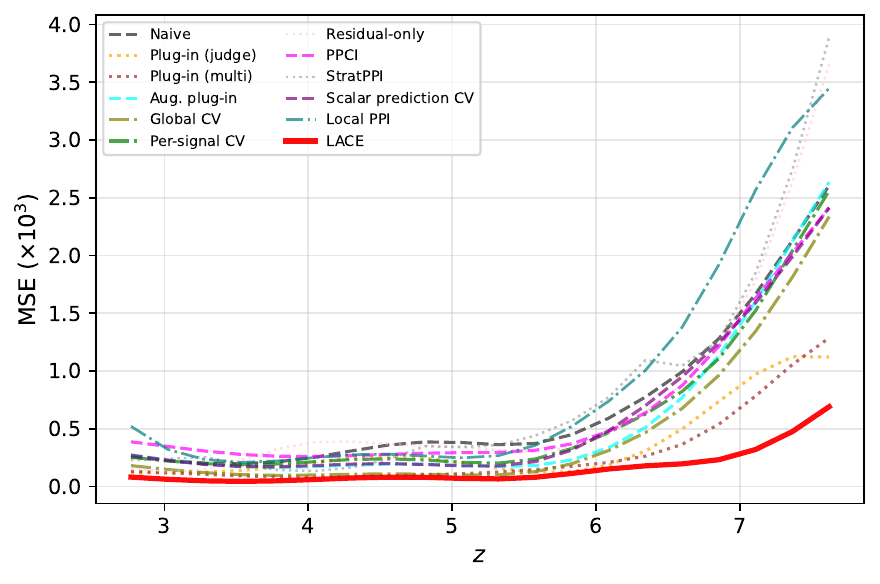}
\end{subfigure}
\par\smallskip
\begin{subfigure}[b]{0.12\textwidth}
\centering
\vspace{1.2cm}
\textbf{ARC}
\end{subfigure}
\hfill
\begin{subfigure}[b]{0.27\textwidth}
\centering
\includegraphics[width=\textwidth]{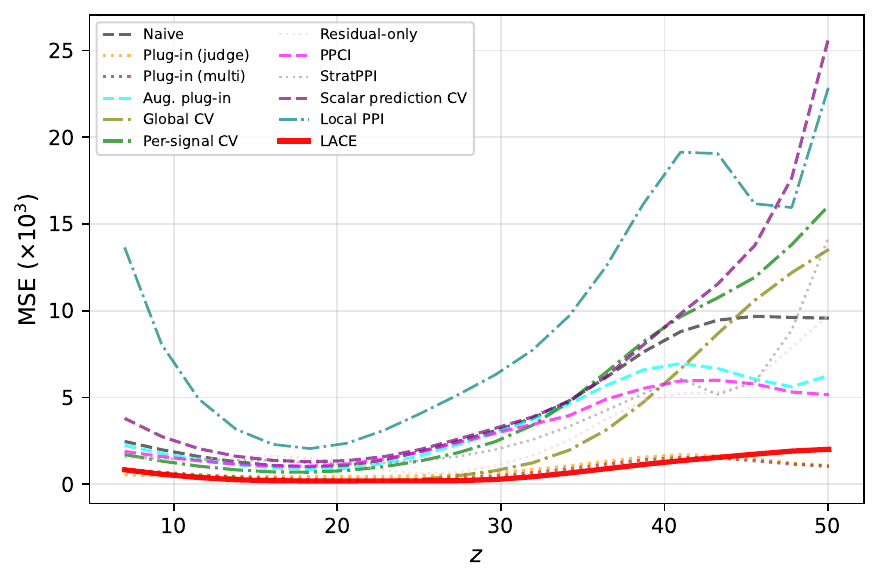}
\end{subfigure}
\hfill
\begin{subfigure}[b]{0.27\textwidth}
\centering
\includegraphics[width=\textwidth]{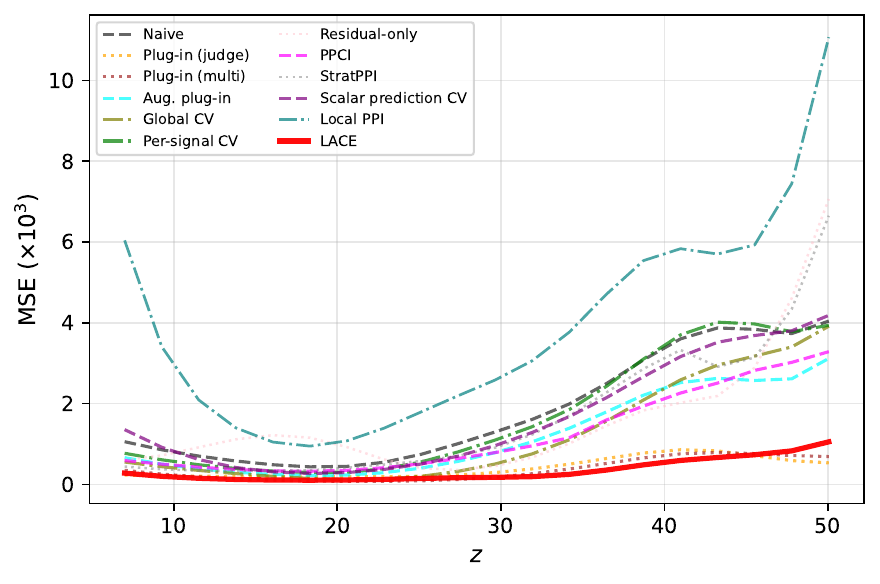}
\end{subfigure}
\hfill
\begin{subfigure}[b]{0.27\textwidth}
\centering
\includegraphics[width=\textwidth]{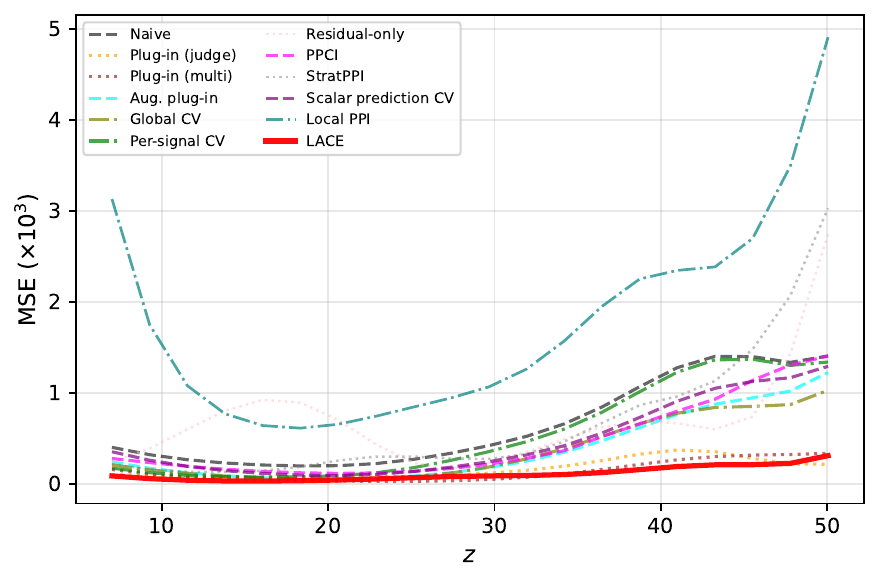}
\end{subfigure}
\caption{Pointwise MSE ($\times 10^3$) for Ministral 3B. Each row: one dataset (label on left), columns: $n=50, 100, 200$.}
\label{fig:mse-ministral-3b}
\end{figure}

\begin{figure}[htbp]
\centering
\begin{subfigure}[b]{0.12\textwidth}
\centering
\vspace{1.2cm}
\textbf{MATH-500}
\end{subfigure}
\hfill
\begin{subfigure}[b]{0.27\textwidth}
\centering
\includegraphics[width=\textwidth]{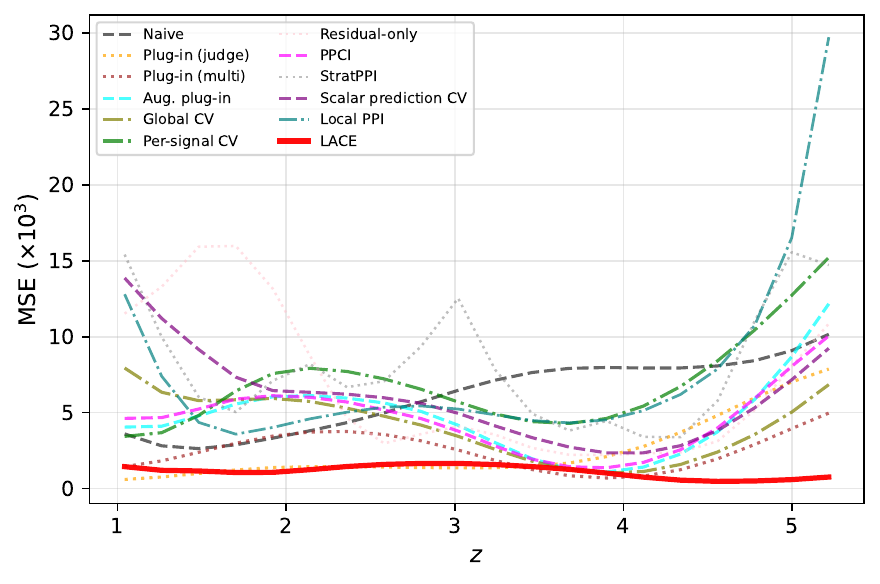}
\end{subfigure}
\hfill
\begin{subfigure}[b]{0.27\textwidth}
\centering
\includegraphics[width=\textwidth]{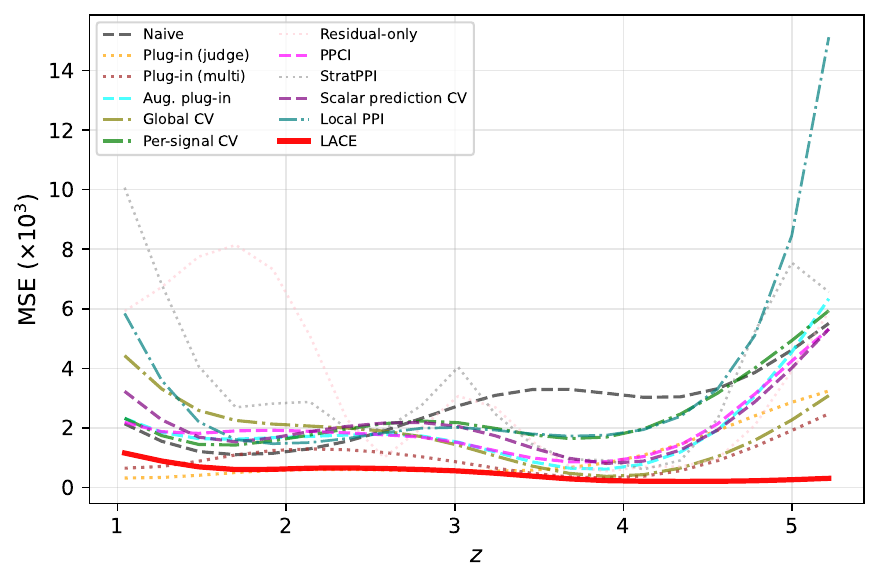}
\end{subfigure}
\hfill
\begin{subfigure}[b]{0.27\textwidth}
\centering
\includegraphics[width=\textwidth]{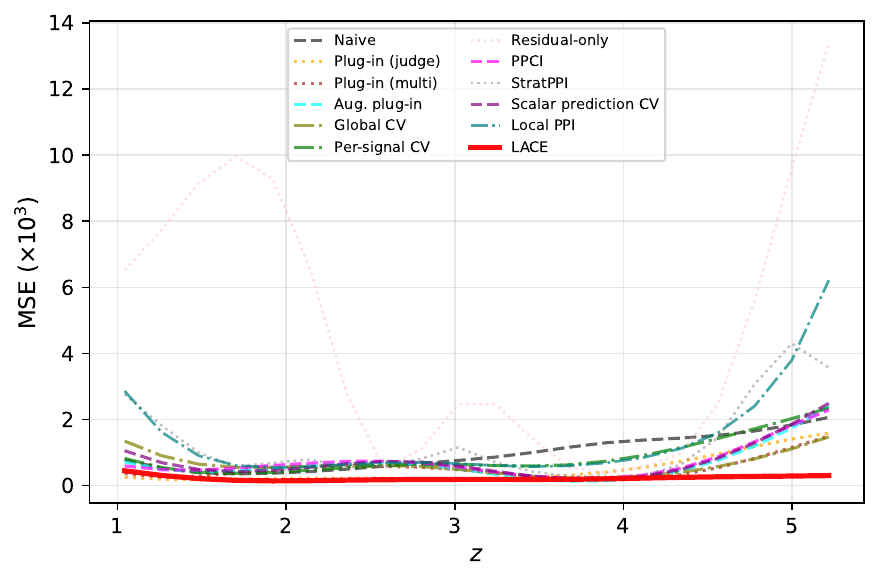}
\end{subfigure}
\par\smallskip
\begin{subfigure}[b]{0.12\textwidth}
\centering
\vspace{1.2cm}
\textbf{ScienceQA}
\end{subfigure}
\hfill
\begin{subfigure}[b]{0.27\textwidth}
\centering
\includegraphics[width=\textwidth]{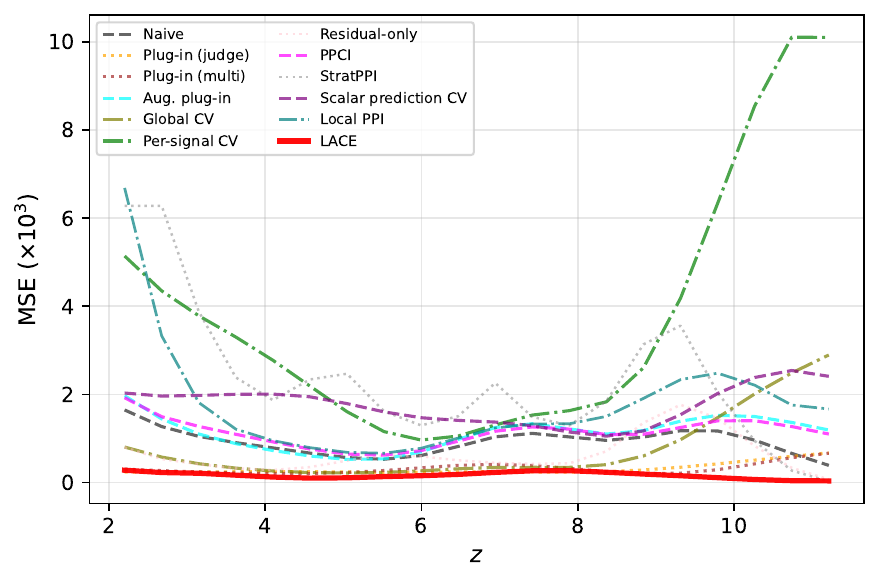}
\end{subfigure}
\hfill
\begin{subfigure}[b]{0.27\textwidth}
\centering
\includegraphics[width=\textwidth]{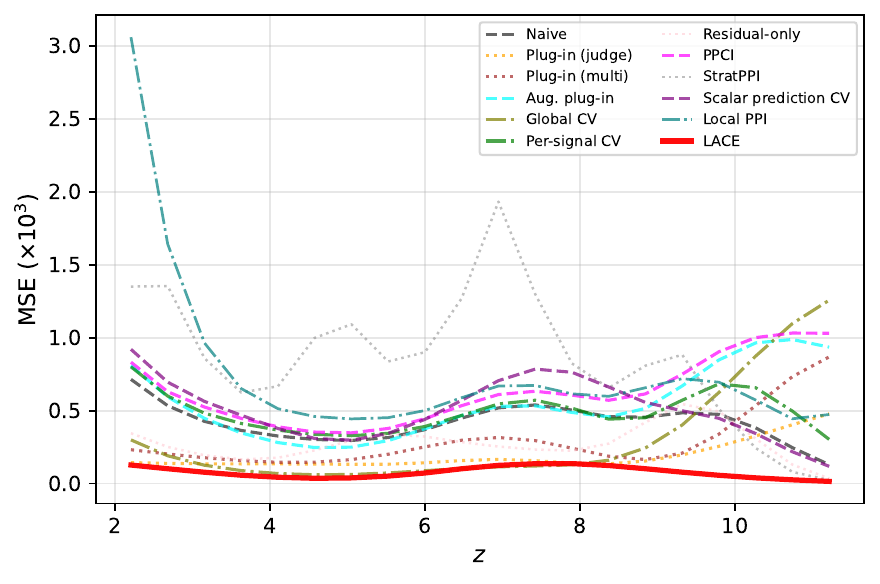}
\end{subfigure}
\hfill
\begin{subfigure}[b]{0.27\textwidth}
\centering
\includegraphics[width=\textwidth]{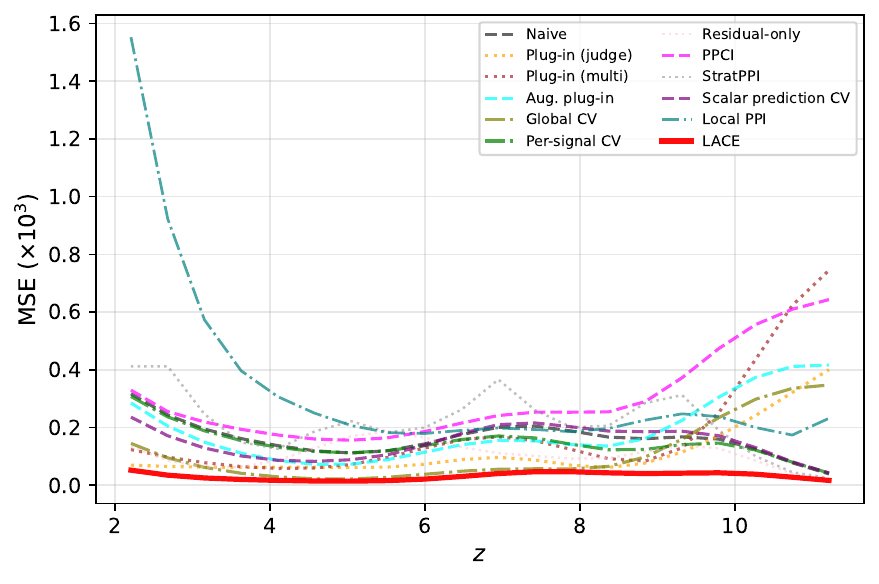}
\end{subfigure}
\par\smallskip
\begin{subfigure}[b]{0.12\textwidth}
\centering
\vspace{1.2cm}
\textbf{MMLU}
\end{subfigure}
\hfill
\begin{subfigure}[b]{0.27\textwidth}
\centering
\includegraphics[width=\textwidth]{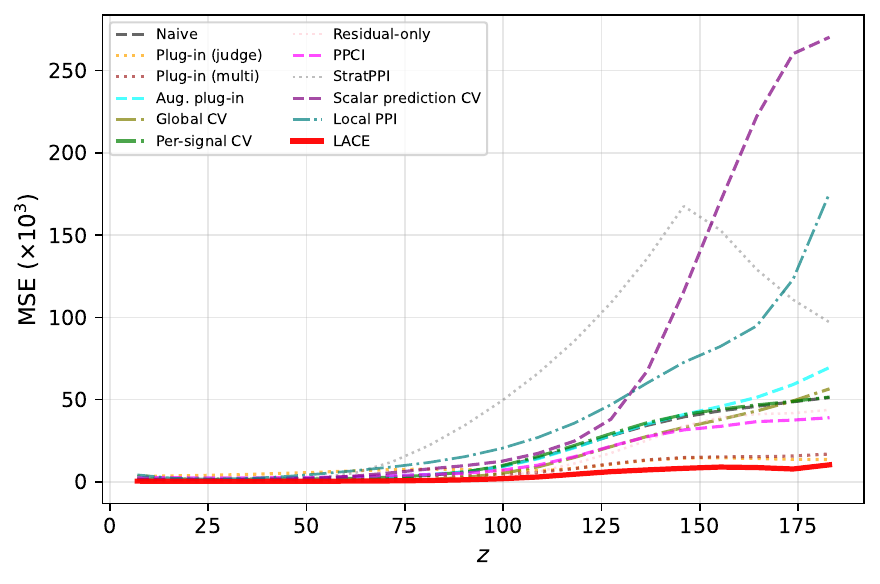}
\end{subfigure}
\hfill
\begin{subfigure}[b]{0.27\textwidth}
\centering
\includegraphics[width=\textwidth]{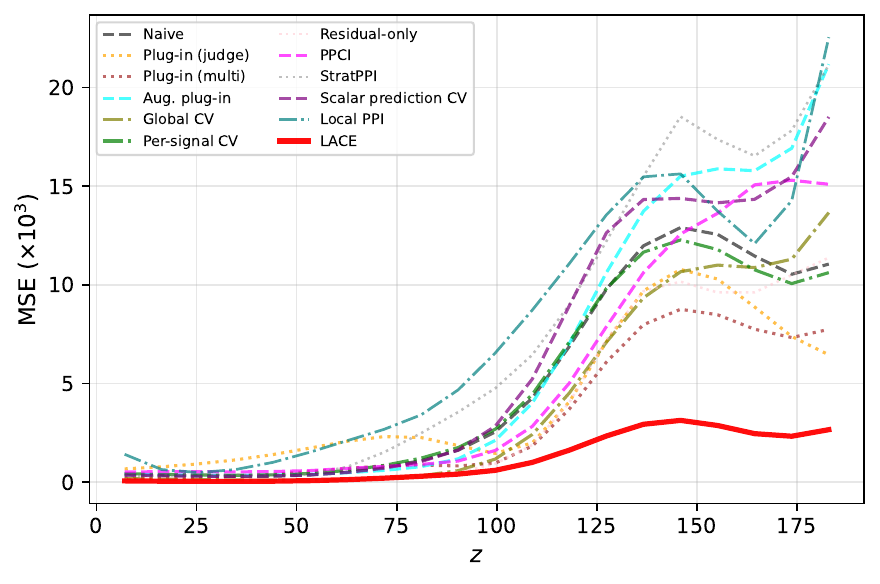}
\end{subfigure}
\hfill
\begin{subfigure}[b]{0.27\textwidth}
\centering
\includegraphics[width=\textwidth]{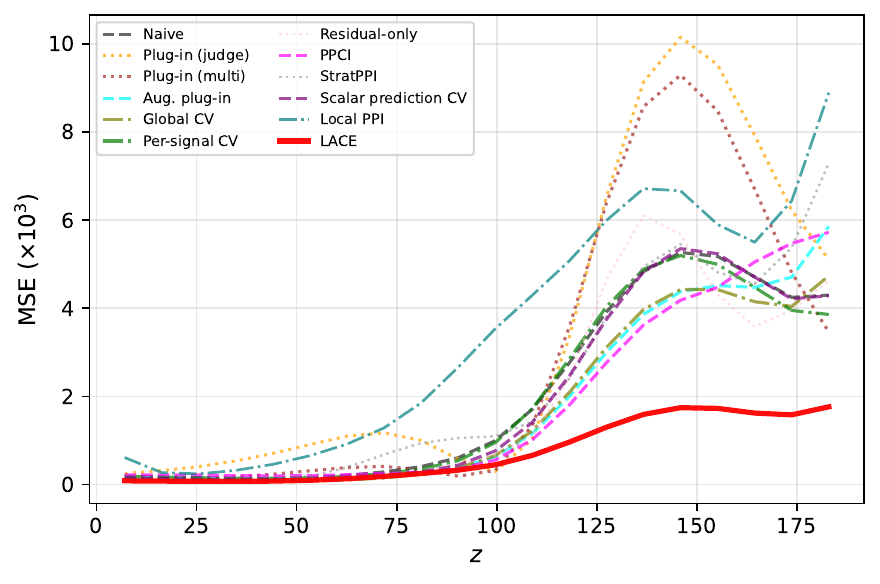}
\end{subfigure}
\par\smallskip
\begin{subfigure}[b]{0.12\textwidth}
\centering
\vspace{1.2cm}
\textbf{WinoGrande}
\end{subfigure}
\hfill
\begin{subfigure}[b]{0.27\textwidth}
\centering
\includegraphics[width=\textwidth]{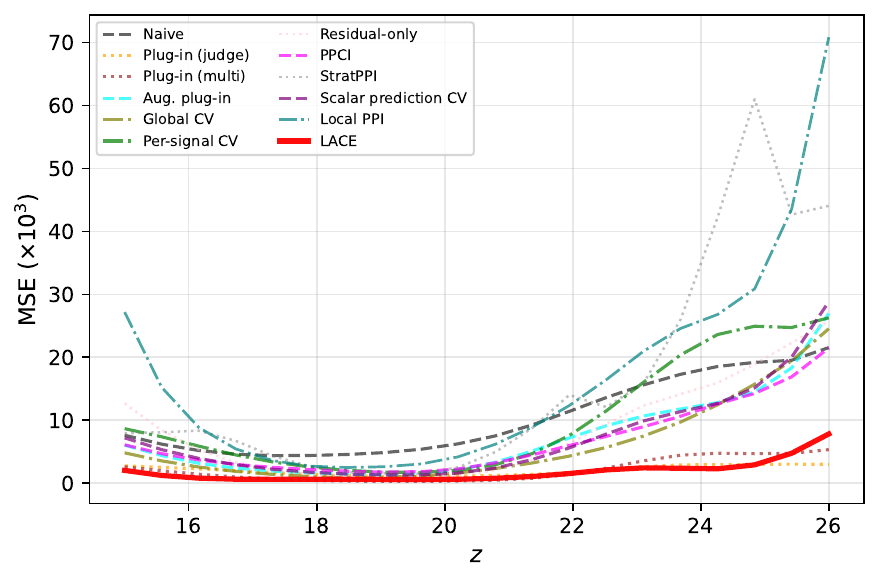}
\end{subfigure}
\hfill
\begin{subfigure}[b]{0.27\textwidth}
\centering
\includegraphics[width=\textwidth]{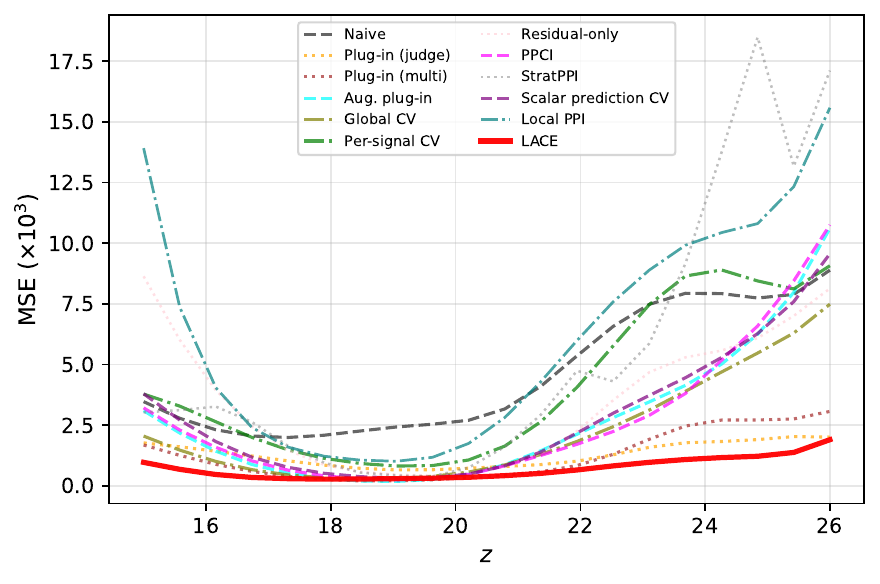}
\end{subfigure}
\hfill
\begin{subfigure}[b]{0.27\textwidth}
\centering
\includegraphics[width=\textwidth]{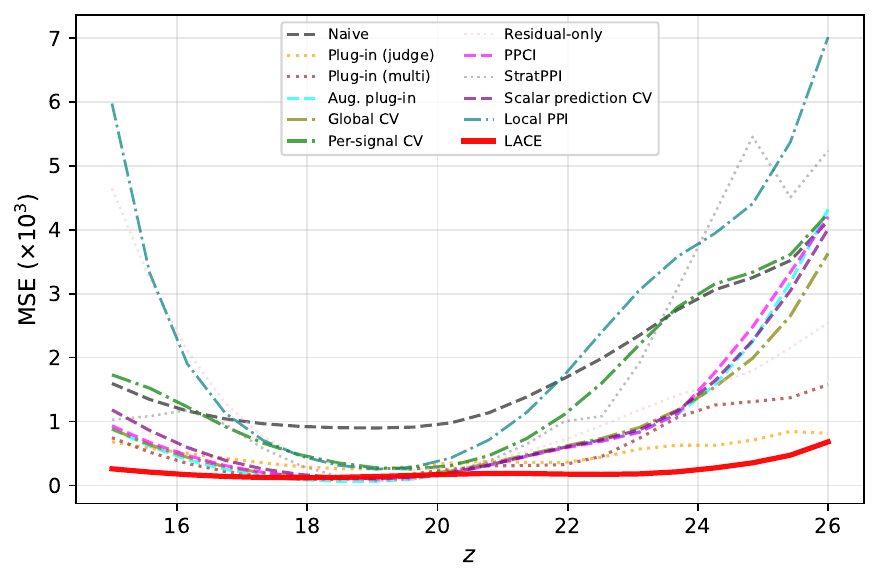}
\end{subfigure}
\par\smallskip
\begin{subfigure}[b]{0.12\textwidth}
\centering
\vspace{1.2cm}
\textbf{HellaSwag}
\end{subfigure}
\hfill
\begin{subfigure}[b]{0.27\textwidth}
\centering
\includegraphics[width=\textwidth]{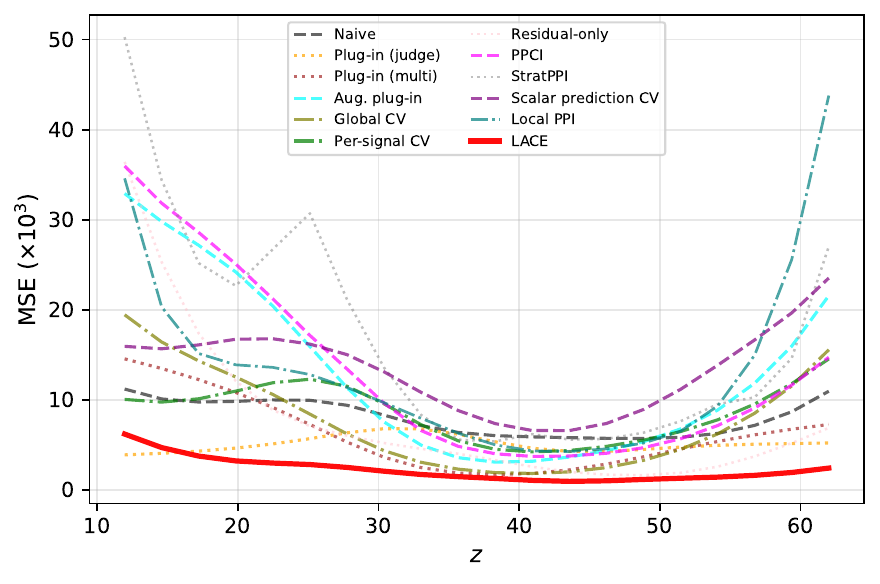}
\end{subfigure}
\hfill
\begin{subfigure}[b]{0.27\textwidth}
\centering
\includegraphics[width=\textwidth]{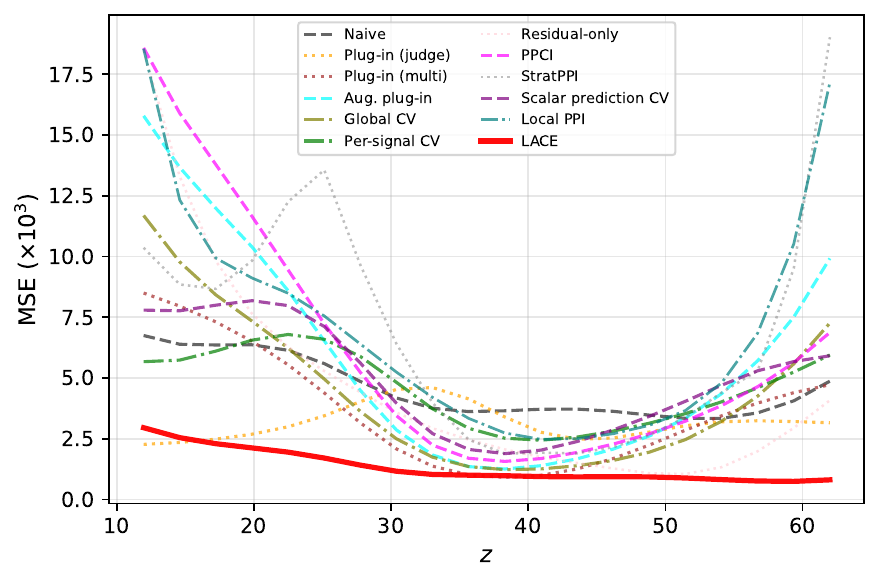}
\end{subfigure}
\hfill
\begin{subfigure}[b]{0.27\textwidth}
\centering
\includegraphics[width=\textwidth]{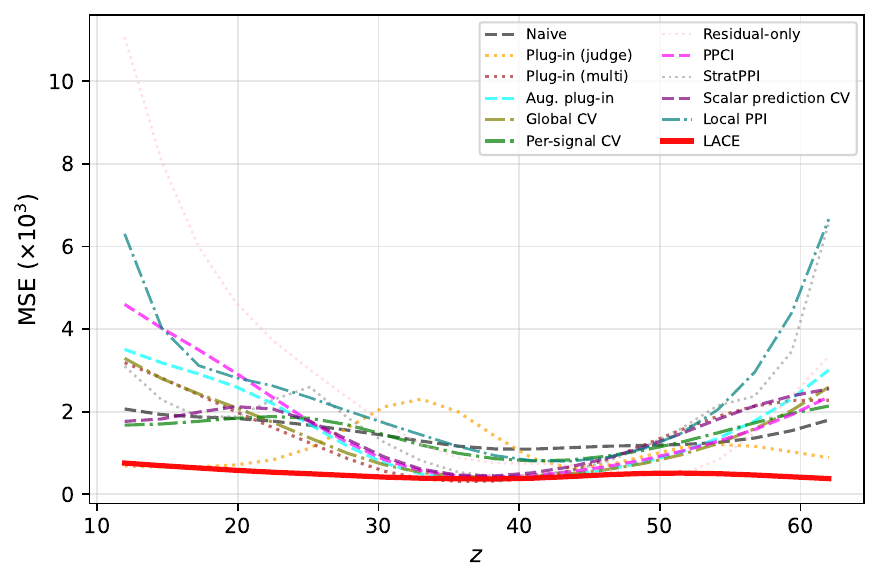}
\end{subfigure}
\par\smallskip
\begin{subfigure}[b]{0.12\textwidth}
\centering
\vspace{1.2cm}
\textbf{TruthfulQA}
\end{subfigure}
\hfill
\begin{subfigure}[b]{0.27\textwidth}
\centering
\includegraphics[width=\textwidth]{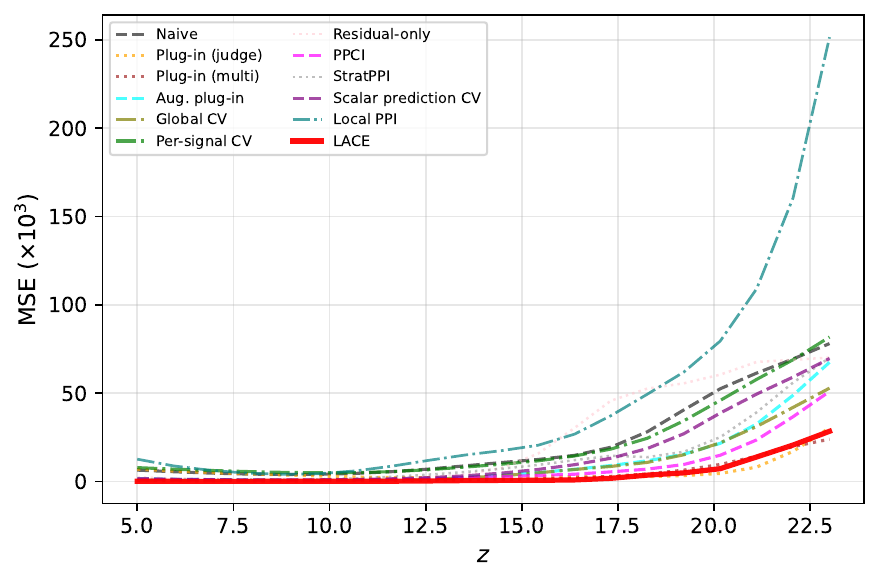}
\end{subfigure}
\hfill
\begin{subfigure}[b]{0.27\textwidth}
\centering
\includegraphics[width=\textwidth]{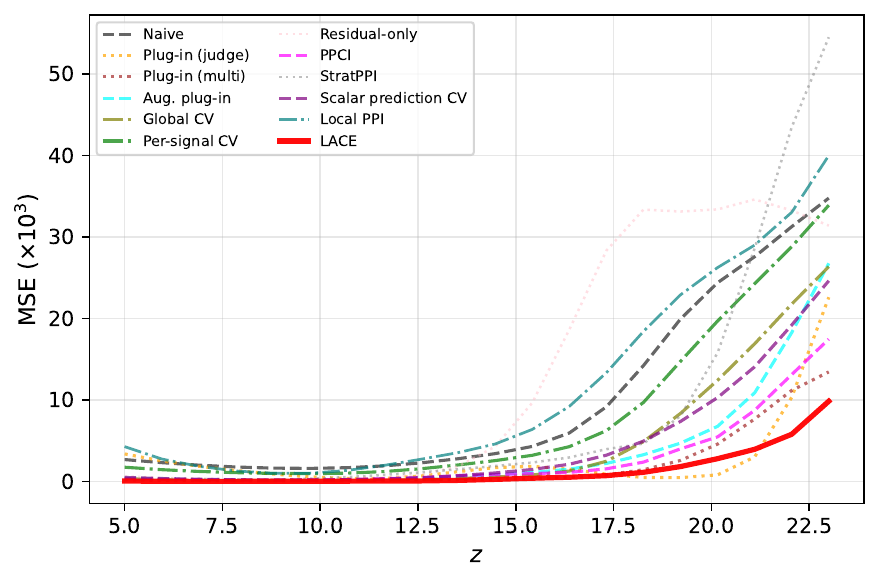}
\end{subfigure}
\hfill
\begin{subfigure}[b]{0.27\textwidth}
\centering
\includegraphics[width=\textwidth]{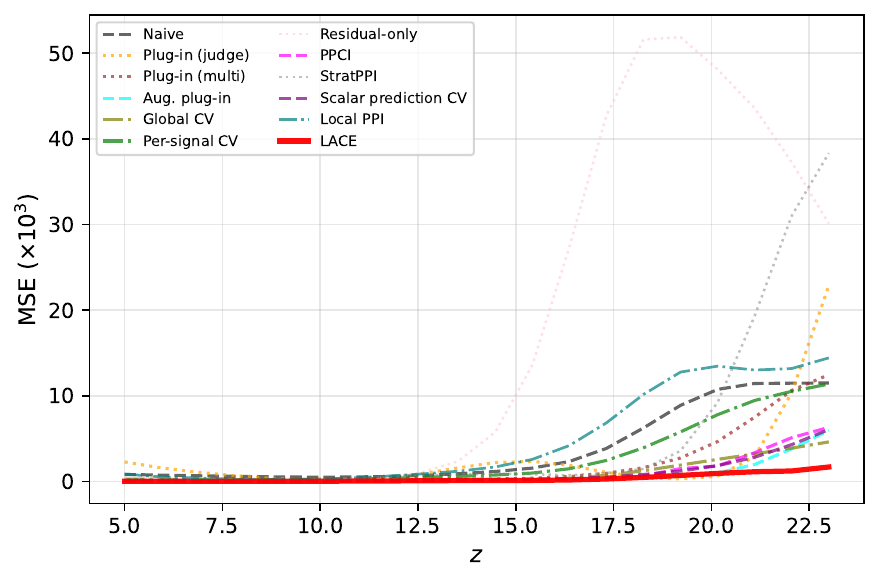}
\end{subfigure}
\par\smallskip
\begin{subfigure}[b]{0.12\textwidth}
\centering
\vspace{1.2cm}
\textbf{GSM8K}
\end{subfigure}
\hfill
\begin{subfigure}[b]{0.27\textwidth}
\centering
\includegraphics[width=\textwidth]{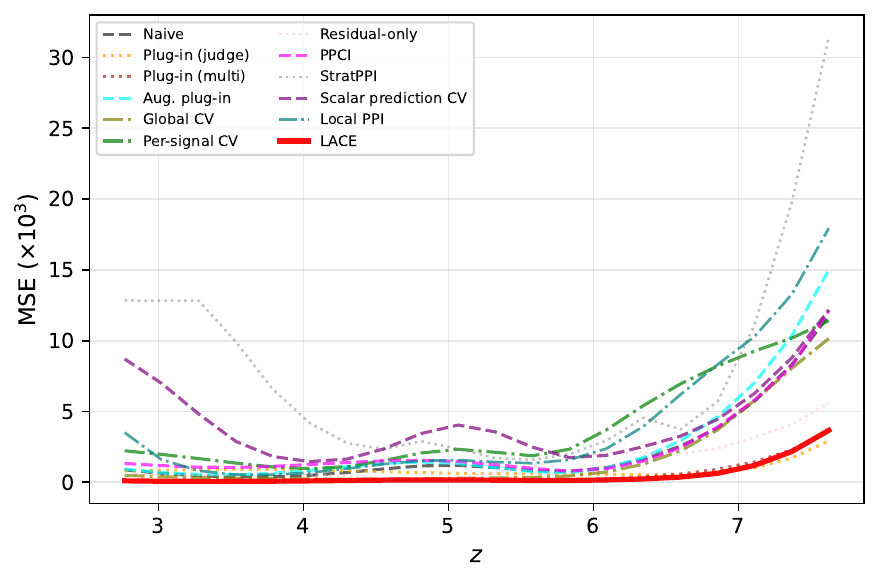}
\end{subfigure}
\hfill
\begin{subfigure}[b]{0.27\textwidth}
\centering
\includegraphics[width=\textwidth]{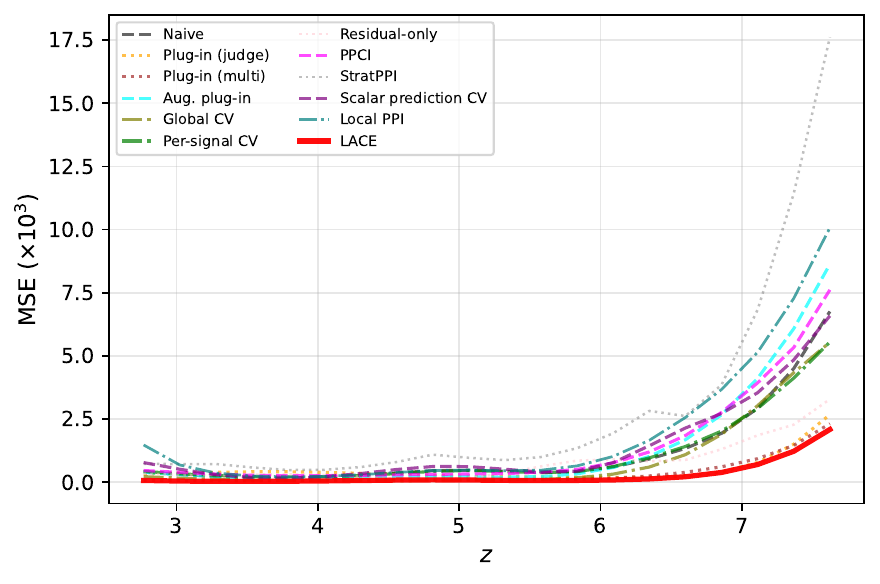}
\end{subfigure}
\hfill
\begin{subfigure}[b]{0.27\textwidth}
\centering
\includegraphics[width=\textwidth]{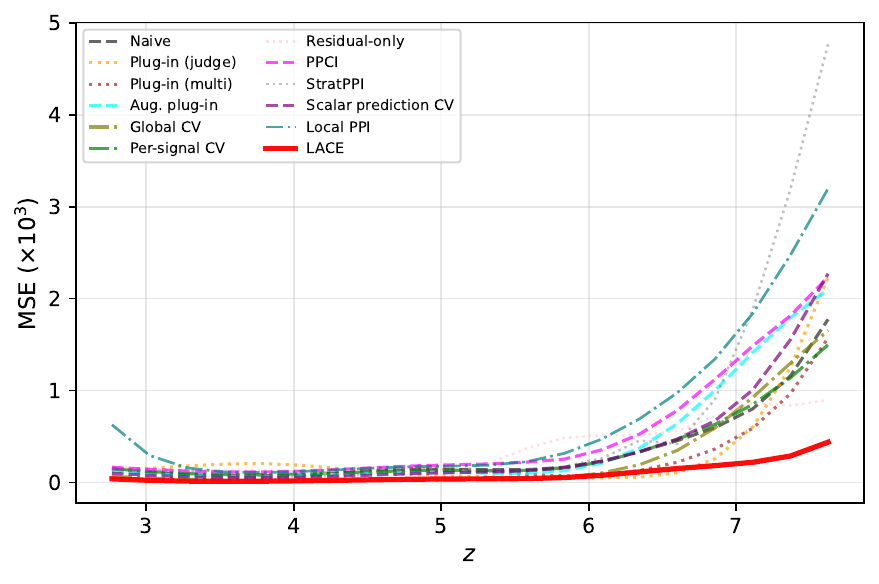}
\end{subfigure}
\par\smallskip
\begin{subfigure}[b]{0.12\textwidth}
\centering
\vspace{1.2cm}
\textbf{ARC}
\end{subfigure}
\hfill
\begin{subfigure}[b]{0.27\textwidth}
\centering
\includegraphics[width=\textwidth]{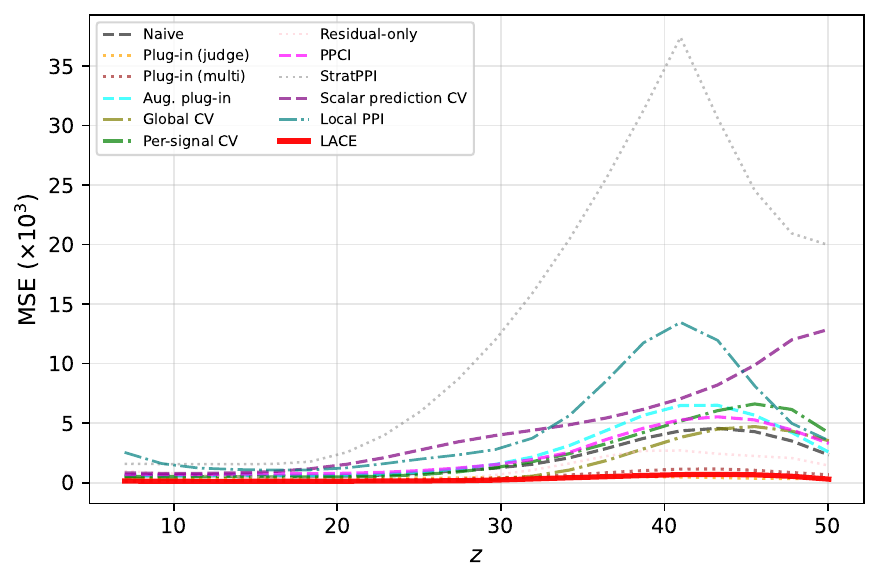}
\end{subfigure}
\hfill
\begin{subfigure}[b]{0.27\textwidth}
\centering
\includegraphics[width=\textwidth]{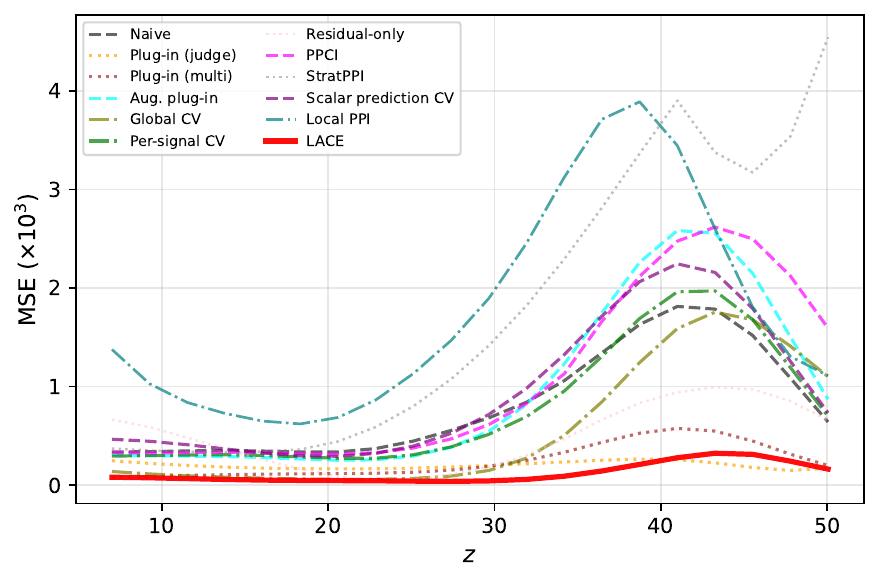}
\end{subfigure}
\hfill
\begin{subfigure}[b]{0.27\textwidth}
\centering
\includegraphics[width=\textwidth]{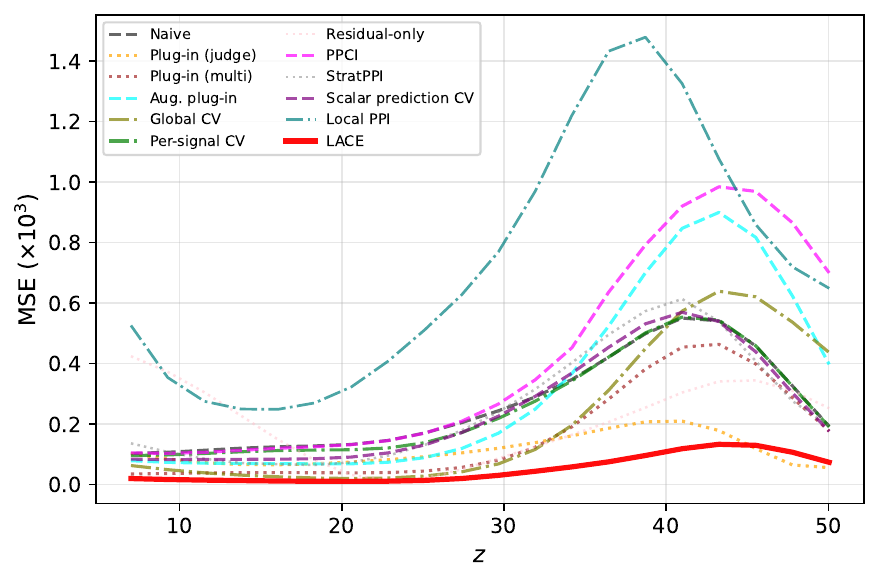}
\end{subfigure}
\caption{Pointwise MSE ($\times 10^3$) for Qwen3 32B. Each row: one dataset (label on left), columns: $n=50, 100, 200$.}
\label{fig:mse-qwen3-32b}
\end{figure}

\clearpage
\section{Additional Implementation Details}
\label{sec:implementation-details}
For continuous $Z$, the naive and \method{} estimators use the same
kernel so the comparison isolates the contribution of auxiliary signals. In
the real-data experiments, ordered discrete axes are jittered as described
below and all eight benchmarks use Gaussian smoothing. Discrete simulation
profiles use within-group means. Table~\ref{tab:ablation:ordinal} separately
reports triangular ordinal and exact within-group alternatives for the three
ordered-profile benchmarks.



The main implementation always uses $h=h_0$, $b=h$, and
$\lambda=0.3$; it neither enlarges the nuisance bandwidth nor changes
$\lambda$ adaptively. Bandwidth/ridge sensitivity and nested cross-validation
are reported as separate analyses below.

\paragraph{Continuous relaxation of discrete profiling axes.}
For benchmarks with naturally ordered discrete axes (MATH-500 difficulty level $g\in\{1,\ldots,5\}$,
ScienceQA grade $g\in\{2,\ldots,12\}$, GSM8K solution steps $g\in\{3,\ldots,10\}$), we construct
a continuous profiling variable $Z_i = g_i + U_i$ where $U_i\sim\mathrm{Uniform}(-0.4,\,0.4)$
independently across items. This preserves the ordering and group structure while enabling
Gaussian kernel smoothing with a data-driven bandwidth
$h = 1.5 \times 1.06\,\hat\sigma_Z\, M^{-1/5}$.
For benchmarks without a natural ordinal axis (MMLU, WinoGrande, HellaSwag, TruthfulQA, ARC),
we use token length (of the question, context, or sentence) as a continuous profiling variable
directly.

\subsection{\texorpdfstring{Robustness analyses}{Robustness analyses}}
\label{app:robustness}

\paragraph{Signal contributions.}
Leave-one-family-out ablations use all 24 cells and three budgets. Removing
the three pairwise-anchor signals retains only 42.1\%, 47.0\%, and 52.1\% of
full-\method{} efficiency at $n_{\mathrm{lab}}=50,100,200$, respectively,
and this family is the largest contributor in 23/24 cells at $n=100$.
Removing the judge mean, self-confidence, or judge-disagreement family
retains 74.0--89.3\% across budgets, showing that the gain is distributed
across signal families but is driven most strongly by pairwise comparisons.
Across all 432 individual-signal removals, the 5th, median, and 95th
percentiles of the percentage change in RE are $-56.41\%$, $-17.48\%$, and
$4.11\%$, respectively.

\paragraph{Ridge and bandwidth sensitivity.}
We evaluate the full $5\times5$ grid
$\lambda\in\{0.03,0.1,0.3,1,3\}$ and
$h/h_0\in\{0.5,0.75,1,1.5,2\}$ with $b=h$ on 40 matched splits;
Table~\ref{tab:sensitivity-grid} summarizes the grid. This grid is a stress
test, not the procedure used to obtain the reported results, and varying $h$
is not a pure robustness perturbation: a narrower bandwidth both changes the
smoothed target (the resolution of the estimand) and reduces the local
effective sample size. At the shared default, no cell has RE below one at any
budget. Considering every configuration, 9/24, 2/24, and 0/24 cells contain
any sub-unit RE at $n=50,100,200$; every failure occurs at a bandwidth
narrowed to $0.5h_0$ or $0.75h_0$ with $n\le100$, where weaker ridge values
amplify the deterioration. At the manuscript bandwidth $h_0$, every tested
$\lambda$ achieves RE above one in all 72 benchmark--model--budget cells
(360 configuration--cell comparisons; minimum RE 1.419), so the reported
improvement does not depend on the particular default $\lambda=0.3$; we do
not claim uniform improvement across arbitrary bandwidths or target
resolutions. For each ordered-profile cell, we define the sparse region
as the five evaluation points in the lowest quartile of full-pool kernel
effective sample size,
$n_{\mathrm{eff}}(z)=\{\sum_{i\in T}w_{i,T,h_0}(z)^2\}^{-1}$; this definition
uses no gold outcomes. Restricting MSE to these points, default median RE is
4.064, 4.787, and 4.313, with 0/9 default cells below one at each budget.

\begin{table}[htbp]
\centering
\footnotesize
\caption{Ridge--bandwidth stress grid
($\lambda\in\{0.03,0.1,0.3,1,3\}$, $h/h_0\in\{0.5,0.75,1,1.5,2\}$, 40 matched
splits). Each row summarizes, across the 24 cells, the cellwise minimum RE
over the entire 25-configuration grid. For each bandwidth, RE is
evaluated against the corresponding realized full-pool gold kernel profile,
so changing the bandwidth changes the target resolution.}
\label{tab:sensitivity-grid}
\textit{Panel A: Full-profile grid minima across 24 cells.}\\[-2pt]
\begin{tabular}{lccccc}
\toprule
Budget & Worst & 10th percentile & Median & Default below 1 & Any grid value below 1 \\
\midrule
$n{=}50$ & 0.037 & 0.403 & 1.076 & 0/24 & 9/24 \\
$n{=}100$ & 0.490 & 1.149 & 1.752 & 0/24 & 2/24 \\
$n{=}200$ & 1.394 & 1.544 & 2.043 & 0/24 & 0/24 \\
\bottomrule
\end{tabular}

\vspace{4pt}
\textit{Panel B: Sparse-region summary across nine ordered-profile cells.}\\[-2pt]
\begin{tabular}{lccc}
\toprule
Budget & Default median & Grid minimum & Default below 1 \\
\midrule
$n{=}50$ & 4.064 & 0.160 & 0/9 \\
$n{=}100$ & 4.787 & 0.615 & 0/9 \\
$n{=}200$ & 4.313 & 1.331 & 0/9 \\
\bottomrule
\end{tabular}
\end{table}

\paragraph{Within-$L$ three-fold cross-validation.}
As a separate robustness analysis of hyperparameter selection, we use
three-fold cross-validation to select each method's hyperparameter entirely
within every sampled labeled set $L$. For each benchmark--model cell, label
budget, and outer split, we partition $L$ into three approximately
outcome-balanced folds. For each candidate value and held-out fold, the method
is fitted using gold outcomes only from the other two folds; full-pool $Z$ and
auxiliary signals remain available only as unlabeled covariates. Candidates
are scored by the kernel-mass-weighted squared distance between the fitted
profile and the $h_0$-smoothed held-out gold profile, averaged across the
three folds. The selected estimator is then refit on all of $L$.
Each method receives an equal-size five-point grid for its applicable
parameter: \method{}, Global CV, Per-signal CV, and Residual-only select the
ridge penalty $\lambda\in\{0.03,0.1,0.3,1,3\}$; Plug-in (judge), Plug-in
(multi), Aug.\ plug-in, and Scalar prediction CV select the logistic
regularization $C\in\{0.03,0.1,0.3,1,3\}$; PPCI selects the multiplier
$m\in\{0.1,0.3,1,3,10\}$ applied to its L-curve-selected Tikhonov parameter,
with predictor $C=1$ fixed; StratPPI selects the number of strata
$G\in\{3,4,5,7,10\}$ with predictor $C=1$ fixed; Local PPI, which has no
internal regularization parameter, selects its local-linear bandwidth
multiplier in $\{0.5,0.75,1,1.5,2\}$ of $h_0$; and Naive varies no parameter.
The profile bandwidth remains $h_0$, \method{} uses $b=h$, and no gold
outcomes outside $L$ enter selection.
Table~\ref{tab:cv-robustness} reports the CV results; the corresponding
fixed-hyperparameter values are the rows of Table~\ref{tab:real:aggregate},
and every method stays close to its fixed-hyperparameter value. Across
$B=100$ matched outer splits, \method{} has geometric-mean RE 5.266, 5.646,
and 5.415 by budget (5.440 overall) and is best in all 72 cells. Thus the
main conclusion is unchanged when every method receives the same within-$L$
CV protocol.

\begin{table}[htbp]
\centering
\scriptsize
\setlength{\tabcolsep}{3pt}
\caption{Within-$L$ three-fold cross-validation. Entries are
geometric-mean RE $[95\%\ \mathrm{CI}]$ across the 24 benchmark--model cells;
intervals use the $B=100$ matched label splits, and the Overall column
treats the three nested budgets within each cell jointly, with the
label-split permutation as the joint resampling unit, as in
Table~\ref{tab:real:aggregate}. The fixed-hyperparameter
counterparts are the corresponding rows of
Table~\ref{tab:real:aggregate}. Target: realized full-pool gold kernel
profile; randomness: matched $L$-draws conditional on $T$.}
\label{tab:cv-robustness}
\resizebox{\textwidth}{!}{%
\begin{tabular}{lcccc}
\toprule
Method & $n{=}50$ & $n{=}100$ & $n{=}200$ & Overall \\
\midrule
Naive & $1.000\,[1.000,1.000]$ & $1.000\,[1.000,1.000]$ & $1.000\,[1.000,1.000]$ & $1.000\,[1.000,1.000]$ \\
Plug-in (judge) & $2.960\,[2.780,3.150]$ & $2.650\,[2.510,2.800]$ & $2.030\,[1.930,2.140]$ & $2.516\,[2.392,2.646]$ \\
Plug-in (multi) & $3.050\,[2.860,3.250]$ & $2.720\,[2.580,2.870]$ & $1.950\,[1.840,2.070]$ & $2.529\,[2.419,2.644]$ \\
Aug.\ plug-in & $1.030\,[0.981,1.082]$ & $1.190\,[1.136,1.247]$ & $1.460\,[1.381,1.543]$ & $1.214\,[1.174,1.256]$ \\
Global CV & $1.370\,[1.305,1.438]$ & $1.490\,[1.421,1.562]$ & $1.617\,[1.555,1.681]$ & $1.489\,[1.438,1.542]$ \\
Per-signal CV & $0.930\,[0.887,0.975]$ & $1.130\,[1.104,1.157]$ & $1.205\,[1.183,1.228]$ & $1.083\,[1.059,1.107]$ \\
Residual-only & $1.205\,[1.146,1.267]$ & $1.080\,[1.034,1.128]$ & $0.653\,[0.627,0.680]$ & $0.947\,[0.919,0.976]$ \\
Scalar prediction CV & $0.815\,[0.737,0.901]$ & $1.245\,[1.199,1.293]$ & $1.485\,[1.417,1.556]$ & $1.146\,[1.098,1.196]$ \\
Local PPI & $0.602\,[0.575,0.630]$ & $0.690\,[0.672,0.709]$ & $0.647\,[0.629,0.665]$ & $0.646\,[0.631,0.661]$ \\
PPCI & $1.205\,[1.143,1.270]$ & $1.295\,[1.238,1.355]$ & $1.380\,[1.318,1.445]$ & $1.291\,[1.252,1.331]$ \\
StratPPI & $0.735\,[0.658,0.821]$ & $0.865\,[0.788,0.949]$ & $0.920\,[0.852,0.993]$ & $0.836\,[0.777,0.900]$ \\
\method{} & $\mathbf{5.266}\,[4.953,5.599]$ & $\mathbf{5.646}\,[5.277,6.041]$ & $\mathbf{5.415}\,[5.141,5.705]$ & $\mathbf{5.440}\,[5.208,5.682]$ \\
\bottomrule
\end{tabular}%
}
\end{table}

\paragraph{Weaker judges.}
Holding candidate outcomes, $Z$, anchors, hyperparameters, and matched splits
fixed, we replace only the judge used for judge-derived signals. Median
\method{} RE with Llama-3.1-8B is 4.090, 4.150, and 3.940 by budget, with RE
at least one in 70/72 cells; GPT-OSS-20B gives 4.563, 4.970, and 4.441, with
RE at least one in 71/72 cells. Under the same protocol, the Claude Opus 4.6
judge used in the main experiments attains 5.565, 5.840, and 5.385, with RE at
least one in all 72 cells, so the weaker judges retain most of the efficiency
gain. The cost gap is large: each candidate-item uses nine judge calls, and
under provider-reported token usage (108{,}000 calls per judge, 12{,}000
candidate-items) at public hosted prices, the judge-call cost is
USD~0.00125 per item for Llama-3.1-8B (about 3.0\% of the Opus cost) and
USD~0.00057 for GPT-OSS-20B (about 1.4\%), versus USD~0.04107 for Opus 4.6;
these figures cover the nine judge calls only and exclude candidate and
anchor response generation, which is held fixed across the comparison.

\paragraph{Jointly weaker judge and anchor.}
We also replace Claude Opus 4.6, used as both the judge and the strong anchor,
with either Llama-3.1-8B or GPT-OSS-20B, while holding candidate outcomes,
$Z$, hyperparameters, and matched label splits fixed.
Table~\ref{tab:joint-weaker-signals} reports the resulting comparison.

\begin{table}[htbp]
\centering
\small
\setlength{\tabcolsep}{3pt}
\caption{Performance with jointly weaker judge and anchor models. Entries are
overall geometric-mean RE $[95\%\ \mathrm{CI}]$ across 72 cells; as in the
Overall column of Table~\ref{tab:real:aggregate}, the three nested budgets
within each cell are treated jointly, with the label-split permutation as the
joint resampling unit.
Target: realized full-pool gold kernel profile; randomness: matched
$L$-draws conditional on $T$.}
\label{tab:joint-weaker-signals}
\begin{tabular}{lccc}
\toprule
Method &
\shortstack{Opus 4.6 judge\\+ anchor} &
\shortstack{Llama-3.1-8B judge\\+ anchor} &
\shortstack{GPT-OSS-20B judge\\+ anchor} \\
\midrule
Naive & $1.000\,[1.000,1.000]$ & $1.000\,[1.000,1.000]$ & $1.000\,[1.000,1.000]$ \\
Plug-in (judge) & $2.630\,[2.270,3.047]$ & $1.574\,[1.530,1.619]$ & $1.896\,[1.843,1.950]$ \\
Plug-in (multi) & $2.647\,[2.277,3.078]$ & $1.899\,[1.835,1.966]$ & $2.281\,[2.202,2.363]$ \\
Aug.\ plug-in & $1.250\,[1.098,1.423]$ & $1.061\,[1.030,1.093]$ & $1.167\,[1.132,1.203]$ \\
Global CV & $1.566\,[1.516,1.618]$ & $1.254\,[1.223,1.286]$ & $1.423\,[1.386,1.461]$ \\
Per-signal CV & $1.054\,[1.042,1.066]$ & $0.942\,[0.926,0.958]$ & $1.003\,[0.985,1.021]$ \\
Residual-only & $0.906\,[0.833,0.985]$ & $0.716\,[0.696,0.736]$ & $0.837\,[0.813,0.862]$ \\
Scalar prediction CV & $1.122\,[1.050,1.199]$ & $0.976\,[0.940,1.014]$ & $1.061\,[1.019,1.104]$ \\
PPCI & $1.258\,[1.143,1.384]$ & $1.064\,[1.033,1.096]$ & $1.174\,[1.138,1.211]$ \\
StratPPI profile adaptation & $0.792\,[0.726,0.864]$ & $0.657\,[0.626,0.690]$ & $0.730\,[0.695,0.766]$ \\
\method{} & $\mathbf{5.528}\,[5.181,5.898]$ & $\mathbf{3.820}\,[3.672,3.974]$ & $\mathbf{4.457}\,[4.272,4.650]$ \\
\bottomrule
\end{tabular}
\end{table}

Under both replacements, \method{} remains first overall: its RE is
3.820 versus 1.899 for the next-highest method under Llama-3.1-8B, and 4.457
versus 2.281 under GPT-OSS-20B. At the cell level, \method{} remains the best
method in 70/72 cells under Llama-3.1-8B and in 71/72 cells under GPT-OSS-20B.
(The judge-only paragraph above reports median RE across cells, whereas
Table~\ref{tab:joint-weaker-signals} reports geometric means with 95\% CIs;
the two summaries are therefore not directly comparable.)

\paragraph{Exploratory pre-label stratification.}
Using only information available before labeling, we compare uniform random
labeling with proportional allocation over $Z$ deciles, judge-score deciles,
and a $5\times5$ crossed rank grid on 100 paired splits. For each cell we form
the ratio of mean stratified-sampling MSE to mean random-sampling MSE, then
take its median across the 24 cells. At $n=50$ and $n=100$, all three median
ratios exceed one (1.526--1.672 and 1.216--1.293, respectively). At $n=200$,
the ratios are 0.912, 0.912, and 0.774, with reductions exceeding 5\% in
14/24, 13/24, and 14/24 cells. Thus simple proportional stratification is
budget-dependent and can hurt when a small label budget is dispersed across
many strata; active, uncertainty-based, and coverage-based allocation remain
future work.

\paragraph{Kernel choice.}
On the 27 ordered-benchmark--candidate--budget cells, the jittered Gaussian
version has RE above one in 27/27 cells and triangular ordinal smoothing in
25/27 cells, including 9/9 at $n=200$. Table~\ref{tab:ablation:ordinal}
reports all matched-target values.


\begin{table}[h]
\centering
\small
\caption{Ordered-profile kernel ablation. Entries are LACE RE at
$n_{\mathrm{lab}}=50/100/200$ under jittered Gaussian smoothing,
triangular ordinal smoothing, and exact within-group estimation. The
jittered-Gaussian column is the main experimental protocol and reports the
same runs as Tables~3--26; the remaining two columns rerun the estimator
under the alternative locality definitions. Each
estimator is compared with the naive estimator for its own matched
full-pool target; the table therefore tests persistence of the efficiency
gain rather than ranking targets or kernels. If a compact-support method has
no labeled mass at a target level in a replication, that level is omitted
for that replication.}
\label{tab:ablation:ordinal}
\begin{tabular}{lccc}
\toprule
Benchmark / candidate & Jittered Gaussian & Triangular ordinal & Within group \\
\midrule
MATH-500 / Claude Haiku 3 & 6.14/5.78/6.07 & 4.82/4.38/4.41 & 1.75/3.73/7.27 \\
MATH-500 / Ministral 3B & 4.19/4.73/5.20 & 2.84/3.72/4.16 & 1.36/1.76/3.23 \\
MATH-500 / Qwen3 32B & 5.47/5.38/4.32 & 3.27/3.99/3.17 & 2.10/3.35/3.63 \\
ScienceQA / Claude Haiku 3 & 5.96/5.70/5.13 & 1.47/2.00/2.99 & 0.65/1.05/1.72 \\
ScienceQA / Ministral 3B & 5.72/6.56/5.67 & 1.78/2.72/3.14 & 0.61/0.70/1.59 \\
ScienceQA / Qwen3 32B & 5.71/5.90/5.12 & 1.22/2.49/2.70 & 0.34/0.52/1.45 \\
GSM8K / Claude Haiku 3 & 6.88/6.90/5.43 & 1.29/1.18/1.15 & 0.93/1.16/1.21 \\
GSM8K / Ministral 3B & 4.94/5.40/4.46 & 1.68/1.56/2.15 & 1.14/1.33/2.23 \\
GSM8K / Qwen3 32B & 6.20/5.90/3.85 & 0.93/0.79/1.81 & 0.83/0.99/1.48 \\
\bottomrule
\end{tabular}
\end{table}

\section{Simulation study}
\label{app:sim}

We report three complementary simulation analyses. First, Settings B
and C are finite-sample mechanism stress tests for ordered discrete and
unordered categorical profiles. They examine whether locally combining
multiple signals and retaining the unlabeled-pool centering matter when the
useful signal changes across groups. Second, a continuous Setting-A experiment
uses repeated pool sampling and a shrinking bandwidth to examine the local
efficiency formula of Theorem~\ref{thm:gain}, including the effect of estimating
the local coefficient from the same labeled sample. Third, a repeated-pool
study over $24$ DGP settings---eight benchmark-derived profile marginals
crossed with three candidate-indexed outcome regimes---compares practical
estimators at realistic label budgets against analytically known population
profiles.

\subsection{Discrete-profile mechanism stress tests}
\label{app:sim:setup}

\textbf{Data-generating process.}
For each item $i$ at profile value $Z_i=z$ we draw $Y_i\mid Z_i=z\sim\mathrm{Bernoulli}(\theta(z))$ and $S_{ik}=a_k(z)+\gamma_k(z)Y_i+\sigma_k(z)\xi_{ik}$ with $\xi_i\sim\mathcal N(0,I_K)$ i.i.d.\ across items, where $\theta(z)\in(0,1)$ is the prevalence and $(\gamma_k,a_k,\sigma_k)$ are the slope, offset, and residual scale of evaluator $k$.
This yields closed-form $\Sigma_{SS}(z)=\theta(z)\{1-\theta(z)\}\gamma(z)\gamma(z)^\top+\mathrm{diag}(\sigma(z))^2$, $\Sigma_{SY}(z)=\theta(z)\{1-\theta(z)\}\gamma(z)$, $\beta^\star(z)=\Sigma_{SS}(z)^{-1}\Sigma_{SY}(z)$, and $R^2(z)=\Sigma_{SY}(z)^\top\Sigma_{SS}(z)^{-1}\Sigma_{SY}(z)/\sigma_Y^2(z)$.
The two settings share $M=10{,}000$, $K=6$, $B=300$, and
$n_{\mathrm{lab}}\in\{500,\,1{,}000,\,1{,}500\}$, and differ in the profile
space and the structure of $(\gamma_k,\sigma_k,\theta,a_k)$:

\begin{itemize}[leftmargin=*,itemsep=4pt,topsep=2pt]
\item \textbf{Setting B (ordered discrete profile).}
The pool uses a balanced fixed design over $Z_i\in\{1,\ldots,5\}$, with
$M_g=2{,}000$ items in each group.
Exactly one specialist active at each level: $\gamma_k(g)=2$ for $(g,k)\in\{(1,1),(2,2),(3,3),(4,4),(5,1)\}$ and $0$ otherwise (level $5$ reuses evaluator $1$).
$\sigma=(\tfrac12,\tfrac12,\tfrac12,\tfrac12,\,1,\,1)$, $\theta=(0.3,\,0.4,\,0.5,\,0.6,\,0.7)$, and $a_k(g)\overset{\mathrm{i.i.d.}}{\sim}\mathcal N(0,1)$ frozen across replications.
Local weights are within-group means.

\item \textbf{Setting C (unordered categorical profile).}
The pool uses a balanced fixed design over $Z_i\in\{1,\ldots,10\}$, with
$M_g=1{,}000$ items in each group.
For each $g$, $r_g\in\{1,2\}$ specialist indices are drawn uniformly without replacement, with $\gamma_k(g)=2$ and $\sigma_k(g)=\tfrac12$ at specialists and $\gamma_k(g)=0$, $\sigma_k(g)=1$ elsewhere.
$\theta(g)\overset{\mathrm{i.i.d.}}{\sim}\mathrm{Uniform}(0.25,\,0.75)$ and $a_k(g)\overset{\mathrm{i.i.d.}}{\sim}\mathcal N(0,1)$, all frozen across replications.
Local weights are within-group means with the empty-category fallback $\widehat\theta(g)=\bar Y_L$.
\end{itemize}

The designs are frozen before the Monte Carlo experiment: group
counts, $\theta(g)$, offsets, and specialist assignments remain fixed across
replications, while $Y$, $S$, and the labeled subset $L$ are redrawn. The
reported risks are therefore conditional on these prespecified designs.

\textbf{Estimators.}
We compare feasible \method{} and its oracle-coefficient counterpart with the
gold-only naive estimator, judge-only and multivariate plug-in estimators,
augmented plug-in, Global CV, Per-signal CV, Residual-only, and Scalar
prediction CV. The estimator definitions match
Section~\ref{estimators_compared}. Feasible \method{}, Global CV, Per-signal
CV, and Residual-only use the same prespecified ridge $\lambda=0.3$. The
augmented plug-in and Scalar prediction CV use the same two-fold cross-fitted
labeled-set logistic predictor. PPCI is evaluated in the continuous-profile
repeated-pool comparison below; it is not included here because the unordered
categories in Setting C have no canonical profile distance for its RKHS.

We report
$\mathrm{RE}=\overline{\mathrm{MSE}}^{\mathrm{naive}}/
\overline{\mathrm{MSE}}$ over $B=300$ replications, with MSE unweighted across
groups against $\theta(g)$. Larger RE means smaller mean MSE. We compute RE as
a ratio of Monte Carlo mean MSEs, avoiding the Jensen-type bias that would
arise from averaging per-replication ratios.

\subsection{Results}
\label{app:sim:efficiency}

Tables~\ref{tab:sim:re:B} and~\ref{tab:sim:re:C} report the two
conditional-on-design stress tests.

\begin{table}[h]
\centering
\small
\caption{Setting B (ordered discrete): conditional-on-design relative
efficiency over $B=300$ replications, with
$\overline{\mathrm{MSE}}\times 10^3$ in parentheses. Ridge-regularized
estimators use the prespecified $\lambda=0.3$; bold marks the strongest
feasible estimator. MSE is evaluated against the fixed conditional
design profile $\theta(g)$.}
\label{tab:sim:re:B}
\begin{tabular}{lccc}
\toprule
Estimator & $n_{\mathrm{lab}}{=}500$ & $n_{\mathrm{lab}}{=}1{,}000$ & $n_{\mathrm{lab}}{=}1{,}500$ \\
\midrule
Naive & $1.00\times$ ($2.09$) & $1.00\times$ ($1.20$) & $1.00\times$ ($0.79$) \\
Plug-in (judge-only) & $0.17\times$ ($12.17$) & $0.10\times$ ($11.74$) & $0.07\times$ ($11.74$) \\
Plug-in (multivariate) & $0.14\times$ ($14.92$) & $0.08\times$ ($14.62$) & $0.05\times$ ($14.62$) \\
Augmented plug-in & $1.33\times$ ($1.57$) & $1.43\times$ ($0.84$) & $1.42\times$ ($0.56$) \\
Global CV & $1.32\times$ ($1.58$) & $1.34\times$ ($0.89$) & $1.31\times$ ($0.61$) \\
Per-signal CV ($S_1$) & $1.29\times$ ($1.62$) & $1.36\times$ ($0.88$) & $1.27\times$ ($0.62$) \\
Residual-only ablation & $0.01\times$ ($286.08$) & $<0.01\times$ ($282.04$) & $<0.01\times$ ($284.26$) \\
Scalar prediction CV & $1.27\times$ ($1.65$) & $1.43\times$ ($0.84$) & $1.37\times$ ($0.58$) \\
\method{} (oracle $\beta^\star$) & $3.73\times$ ($0.56$) & $3.75\times$ ($0.32$) & $2.97\times$ ($0.27$) \\
\method{} (feasible) & $\mathbf{3.04}\times$ ($0.69$) & $\mathbf{3.22}\times$ ($0.37$) & $\mathbf{2.64}\times$ ($0.30$) \\
\bottomrule
\end{tabular}
\end{table}

\begin{table}[h]
\centering
\small
\caption{Setting C (unordered categorical): conditional-on-design relative
efficiency over $B=300$ replications, with
$\overline{\mathrm{MSE}}\times 10^3$ in parentheses. Ridge-regularized
estimators use the prespecified $\lambda=0.3$; bold marks the strongest
feasible estimator. MSE is evaluated against the fixed conditional
design profile $\theta(g)$.}
\label{tab:sim:re:C}
\begin{tabular}{lccc}
\toprule
Estimator & $n_{\mathrm{lab}}{=}500$ & $n_{\mathrm{lab}}{=}1{,}000$ & $n_{\mathrm{lab}}{=}1{,}500$ \\
\midrule
Naive & $1.00\times$ ($4.77$) & $1.00\times$ ($2.17$) & $1.00\times$ ($1.51$) \\
Plug-in (judge-only) & $0.16\times$ ($28.93$) & $0.08\times$ ($28.16$) & $0.05\times$ ($28.24$) \\
Plug-in (multivariate) & $0.24\times$ ($19.55$) & $0.12\times$ ($18.51$) & $0.08\times$ ($18.48$) \\
Augmented plug-in & $1.28\times$ ($3.72$) & $1.25\times$ ($1.73$) & $1.28\times$ ($1.18$) \\
Global CV & $1.31\times$ ($3.65$) & $1.25\times$ ($1.74$) & $1.26\times$ ($1.19$) \\
Per-signal CV ($S_1$) & $1.43\times$ ($3.33$) & $1.40\times$ ($1.55$) & $1.36\times$ ($1.11$) \\
Residual-only ablation & $0.03\times$ ($163.98$) & $0.01\times$ ($161.90$) & $0.01\times$ ($162.33$) \\
Scalar prediction CV & $1.29\times$ ($3.69$) & $1.22\times$ ($1.77$) & $1.30\times$ ($1.16$) \\
\method{} (oracle $\beta^\star$) & $5.00\times$ ($0.95$) & $3.77\times$ ($0.57$) & $3.35\times$ ($0.45$) \\
\method{} (feasible) & $\mathbf{3.22}\times$ ($1.48$) & $\mathbf{3.09}\times$ ($0.70$) & $\mathbf{2.96}\times$ ($0.51$) \\
\bottomrule
\end{tabular}
\end{table}

Feasible \method{} has the highest RE among feasible methods at every
budget in both discrete-profile settings. The larger oracle--feasible gap in
Setting C at $n_{\mathrm{lab}}=500$ is consistent with the harder local
coefficient problem: ten categories leave only about $50$ labeled observations
per category for estimating a six-dimensional coefficient.

The remaining rows isolate the intended mechanisms. A single Global CV
cannot follow a useful signal direction that changes across groups, while
Per-signal CV discards the complementary specialists. The judge-only and
multivariate plug-ins use globally parameterized logistic predictors and are
misspecified under these group-varying signal slopes, falling below the naive
estimator at all budgets. Finally, removing the unlabeled-pool centering in
Residual-only causes a much larger error, illustrating the role of the
centering term in Proposition~\ref{prop:identity}.

\subsection{Gain formula in its asymptotic bandwidth regime}
\label{app:sim:thm4}

\paragraph{Continuous specialist DGP.}
For each replication, $Z_i\overset{\mathrm{i.i.d.}}{\sim}
\mathrm{Uniform}[0,1]$ and $Y_i\mid Z_i=z\sim\mathrm{Bernoulli}(1/2)$. Five
specialist signals have triangular slopes
$\gamma_k(z)=2(1-5|z-\mu_k|)_+$ at
$\mu_k=(2k-1)/10$, $k=1,\ldots,5$, and the sixth is pure noise with
$\gamma_6\equiv0$. We set
$\sigma=(1/2,1/2,1/2,1/2,1/2,1)$ and
$a_k(z)=c_k\sin(\pi z)$ for $c=(1,-1,1,-1,1,0)$, and evaluate on
$\{0.10,0.15,\ldots,0.90\}$. The induced conditional first and second moments
are continuous at every evaluation point despite the triangular kinks, and
$B_h(z)=0$ because $\theta_0(z)\equiv1/2$.

The fixed-bandwidth Setting-A control uses $h=0.12$, whereas

Theorem~\ref{thm:gain} is a pointwise, shrinking-bandwidth statement. We
therefore run a dedicated continuous Setting-A simulation in which the entire
pool is independently redrawn in every replication (500--2{,}400 replications
per pool size). We inject the population coefficient $\beta^\star(z)$ to
isolate the variance formula from coefficient-estimation error, hold the label
fractions fixed at $\pi\in\{0.05,0.10,0.15\}$, and grow the pool to
$M=320{,}000$ with the shrinking bandwidth $h=M^{-1/3}$, so that $h\to0$ and
$nh\to\infty$ as required by the theorem; the fixed bandwidth
$h=0.12$ is retained as a diagnostic control. In this DGP $\theta_0(z)=0.5$,
$f_Z(z)=1$, and the evaluation grid is equally weighted, so the naive leading
variance is constant across the grid and the profile-level theoretical RE is
the correctly aggregated pointwise prediction rather than an unweighted
approximation to a heterogeneous variance profile.

We also evaluate the feasible estimator from
Equation~\eqref{eq:lace-feasible} in the same paired replications. It uses
coefficient bandwidth $b=h$ and the prespecified vanishing ridge sequence
$\lambda=1/(nh)$, which satisfies the sufficient conditions in
Assumption~\ref{ass:nuisance}; the same labeled observations estimate
$\widehat\beta(z)$ and the final profile, and no Monte Carlo outcome is used
for tuning.

\begin{table}[htbp]
\centering
\footnotesize
\caption{Theorem~\ref{thm:gain} in its asymptotic regime, at
$M=320{,}000$ ($h=0.0146$). Entries are profile-level RE with paired Monte
Carlo $95\%$ CIs. MSE is evaluated against the population profile
$\theta_0(z)=1/2$. The feasible column uses $b=h$ and
$\lambda=1/(nh)$.}
\label{tab:sim:thm4}
\resizebox{\textwidth}{!}{%
\begin{tabular}{ccccc}
\toprule
\shortstack{Label\\fraction $\pi$} &
\shortstack{Theorem~\ref{thm:gain}\\prediction} &
\shortstack{Empirical RE (oracle)\\$h=M^{-1/3}$ ($95\%$ CI)} &
\shortstack{Empirical RE (feasible)\\$h=M^{-1/3}$ ($95\%$ CI)} &
\shortstack{Empirical RE (oracle)\\fixed $h$ ($95\%$ CI)} \\
\midrule
0.05 & 3.245 & 3.216 $(3.101, 3.332)$ & 3.159 $(3.045, 3.272)$ & 2.332 $(2.194, 2.470)$ \\
0.10 & 2.902 & 2.916 $(2.818, 3.014)$ & 2.904 $(2.807, 3.002)$ & 2.197 $(2.070, 2.324)$ \\
0.15 & 2.625 & 2.610 $(2.522, 2.698)$ & 2.601 $(2.512, 2.689)$ & 1.994 $(1.878, 2.109)$ \\
\bottomrule
\end{tabular}
}
\end{table}

All three predictions fall inside the shrinking-bandwidth profile-level
intervals. Because Theorem~\ref{thm:gain} is pointwise, we also compare the
predicted and empirical gain at all $17$ evaluation points: the oracle
pointwise intervals contain the prediction at $15/17$, $17/17$, and $16/17$
points for $\pi=0.05,0.10,0.15$, respectively.
The feasible profile-level intervals also contain all three
predictions, and its pointwise intervals contain the corresponding prediction
at $15/17$, $16/17$, and $16/17$ points. At $M=320{,}000$, its RE is within
$1.80\%$, $0.39\%$, and $0.36\%$ of the oracle RE. Across the reported pool
sizes, the mean squared coefficient error decreases monotonically from
$0.00922$ to $0.000843$, from $0.00611$ to $0.000435$, and from $0.00475$ to
$0.000296$ for $\pi=0.05,0.10,0.15$, respectively.
The fixed-bandwidth control remains substantially below the prediction
even as $M$ increases (at
$\pi=0.05$ its RE changes only from $2.282$ to $2.332$ as $M$ grows from
$5{,}000$ to $320{,}000$, agreeing with the corresponding fixed-bandwidth
oracle result of about $2.33\times$). This isolates the gap between the
fixed-bandwidth result and the pointwise prediction as a bandwidth-regime
effect: holding $h=0.12$ fixed
mixes nearby regions whose signal coefficients vary with $z$, whereas
Theorem~\ref{thm:gain} is a shrinking-bandwidth, pointwise result.
Figure~\ref{fig:thm4-pointwise} shows the pointwise comparison directly: at
$M=320{,}000$ the oracle and feasible curves track the predicted gain
$1/\{1-(1-\pi)R^2(z)\}$ across the evaluation grid, reproducing its peaks at
the specialist centers and its troughs between them, whereas the
fixed-bandwidth diagnostic control sits well below the prediction, consistent
with the bandwidth-regime explanation above.

\begin{figure}[htbp]
\centering
\includegraphics[width=\textwidth]{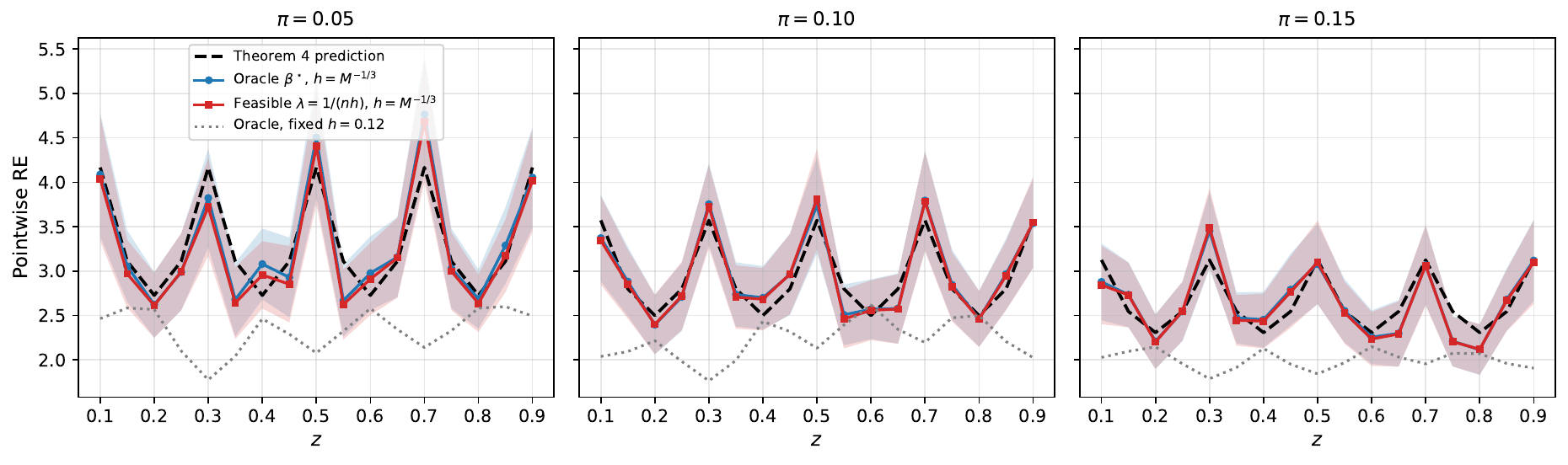}
\caption{Pointwise gain at $M=320{,}000$ for
$\pi\in\{0.05,0.10,0.15\}$. Each panel plots, across the $17$ evaluation
points, the Theorem~\ref{thm:gain} prediction $1/\{1-(1-\pi)R^2(z)\}$
(dashed, computed from the DGP's analytically known population
$R^2(z)$), the oracle-coefficient empirical RE under the shrinking bandwidth
$h=M^{-1/3}$ (circles), and the feasible estimator with vanishing ridge
$\lambda=1/(nh)$ (squares); shaded regions are paired Monte Carlo $95\%$
intervals. The fixed-bandwidth control $h=0.12$ (dotted) stays below the
prediction throughout. The oracle intervals contain the prediction at
$15/17$, $17/17$, and $16/17$ points and the feasible intervals at $15/17$,
$16/17$, and $16/17$ points.}
\label{fig:thm4-pointwise}
\end{figure}

\subsection{Superpopulation risk under repeated pool resampling}
\label{app:sim:superpop}

The fixed-pool profile MSE in the main experiments is not, by itself, an
empirical validation of the superpopulation gain theorem. We therefore add a
complementary finite-sample experiment that uses the theorem's two-stage
sampling framework and population target at realistic label budgets. The
reported quantity remains integrated finite-sample MSE, rather than the
pointwise asymptotic variance in Theorem~\ref{thm:gain}; direct validation of
that variance formula is provided by the shrinking-bandwidth experiment above.

\paragraph{The 24 prespecified data-generating settings.}
The settings form the Cartesian product
\[
24=8\ \text{benchmark-derived profile marginals}\times
3\ \text{candidate-indexed outcome regimes}.
\]
Let $d\in\{1,\ldots,8\}$ index the benchmark profile marginal and
$m\in\{1,2,3\}$ index the candidate outcome regime, so a setting is
$c=(d,m)$. The benchmark index determines $P_{Z,d}$, while the candidate
index selects a distinct conditional outcome law $p_m$. Thus the candidate
index changes the probability law rather than merely the random seed. No real
candidate outcome or judge output is reused. Concretely, benchmark $d$ takes
the $N_d$ observed source profile values
$z^{\mathrm{src}}_{d1},\ldots,z^{\mathrm{src}}_{dN_d}$ (with $N_d=500$
throughout) and converts each to
its empirical midrank in $[0,1]$: item $i$ receives
\[
z_{di}=
\frac{\#\{j:z^{\mathrm{src}}_{dj}<z^{\mathrm{src}}_{di}\}
 +\tfrac12\#\{j:z^{\mathrm{src}}_{dj}=z^{\mathrm{src}}_{di}\}}{N_d},
\]
that is, the fraction of items $j$ in benchmark $d$ ranked strictly below item
$i$, counting ties as one half. Let $P_{Z,d}$ be uniform on
$\{z_{di}\}_{i=1}^{N_d}$ and set $P_{Z,c}=P_{Z,d}$ for $c=(d,m)$.
The eight benchmark indices therefore differ through their
prespecified profile marginals and tie structures. Within each benchmark,
the candidate index selects
\[
p_m(z)=\mu_m+0.15\cos(2\pi z),\qquad
(\mu_1,\mu_2,\mu_3)=(0.35,0.50,0.65),
\]
\[
Y\mid Z=z,c=(d,m)\sim\operatorname{Bernoulli}\{p_m(z)\}.
\]
The three offsets represent lower-, middle-, and higher-accuracy
outcome regimes and ensure $p_m(z)\in[0.20,0.80]$. All settings retain the
same locally heterogeneous signal mechanism below, but their joint laws
differ through $P_{Z,d}$ and $p_m$.
With $X=2Y-1$, the six raw auxiliary signals are
\[
S^{\mathrm{raw}}=a(z)+\gamma(z)X+D\varepsilon,\qquad
\varepsilon\sim\mathcal N(0,I_6),
\]
\[
a(z)=0.75\sin(2\pi z)(1,-1,1,-1,1,0)^\top,
\]
\[
\begin{aligned}
\gamma(z)=\bigl(&0.5\cos 2\pi z,\;2\sin 2\pi z,\;2\cos 4\pi z,\\
                &2\sin 4\pi z,\;2\cos 6\pi z,\;0\bigr)^\top .
\end{aligned}
\]
\[
D=\operatorname{diag}(0.25,0.25,0.25,0.25,0.25,0.50).
\]
The first coordinate is the prespecified primary signal, coordinates $2$--$5$
are complementary smooth specialists whose useful direction changes with
$z$, and coordinate $6$ is pure noise. Each coordinate is centered and scaled
using its fixed population mean and standard deviation under DGP
setting $c=(d,m)$, before Monte Carlo sampling. All settings have a
one-dimensional ordered profile; ordered-discrete and categorical kernel choices are assessed
separately in Table~\ref{tab:ablation:ordinal}.

The sampling law, evaluation target, and loss are as follows.
For every DGP setting $c=(d,m)$, budget $n\in\{50,100,200\}$, and replication
$b=1,\ldots,B$ with $B=500$, we set $M_n=10n$, draw a fresh pool
$T_{cnb}=\{(Z_i,Y_i,S_i)\}_{i=1}^{M_n}$ i.i.d.\ from setting $c$, and sample
$L_{cnb}\subset T_{cnb}$ uniformly without replacement with $|L_{cnb}|=n$.
Every method receives the same full-pool $(Z,S)$ and gold $Y$ only on the same
$L_{cnb}$. Because the pool size $M_n=10n$ differs across budgets, the three
budgets use independent pools of sizes $500$, $1{,}000$, and $2{,}000$; unlike
the main experiments, the labeled sets are not nested across budgets.

The population target is analytically known:
$\theta_{0,c}(z)=\mathbb E_c(Y\mid Z=z)=p_m(z)$. Let $Q_d$ be the quantile
function of $P_{Z,d}$. We evaluate on the $G=20$ support points
$z_{dg}=Q_d\{0.05+0.90(g-1)/(G-1)\}$ with uniform grid weight, using the
Gaussian kernel and
\[
h_{d,n}=1.5(1.06)\operatorname{sd}_d(Z)M_n^{-1/5},\qquad b_{d,n}=h_{d,n}.
\]
For method $r$, the replication-level error is
\[
e_{cnb,r}=\frac1G\sum_{g=1}^{G}
\{\widehat\theta_{cnb,r}(z_{dg})-\theta_{0,c}(z_{dg})\}^2.
\]
The resulting risk includes benchmark-pool sampling, label subsampling, and
finite-bandwidth approximation error.

All estimators follow the definitions in
Section~\ref{estimators_compared}, with every hyperparameter fixed in advance.
\method{} and the vector or scalar control-variate variants use
$\lambda=0.3$; logistic predictors use $C=1$, with two-fold cross-fitting for
prediction-based rectifiers; scalar prediction CV uses no additional ridge;
StratPPI uses five equal-width strata; Local PPI uses local-linear bandwidth
$h_{d,n}$; and PPCI uses a Mat\'ern-$5/2$ kernel on $|z-z'|$ with its
$25$-point L-curve Tikhonov rule. No method is tuned using an
outer-replication error or a gold outcome outside its labeled subset.

It remains to describe how the setting-level results are aggregated and how
the intervals are formed. Within DGP setting $c$, budget
$n$, and method $r$,
relative efficiency is the ratio of Monte Carlo mean errors,
\[
\operatorname{RE}_{c,n,r}
=
\frac{B^{-1}\sum_b e_{cnb,\mathrm{naive}}}
     {B^{-1}\sum_b e_{cnb,r}}.
\]
Budget columns are geometric means across the $24$ DGP settings and
Overall is the geometric mean across all $72$ setting--budget combinations.
For each of $5{,}000$ bootstrap draws and each setting--budget pair $(c,n)$,
we independently resample the replication indices $\{1,\ldots,B\}$ with
replacement and recompute every RE using the selected replications. Within a
setting--budget pair, the same selected indices are used for every method,
making the intervals paired because all methods saw the same pool and labels.
Resampling is independent across budgets, matching the independent pools used
at the three budgets. The DGP settings are not resampled, so the
intervals quantify Monte Carlo uncertainty conditional on the $24$
prespecified DGP settings.

\begin{table}[htbp]
\centering
\footnotesize
\setlength{\tabcolsep}{3pt}
\caption{Repeated-pool superpopulation risk over $24$ DGP
settings formed by eight benchmark-derived profile marginals and three
candidate-indexed outcome regimes
($B=500$, $M_n=10n$, and labeled fraction $\pi=0.1$).
Budget columns are geometric-mean RE across the $24$ settings and Overall is
the geometric mean across all $72$ setting--budget combinations. Brackets are paired
Monte Carlo $95\%$ percentile intervals from $5{,}000$ bootstrap resamples of
the $B=500$ replications; ``Best'' counts setting--budget combinations where
the method is top.
MSE is evaluated against the population target
$\theta_{0,c}(z)=p_m(z)$.}
\label{tab:sim:superpop}
\resizebox{\textwidth}{!}{%
\begin{tabular}{lccccc}
\toprule
Method & RE $n{=}50$ & RE $n{=}100$ & RE $n{=}200$ & Overall & Best \\
\midrule
Naive & $1.000\,[1.000,1.000]$ & $1.000\,[1.000,1.000]$ & $1.000\,[1.000,1.000]$ & $1.000\,[1.000,1.000]$ & 0/72 \\
\method{} & $\mathbf{2.235}\,[2.212,2.260]$ & $\mathbf{2.328}\,[2.306,2.353]$ & $\mathbf{1.849}\,[1.833,1.867]$ & $\mathbf{2.127}\,[2.112,2.144]$ & 72/72 \\
Global CV & $0.953\,[0.948,0.958]$ & $0.974\,[0.971,0.978]$ & $0.990\,[0.988,0.992]$ & $0.972\,[0.970,0.975]$ & 0/72 \\
Per-signal CV & $1.168\,[1.164,1.173]$ & $1.178\,[1.174,1.182]$ & $1.138\,[1.134,1.141]$ & $1.161\,[1.158,1.164]$ & 0/72 \\
Residual-only & $0.824\,[0.812,0.837]$ & $0.579\,[0.569,0.587]$ & $0.354\,[0.349,0.360]$ & $0.553\,[0.546,0.559]$ & 0/72 \\
Plug-in (judge) & $1.063\,[1.053,1.073]$ & $0.758\,[0.750,0.765]$ & $0.493\,[0.488,0.498]$ & $0.735\,[0.729,0.741]$ & 0/72 \\
Plug-in (multi) & $0.881\,[0.874,0.887]$ & $0.672\,[0.667,0.678]$ & $0.466\,[0.462,0.471]$ & $0.651\,[0.647,0.655]$ & 0/72 \\
Aug.\ plug-in & $0.618\,[0.611,0.626]$ & $0.768\,[0.761,0.776]$ & $0.912\,[0.906,0.917]$ & $0.757\,[0.752,0.761]$ & 0/72 \\
Scalar prediction CV & $0.453\,[0.440,0.467]$ & $0.708\,[0.695,0.721]$ & $0.969\,[0.961,0.977]$ & $0.677\,[0.669,0.686]$ & 0/72 \\
Local PPI & $0.499\,[0.493,0.506]$ & $0.462\,[0.456,0.469]$ & $0.374\,[0.369,0.380]$ & $0.442\,[0.437,0.447]$ & 0/72 \\
StratPPI & $0.155\,[0.151,0.161]$ & $0.335\,[0.328,0.343]$ & $0.552\,[0.544,0.559]$ & $0.306\,[0.302,0.311]$ & 0/72 \\
PPCI & $0.727\,[0.718,0.736]$ & $0.835\,[0.826,0.843]$ & $0.936\,[0.929,0.944]$ & $0.828\,[0.823,0.834]$ & 0/72 \\
\bottomrule
\end{tabular}%
}
\end{table}

Under this finite-sample superpopulation risk, \method{} attains
geometric-mean REs of $2.235\times$, $2.328\times$, and $1.849\times$ at
budgets $50$, $100$, and $200$, respectively, and $2.127\times$ overall. It is
the best-performing method in all $72$ setting--budget
combinations; the next-best overall method is Per-signal CV at
$1.161\times$. The smaller margin than in the fixed-pool result is
consistent with the additional benchmark-pool variation and finite-bandwidth
approximation error included by this population-target criterion.

The decline at the largest budget is not caused by a changing labeled
fraction, since $\pi=0.1$ throughout. RE is a ratio of two budget-specific
finite-sample integrated risks and need not be monotone in $n$; here both
$M_n$ and the bandwidth change with the budget while the ridge penalty remains
fixed. This behavior does not contradict the pointwise asymptotic gain result.
Together with the shrinking-bandwidth experiment above, the repeated-pool
experiment shows that the practical estimator's advantage persists when both
the benchmark pool and the labeled subset are freshly sampled, while the
original real-data experiment answers the distinct fixed-pool question.

\newpage

\section{Cheap Evaluator Signal Collection}
\label{app:prompts}

This appendix documents the exact prompts and procedures used to collect the $K=6$ cheap
evaluator signals described in Section~\ref{sec:experiments}. For each model $m$ and item $i$,
signals are collected via API calls to a separate judge model (Opus~4.6). The evaluated model
itself provides the self-reported confidence signal within its original response, without an
additional API call.

\subsection{Pointwise Judge Signal ($S^{\mathrm{judge}}_{im}$)}

We query the judge model with three paraphrased prompts and average the returned scores:
$S^{\mathrm{judge}}_{im} = 3^{-1}\sum_{r=1}^3 G_{imr}$.
The variance across the three scores yields the disagreement signal
$S^{\mathrm{disagree}}_{im} = 3^{-1}\sum_{r=1}^3(G_{imr} - S^{\mathrm{judge}}_{im})^2$,
which measures sensitivity to prompt paraphrasing rather than calibrated judge uncertainty and
serves as a covariate in the local ridge regression.

\begin{tcolorbox}[
    title=\textbf{Pointwise Judge: Variant 1 (confidence rating)},
    colback=blue!5!white,
    colframe=blue!40!black,
    coltitle=black,
    colbacktitle=blue!10,
    fonttitle=\footnotesize,
    fontupper=\scriptsize,
    arc=6pt,
    boxrule=1pt,
    left=8pt, right=8pt,
    top=6pt, bottom=6pt,
    enhanced,
    breakable
]
Look at this solution carefully. Rate your confidence that the final answer is correct.\\
\vspace{0.5\baselineskip}

Question: \{question\}\\
\vspace{0.5\baselineskip}

Solution: \{response\}\\
\vspace{0.5\baselineskip}

Respond with ONLY a number between 0.0 and 1.0.
\end{tcolorbox}

\begin{tcolorbox}[
    title=\textbf{Pointwise Judge: Variant 2 (correctness evaluation)},
    colback=blue!5!white,
    colframe=blue!40!black,
    coltitle=black,
    colbacktitle=blue!10,
    fonttitle=\footnotesize,
    fontupper=\scriptsize,
    arc=6pt,
    boxrule=1pt,
    left=8pt, right=8pt,
    top=6pt, bottom=6pt,
    enhanced,
    breakable
]
Evaluate whether this solution arrives at the correct answer. Consider the reasoning quality and final answer.\\
\vspace{0.5\baselineskip}

Question: \{question\}\\
\vspace{0.5\baselineskip}

Candidate answer: \{response\}\\
\vspace{0.5\baselineskip}

Output a single number 0.0 (certainly wrong) to 1.0 (certainly correct):
\end{tcolorbox}

\begin{tcolorbox}[
    title=\textbf{Pointwise Judge: Variant 3 (expert grading)},
    colback=blue!5!white,
    colframe=blue!40!black,
    coltitle=black,
    colbacktitle=blue!10,
    fonttitle=\footnotesize,
    fontupper=\scriptsize,
    arc=6pt,
    boxrule=1pt,
    left=8pt, right=8pt,
    top=6pt, bottom=6pt,
    enhanced,
    breakable
]
You are an expert grader. Judge the correctness of this response.\\
\vspace{0.5\baselineskip}

Problem: \{question\}\\
\vspace{0.5\baselineskip}

Student response: \{response\}\\
\vspace{0.5\baselineskip}

Probability the final answer is correct (output only a number 0.0--1.0):
\end{tcolorbox}

\subsection{Pairwise Comparison Signal ($S^{\mathrm{pair},h}_{im}$)}

For each of three anchor models $a_h$ ($h=1,2,3$: weak, medium, strong), the judge compares
the evaluated model's response against the anchor's response on the same item. The comparison
is repeated with answer order swapped to mitigate position bias. The pairwise score is
$S^{\mathrm{pair},h}_{im} = (w + 0.5\,t)/2$, where $w$ and $t$ count wins and ties across the
two order-swapped calls.

\begin{tcolorbox}[
    title=\textbf{Pairwise Comparison Prompt},
    colback=green!5!white,
    colframe=green!40!black,
    coltitle=black,
    colbacktitle=green!10,
    fonttitle=\footnotesize,
    fontupper=\scriptsize,
    arc=6pt,
    boxrule=1pt,
    left=8pt, right=8pt,
    top=6pt, bottom=6pt,
    enhanced,
    breakable
]
Compare these two solutions to the question below.\\
Which solution is more likely to have the correct final answer?\\
\vspace{0.5\baselineskip}

Question: \{question\}\\
\vspace{0.5\baselineskip}

Solution A:\\
\{response\_a\}\\
\vspace{0.5\baselineskip}

Solution B:\\
\{response\_b\}\\
\vspace{0.5\baselineskip}

Reply with ONLY one of: ``A'', ``B'', or ``TIE''.
\end{tcolorbox}

\noindent
For each anchor, two calls are made: one with the evaluated model as Solution~A and anchor
as Solution~B, and one with the order swapped. This produces two verdicts per anchor,
yielding $S^{\mathrm{pair},h}_{im} \in \{0, 0.25, 0.5, 0.75, 1\}$.

\subsection{Self-Reported Confidence ($S^{\mathrm{conf}}_{im}$)}

The evaluated model reports its own confidence during answer generation. We append the
following suffix to every benchmark prompt:

\begin{tcolorbox}[
    title=\textbf{Confidence Elicitation Suffix},
    colback=orange!5!white,
    colframe=orange!40!black,
    coltitle=black,
    colbacktitle=orange!10,
    fonttitle=\footnotesize,
    fontupper=\scriptsize,
    arc=6pt,
    boxrule=1pt,
    left=8pt, right=8pt,
    top=6pt, bottom=6pt,
    enhanced,
    breakable
]
After your answer, rate your confidence from 0.0 to 1.0. Format: Confidence: X.XX
\end{tcolorbox}

\noindent
This signal requires no additional API call beyond the original model response, but the elicitation
suffix adds output tokens and may affect the generated answer. The confidence value is extracted
via regex matching on the model output.

\subsection{Signal Summary}

\begin{table}[h]
\centering
\caption{Summary of the $K=6$ cheap evaluator signals collected per item.}
\begin{tabular}{llcl}
\toprule
Signal & Description & API calls & Range \\
\midrule
$S^{\mathrm{judge}}$ & Mean of 3 pointwise judge scores & 3 & $[0,1]$ \\
$S^{\mathrm{disagree}}$ & Variance of 3 pointwise judge scores & 0 (derived) & $[0, 2/9]$ \\
$S^{\mathrm{pair},1}$ & Pairwise vs.\ weak anchor (Ministral~3B) & 2 & $\{0,0.25,0.5,0.75,1\}$ \\
$S^{\mathrm{pair},2}$ & Pairwise vs.\ medium anchor (Haiku~3) & 2 & $\{0,0.25,0.5,0.75,1\}$ \\
$S^{\mathrm{pair},3}$ & Pairwise vs.\ strong anchor (Opus~4.6) & 2 & $\{0,0.25,0.5,0.75,1\}$ \\
$S^{\mathrm{conf}}$ & Self-reported confidence & 0 additional & $[0,1]$ \\
\midrule
\multicolumn{2}{l}{Total per item} & 9 & $K=6$ signals \\
\bottomrule
\end{tabular}
\end{table}

\end{document}